\documentclass{article}
\usepackage[nonatbib,preprint]{neurips_2026}

\usepackage[T1]{fontenc}
\usepackage{microtype}
\usepackage{hyperref}
\usepackage{url}
\usepackage{booktabs}
\usepackage[numbers]{natbib}
\usepackage{amsmath}
\usepackage{amssymb}
\usepackage{mathtools}
\usepackage{amsthm}
\usepackage{adjustbox}
\usepackage[table]{xcolor}
\usepackage{tikz} % For drawing the heatmap placeholder
\usepackage{graphicx}
\graphicspath{{./}{../}}
\usepackage{pifont}
\usepackage{enumitem}
\usepackage{cleveref}
\usepackage{placeins}

\newcommand{\cmark}{\textcolor{green!60!black}{\ding{51}}}
\newcommand{\xmark}{\textcolor{red!70!black}{\ding{55}}}
\usepackage{multicol}
\usepackage{multirow}
\usepackage{wrapfig}
\usepackage[breakable,skins]{tcolorbox}
\usepackage{arydshln}

\definecolor{ForestGreen}{RGB}{34, 139, 34}
\definecolor{BrickRed}{RGB}{178, 34, 34}

\definecolor{VSBlue}{HTML}{1F77B4}
\definecolor{CoGreen}{HTML}{2CA02C}
\definecolor{PopGreen}{HTML}{1B5E20}
\newcommand{\vs}{\textcolor{VSBlue}{\textbf{Verbalized Sampling}}}
\newcommand{\cotrain}{\textcolor{CoGreen}{\textbf{Co-Training}}}
\newcommand{\popcotrain}{\textcolor{PopGreen}{\textbf{Population Co-Training}}}

\theoremstyle{plain}
\newtheorem{theorem}{Theorem}[section]
\newtheorem{proposition}[theorem]{Proposition}
\newtheorem{lemma}[theorem]{Lemma}
\newtheorem{corollary}[theorem]{Corollary}
\newtheorem{remark}[theorem]{Remark}
\theoremstyle{definition}
\newtheorem{definition}[theorem]{Definition}
\newtheorem{assumption}[theorem]{Assumption}

\definecolor{darkblue}{rgb}{0, 0, 0.5}
\hypersetup{colorlinks=true, citecolor=darkblue, linkcolor=darkblue, urlcolor=darkblue}

\title{One Frozen Simulator Is Not Enough: \\ Simulator Collapse in Multi-Agent RL}

\author{%
  \textbf{Simon Yu}\textsuperscript{1} \quad
  \textbf{Nicholas Tomlin}\textsuperscript{2} \quad
  \textbf{Marwa Abdulhai}\textsuperscript{3} \quad
  \textbf{Ximing Lu}\textsuperscript{4} \quad
  \textbf{Derek Chong}\textsuperscript{5} \\
  \textbf{Abe Hou}\textsuperscript{5} \quad
  \textbf{Dilara Soylu}\textsuperscript{5} \quad
  \textbf{Sergey Levine}\textsuperscript{3} \quad
  \textbf{Christopher D. Manning}\textsuperscript{5} \quad
  \textbf{Weiyan Shi}\textsuperscript{1} \\
  \textsuperscript{1}Northeastern University \quad
  \textsuperscript{2}New York University \quad
  \textsuperscript{3}UC Berkeley \\
  \textsuperscript{4}University of Washington \quad
  \textsuperscript{5}Stanford University
}

\newcommand{\wyshi}[1]{\textcolor{orange}{[wyshi: #1]}}
\newcommand{\simon}[1]{\textcolor{blue}{[simon: #1]}}
\newcommand{\derek}[1]{\textcolor{violet}{[derek: #1]}}
\newcommand{\dilara}[1]{\textcolor{green}{[dilara: #1]}}
\newcommand{\abe}[1]{\textcolor{cyan}{[abe: #1]}}
\newcommand{\cm}[1]{\textcolor{teal}{[CM: #1]}}
\newcommand{\nt}[1]{\textcolor{olive}{[NT: #1]}}

\renewcommand{\wyshi}[1]{}
\renewcommand{\simon}[1]{}
\renewcommand{\derek}[1]{}
\renewcommand{\dilara}[1]{}
\renewcommand{\abe}[1]{}
\renewcommand{\cm}[1]{}
\renewcommand{\nt}[1]{}

\begin{document}

\maketitle

\begin{abstract}
Multi-agent reinforcement learning (MARL) for human-AI interaction usually trains the agent against a single large language model (LLM) that simulates the user. We show that agents trained this way can systematically fail to generalize, and we trace the failure to \textbf{simulator collapse}: the simulator LLM is mode-collapsed, and a policy trained against it learns narrow strategies that exploit the simulator's mode behavior; these strategies transfer poorly to unseen simulators and to real users. We formalize the collapse and propose two solutions. \textbf{Verbalized Sampling}, applied at inference time, samples each simulator reply from a verbalized response distribution. \textbf{Co-Training}, applied during training similar to self-play, co-evolves the simulator as the policy, so the policy cannot exploit the mode behavior. A population variant samples the simulator for each rollout from a pool of recent checkpoints. We validate both solutions on three benchmarks: Persuasion for Good, $\tau^2$-bench, and CooperBench. Verbalized Sampling improves held-out success by up to 9\% over single-simulator RL, and Co-Training pushes gains further to 14\%; the human study shows similar gain on real users. Both solutions also preserve the policy diversity that collapses under single-simulator RL. Overall, our findings on simulator collapse point to diverse simulators as a requirement for RL agents that serve real users. We release \textsc{SCOPE}\footnote{Code: \url{https://github.com/CHATS-lab/scope_usim}}, an open-source framework for Population Co-Training in multi-agent RL.
\vspace{-2em}
\end{abstract}

\begin{figure*}[ht!]
\centering
\includegraphics[width=\textwidth]{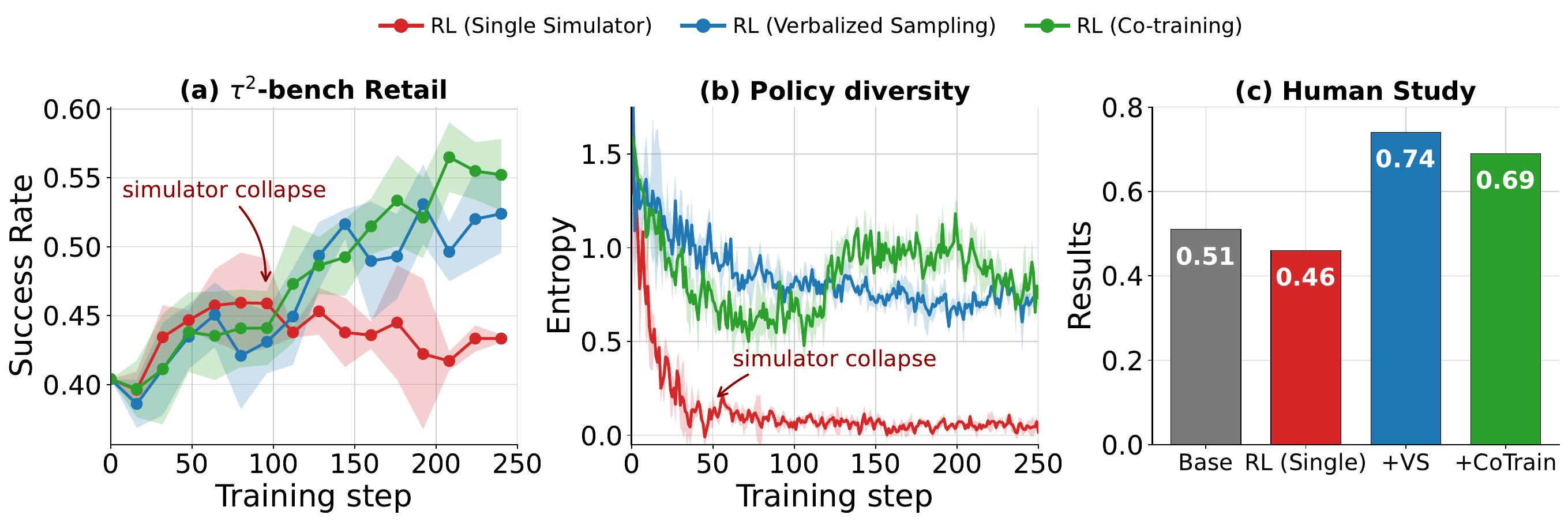}
\caption{\textbf{Single-simulator RL collapses; \vs{} and \cotrain{} recover.} The figure shows the result on $\tau^2$-bench with Qwen3-4B-Instruct.
\textbf{(a)} Held-out generalization: RL against a single frozen simulator peaks early, then starts collapsing.
\textbf{(b)} Policy entropy: similarly, RL (single simulator) drives the policy's entropy to near zero.
\textbf{(c)} Human study: Co-Training lifts real-user performance over RL (single simulator), which drops below even the untrained baseline.
\vspace{-1em}
}
\label{fig:overview}
\end{figure*}

\section{Introduction}

Reinforcement learning with verifiable rewards (RLVR) has improved LLMs on tasks such as mathematical reasoning~\citep{guo2025deepseek, liu2025understandingr1zeroliketrainingcritical, yu2025dapoopensourcellmreinforcement}, tool use~\citep{feng2025retoolreinforcementlearningstrategic, jin2025searchr1trainingllmsreason}, and software engineering (SWE)~\citep{yang2025swesmithscalingdatasoftware, wei2025swerladvancingllmreasoning}. In these settings, the environment acts as the verifier for the policy. However, many LLM applications are multi-agent and involve real humans, such as customer support~\citep{barres_2-bench_2025}, collaborative coding~\citep{khatua_cooperbench_2026, wang2026position}, dialogue~\citep{wang2019persuasion}, and tutoring~\citep{abdulhai_consistently_2025}, where the outcome depends on how the user responds~\citep{qian_userrl_2025, zhou_tom-swe_2025, wu_humanlm_nodate}. Real users cannot stay in the RL loop for the volume of conversations that training needs, so prior work uses LLM-based user simulators in their place~\citep{10.1145/3586183.3606763, park2026llmagentsgroundedselfreports}.

Existing work usually prompts a single frozen LLM to play the user~\citep{qian_userrl_2025,yu_sotopia-rl_2025}. In multi-agent RL~\citep{abdulhai2023lmrlgymbenchmarksmultiturn}, this simulated user takes the place of the deterministic verifier, and its output distribution decides which states the policy visits and which gradients it learns from. The goal is a policy that generalizes to unseen simulators and to real users. In this paper, we show that this setup is problematic, because many aligned LLMs are mode-collapsed~\citep{jiang_artificial_2025, gxchen2025klregularized, zhang2025verbalizedsamplingmitigatemode}. If the simulated human always cooperates and discloses full details upfront, the agent will never learn how to handle a customer who pushes back. A policy trained against such a simulator tends to overfit to narrow strategies that exploit its mode~\citep{macdiarmid2025natural}. It does well in distribution but fails on real users whose behavior is more diverse and falls outside that mode, and it can end up worse than the untrained model (Figure~\ref{fig:overview}).

We call this failure mode \textbf{simulator collapse} and formalize it. Our analysis in Section~\ref{sec:theory} shows that the policy gradient is biased toward the simulator's mode even when it does not vanish. Over training, repeated updates rank the policy's trajectories by how well they exploit that mode and concentrate probability on a narrow exploit set. This accounts for the rapid drop in policy entropy we observe during training.

We examine two solutions to simulator collapse (Section~\ref{sec:cotrain}), both of which act on the simulator and leave the policy update unchanged. First, \vs{}~\citep{zhang2025verbalizedsamplingmitigatemode} changes how the simulator's replies are drawn, by sampling each turn from a verbalized distribution over potential replies, which restores diversity within a single simulator without retraining it. Second, \cotrain{} trains both the simulator and the policy, each with its own reward function, and the two models co-evolve as in self-play~\citep{liu_selfredteam_2025, liu_spiral_2025}. As the policy improves, the simulator's mode shifts over training and a strategy that exploited an earlier mode no longer works against the current one. \popcotrain{} extends this idea by sampling the simulator from a pool of past checkpoints~\citep{berner2019dota, vinyals2019alphastar}, so that every batch mixes several versions of the simulator, which further increases diversity.

Section~\ref{sec:experiments} evaluates these solutions on Persuasion for Good~\citep{wang2019persuasion} (adversarial dialogue), $\tau^2$-bench~\citep{barres_2-bench_2025} (collaborative task), and CooperBench~\citep{khatua_cooperbench_2026} (collaborative SWE). Single-simulator RL initially improves performance on Persuasion for Good and $\tau^2$-bench, but these gains shrink as training continues and performance falls back toward the initial checkpoint. For Qwen3-4B-Instruct, \vs{} and \popcotrain{} reach $1.20\times$ and $1.35\times$ the best-checkpoint $\tau^2$-bench Retail success of single-simulator RL, respectively, and $1.76\times$ and $1.85\times$ its P4G reward. On CooperBench, population self-play with Qwen3.5-27B reaches $1.06\times$ the success of cross-play against a single frozen partner, although cross-play against a frozen ensemble of three partners does slightly better. In the human study, \cotrain{} raises the $\tau^2$ task score $1.63\times$ over single-simulator RL with real users, and improves P4G naturalness by 1.24 points on the Likert scale. During training, \vs{}, \cotrain{}, and \popcotrain{} also maintain higher policy entropy than single-simulator RL.

Our contributions are:
\begin{enumerate}[topsep=1pt, itemsep=1pt, parsep=0pt, leftmargin=*]
    \item \textbf{We identify simulator collapse in MARL.} We show that a mode-collapsed simulator biases the policy gradient toward its mode, so the policy concentrates on a narrow exploit set and its entropy collapses (\S\ref{sec:theory_analysis}).
    \item \textbf{Two solutions to diversify the simulator for RL training.} Verbalized Sampling diversifies a frozen simulator at inference time, and Co-Training and Population Co-Training train the simulator alongside the policy, as in self-play (\S\ref{sec:cotrain}). We release all three in \textsc{SCOPE}, an open-source framework.

    \item \textbf{Experiments on three benchmarks and with real users.} Both solutions outperform single-simulator RL on Persuasion for Good and $\tau^2$-bench, where single-simulator RL's held-out gains fade back toward the untrained baseline. On CooperBench, population self-play outperforms cross-play with a single frozen partner and comes close to a frozen three-model ensemble (\S\ref{sec:experiments}). In the human study, \cotrain{} also improves $\tau^2$-bench task outcome and P4G dialogue naturalness over single-simulator RL.
    \end{enumerate}

\section{Background}
\label{sec:background}

\paragraph{Multi-agent RL as a POMDP.} We model a multi-turn dialogue as a two-player partially observable Markov decision process (POMDP), in which the state is the shared conversation history. Persuasion for Good and $\tau^2$-bench are partially observable stochastic games (POSGs), in which the user simulator holds private observations such as its goal and persona and only the agent receives the task reward, while CooperBench is a decentralized POMDP (Dec-POMDP), in which two cooperating coding agents share a task-success reward. A single shared-history POMDP covers all three settings, since the theory in \S\ref{sec:theory_analysis} depends only on the joint trajectory distribution, which does not change with how the private observations are split between the two players. At turn $t$, the state is $s_t = (o_0, a_0^\pi, a_0^\phi, \ldots, o_{t-1}, a_{t-1}^\pi, a_{t-1}^\phi)$, the agent samples an utterance $a_t^\pi \sim \pi_\theta(\cdot \mid s_t)$, and the user simulator then samples a response $a_t^\phi \sim \phi_\psi(\cdot \mid s_t, a_t^\pi)$. Each trajectory $\tau = (s_0, a_0^\pi, a_0^\phi, \ldots, s_T)$ ends with a terminal reward $R(\tau)$, and the agent maximizes
\begin{equation}
J(\theta; \psi) = \mathbb{E}_{\tau \sim (\pi_\theta, \phi_\psi)}[R(\tau)].
\end{equation}
The state-visitation distribution $d^{\pi_\theta}_\psi(s) = \sum_{t} \Pr(s_t = s \mid \pi_\theta, \phi_\psi)$ depends on both the agent and the simulator. The simulator therefore shapes which histories the policy learns from, not only the reward each trajectory receives. When the simulator is fixed, we write $J_\phi(\theta) = J(\theta; \psi)$.

\paragraph{Policy update.} We train with REINFORCE~\citep{Williams2004SimpleSG} on full multi-turn trajectories, using group-relative reward normalization. For each task, we sample $G$ trajectories, score each one by its terminal reward $R(\tau^n)$, and assign the z-scored advantage $\hat A^n = (R(\tau^n) - \bar R)/\sigma_R$ (Eq.~\ref{eq:reinforce}) to every agent token in trajectory $n$. If all trajectories in a group receive the same reward, $\sigma_R = 0$ and the group contributes no update. Even when $\sigma_R > 0$, if the simulator keeps responding the same way at the histories visited during training, the remaining contrast mostly ranks agent samples by how well they exploit that simulator, which leaves the gradient nonzero but biased (Section~\ref{sec:theory}).
\begin{equation}
\label{eq:reinforce}
\hat{A}^n = \frac{R(\tau^n) - \bar{R}}{\sigma_R}, \quad \bar{R} = \frac{1}{G}\sum_{n=1}^G R(\tau^n), \quad \sigma_R = \sqrt{\frac{1}{G}\sum_{n=1}^G (R(\tau^n) - \bar{R})^2}.
\end{equation}

\section{Simulator Collapse}
\label{sec:preliminary}
\label{sec:theory}

\begin{figure*}[t]
\centering
\includegraphics[width=\textwidth]{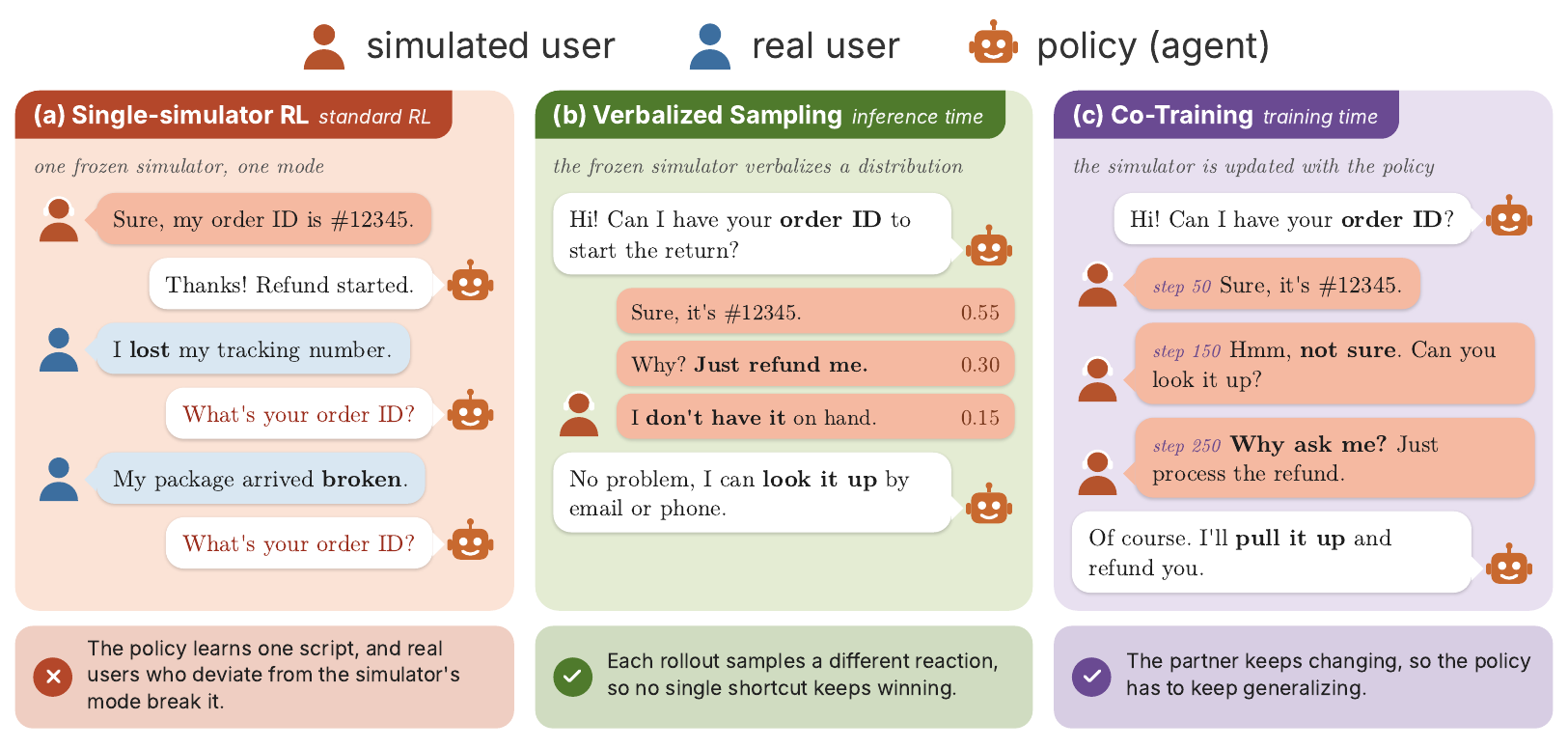}
\caption{\textbf{Simulator collapse and the two solutions.}
\textbf{(a) Single-simulator RL.} A frozen LLM simulator keeps giving its modal reply, so the policy learns one script that fails when real users (blue) deviate from it.
\textbf{(b) Verbalized Sampling (inference time).} A single prompt asks the frozen simulator for several plausible user replies with likelihoods, so one rollout already contains varied reactions such as acceptance, pushback, and refusal.
\textbf{(c) Co-Training (training time).} The simulator is trained alongside the policy, so its replies change across training steps and the mode the policy would otherwise fit keeps moving.
\vspace{-1.6em}}
\label{fig:structure}
\end{figure*}

Aligned LLMs favor typical responses under direct prompting~\citep{jiang_artificial_2025, zhang2025verbalizedsamplingmitigatemode, zhang2025noveltybenchevaluatinglanguagemodels, gxchen2025klregularized}, and user-simulation studies find that LLM simulators are overly cooperative and stylistically uniform~\citep{naous_flipping_2025, zhou2026sim2real, mehri2026measuring}. In this section, we trace what this narrowness does to RL training, and Figure~\ref{fig:structure} sketches both the collapse and the two solutions. When the simulator is collapsed on the policy's training rollouts, the policy gradient comes close to the gradient against a deterministic simulator that always returns its mode (\S\ref{sec:theory_analysis}), and group-relative updates then move policy mass onto the narrow strategy that wins against that mode. Policy entropy drops as a result. \S\ref{exp:fixed_policy} tests both predictions, and \S\ref{sec:cotrain} describes how Verbalized Sampling and Co-Training counter the collapse.

\subsection{Definition and Hypothesis}
\label{sec:theory_definition}

\paragraph{Mode collapse.}
At a simulator turn, the user response is sampled from $\phi_\psi(\cdot \mid s_t, a_t^\pi)$, and we call the most likely response at that turn the simulator's \emph{mode}:
\begin{equation}
    a_\phi^\star(s, a^\pi) \;\in\; \arg\max_{a^\phi} \phi_\psi(a^\phi \mid s, a^\pi),
\label{eq:dominant-response}
\end{equation}
and write $\epsilon_\phi(s, a^\pi) = 1 - \phi_\psi(a_\phi^\star(s, a^\pi) \mid s, a^\pi)$ for the probability that the simulator deviates from its mode, so a small $\epsilon_\phi$ means the simulator keeps emitting its mode. Aligned LLMs concentrate heavily on their mode, an effect usually called mode collapse~\citep{jiang_artificial_2025, zhang2025verbalizedsamplingmitigatemode, zhang2025noveltybenchevaluatinglanguagemodels, gxchen2025klregularized}.

\begin{definition}[Simulator collapse on the training rollouts]
\label{def:conditional-collapse}
For a policy $\pi_\theta$ and a threshold $\epsilon^\star \in [0, 1]$, we say the simulator is \emph{$\epsilon^\star$-collapsed on the training rollouts} if
\begin{equation*}
\mathbb{E}_{(s_t, a_t^\pi) \sim (\pi_\theta, \phi_\psi)}\!\bigl[\epsilon_\phi(s_t, a_t^\pi)\bigr] \;\le\; \epsilon^\star
\end{equation*}
at the simulator turns visited by rollouts from $(\pi_\theta, \phi_\psi)$.
\end{definition}

The definition refers only to the histories that the current policy visits during training, since a simulator can produce many different trajectories across a dataset and still behave almost deterministically at those histories. Collapse therefore means that the per-turn distribution is narrow, and prompts that only change which behavior is modal do not broaden it. We state $\epsilon_\phi$ over exact token sequences for notational convenience, but every result below also holds after a measurable coarsening $\Pi: a^\phi \mapsto b^\phi$ onto behavior classes such as dialogue acts or strategy clusters, because the TV distance to the modal point mass does not increase under projection (Remark~\ref{rem:behavioral-coarsening}). The threshold $\epsilon^\star$ enters Theorem~\ref{thm:modal-gradient}, whose gradient-bias bound is small when $\bar{\epsilon}_H(\theta) \le \epsilon^\star H$. The zero-variance batch fraction in Figure~\ref{fig:training_dynamics} is a one-sided diagnostic of training-signal degeneration rather than a direct estimate of $\epsilon_\phi$. A small $\epsilon_\phi$ does imply small reward variance (Lemma~\ref{lem:user-variance}), but sparse binary rewards and all-failure batches drive the variance to zero as well. Only the simulator-side term of Eq.~\ref{eq:variance-decomposition} tracks $\epsilon_\phi$, so a cleaner diagnostic would separate it from the agent-side term.

\paragraph{Hypothesis.}
A mode-collapsed simulator reduces multi-agent RL to training against one fixed user. The RL policy then learns whichever strategy wins against the simulator's mode, and that strategy fails as soon as other models or real users depart from it.

\subsection{How Simulator Collapse Affects the Policy}
\label{sec:theory_analysis}

Mode collapse affects the policy in two ways. It pulls the policy gradient toward the objective of a deterministic user who always gives the mode response (Theorem~\ref{thm:modal-gradient}), and it strips out the simulator's share of the reward variance, leaving the z-scored advantage to rank samples by how well they exploit the mode (Lemma~\ref{lem:user-variance}). Policy mass then concentrates geometrically on the set $A_x$ of mode-exploit strategies (Proposition~\ref{prop:log-odds}, Corollary~\ref{cor:entropy-collapse}), and the resulting low-entropy policy does worse on real users whose behavior the strategies in $A_x$ do not handle (Proposition~\ref{prop:coverage-regret}). Appendix~\ref{appendix:proof} gives the proofs.

\paragraph{Gradient bias under collapse.}
Let $M_\phi$ be the dialogue POMDP induced by $\phi$, and let $M_{\mathrm{mode}}$ be the modified POMDP in which every simulator turn deterministically emits the mode $a_\phi^\star(s_t, a_t^\pi)$, with objective $J_{\mathrm{mode}}(\theta)$. Since $M_{\mathrm{mode}}$ is determined by $\phi_\psi$ alone, it does not change with $\pi_\theta$.

\begin{theorem}[Simulator collapse induces mode-user optimization]
\label{thm:modal-gradient}
Assume rewards are bounded in $[0, R_{\max}]$ and the trajectory-level policy score satisfies $\lVert \sum_t \nabla_\theta \log \pi_\theta(a_t^\pi \mid s_t) \rVert \le B$ (e.g., under finite-length truncation and gradient clipping, Appendix~\ref{appendix:prelim}). Couple $\pi_\theta$ in $M_\phi$ and $M_{\mathrm{mode}}$ turn by turn (same task, same agent randomness, maximal coupling at each simulator turn). Define the accumulated collapse error along the rollout as
\begin{equation}
    \bar{\epsilon}_H(\theta) \;=\; \mathbb{E}\!\left[\,\sum_{t=1}^{H} \epsilon_\phi(s_t, a_t^\pi)\right].
\label{eq:on-policy-collapse-error}
\end{equation}
Write $P_\phi^\theta$ and $P_{\mathrm{mode}}^\theta$ for the joint trajectory distributions under $\pi_\theta$ in $M_\phi$ and $M_{\mathrm{mode}}$ respectively. Then $D_{\mathrm{TV}}\!\bigl(P_\phi^\theta,\, P_{\mathrm{mode}}^\theta\bigr) \le \bar{\epsilon}_H(\theta)$, and
\begin{equation}
    \bigl\lVert \nabla_\theta J_\phi(\theta) - \nabla_\theta J_{\mathrm{mode}}(\theta) \bigr\rVert \;\le\; 2BR_{\max}\,\bar{\epsilon}_H(\theta).
\label{eq:modal-gradient-bound}
\end{equation}
\end{theorem}

By Theorem~\ref{thm:modal-gradient}, the gradient need not vanish under collapse, but up to $\bar{\epsilon}_H(\theta)$ it points toward the objective in which every user gives the mode response $a_\phi^\star$. The bound applies to the idealized REINFORCE gradient with a bounded trajectory score. Our implementation adds GRPO-style clipping, group-relative normalization, and gradient clipping (Appendix~\ref{appendix:cotrain_framework}), so for the update we actually run, the theorem gives the direction of the bias but does not tightly bound the gradient of the surrogate loss. The bound is informative when $\bar{\epsilon}_H \ll 1$. The zero-variance batch fraction in Figure~\ref{fig:training_dynamics} climbs above $80\%$, which is consistent with this regime at $H{=}30$ and $50$, subject to the caveats about this proxy in \S\ref{sec:theory_definition} (Appendix~\ref{appendix:proof-modal-gradient}).

\paragraph{What group-relative advantages measure.}
Collapse also affects the within-task reward variance that group-normalized RL uses to z-score rewards. Writing $\xi_\pi$ for agent-side randomness, $\xi_U$ for simulator-side randomness, and $R_x = R(x, \xi_\pi, \xi_U)$, the law of total variance gives
\begin{equation}
    \mathrm{Var}[R_x \mid x] = \underbrace{\mathbb{E}_{\xi_\pi}\!\bigl[\mathrm{Var}_{\xi_U}(R_x \mid x, \xi_\pi)\bigr]}_{\text{simulator-side contrast}} + \underbrace{\mathrm{Var}_{\xi_\pi}\!\bigl[\mathbb{E}_{\xi_U}(R_x \mid x, \xi_\pi)\bigr]}_{\text{agent-side contrast}}.
\label{eq:variance-decomposition}
\end{equation}

\begin{lemma}[Collapse removes simulator-side reward contrast]
\label{lem:user-variance}
If the simulator's trajectory is $\epsilon_H(x, \xi_\pi)$-close in TV to the mode trajectory at $(x, \xi_\pi)$, then $\mathrm{Var}_{\xi_U}(R_x \mid x, \xi_\pi) \le R_{\max}^2\,\epsilon_H(x, \xi_\pi)$.
\end{lemma}

When the simulator-side term vanishes, the z-scored advantage reflects only differences between the agent's own samples. Group-relative RL then ranks those samples by how well they exploit the simulator's mode and ignores how they would fare against other user behaviors.

\paragraph{Policy concentration under a persistent mode advantage.}
Let $Y$ be an abstraction that groups semantically equivalent agent responses into a small set of strategies, such as ``open with empathy'' or ``cite charity statistics''. It is the agent-side counterpart of the user-side coarsening in Definition~\ref{def:conditional-collapse}, and we write $q_k(y \mid x)$ for the strategy distribution at update $k$. A strategy $y$ \emph{exploits} the simulator's mode $a_\phi^\star$ if its mode-user value $Q_{\mathrm{mode}}(x, y)$ is high, that is, if $y$ earns reward in $M_{\mathrm{mode}}$, where the simulator emits $a_\phi^\star$ at every turn. We write $A_x$ for the set of mode-exploit strategies and $\Delta_x > 0$ for the mode-exploit gap, the smallest $Q_{\mathrm{mode}}$ advantage of any $y \in A_x$ over any $y' \notin A_x$ (Appendix~\ref{appendix:proof-log-odds}). Because the realized rollouts are dominated by trajectories in which the simulator emits $a_\phi^\star$ (Theorem~\ref{thm:modal-gradient}), $\Delta_x$ lower-bounds the advantage of every mode-exploit strategy over every other strategy on the trajectories the agent trains on. We model the token-level GRPO step that we actually run as an idealized KL-regularized softmax update on $q_k$. This idealization captures the mechanism behind the entropy collapse we observe, but the proposition below does not apply directly to the optimizer we implement.

\begin{proposition}[Mode advantage concentrates policy mass]
\label{prop:log-odds}
Under a KL-regularized softmax policy-gradient step on $q_k(y \mid x)$ with learning rate $\eta$ and bounded $Q$-estimation errors, the log-odds of $A_x$ satisfy
\begin{equation}
    \log\!\tfrac{q_{k+1}(A_x \mid x)}{q_{k+1}(A_x^c \mid x)} \;\ge\; \log\!\tfrac{q_k(A_x \mid x)}{q_k(A_x^c \mid x)} + g_x,
\label{eq:log-odds-growth}
\end{equation}
where $g_x > 0$ depends on $\eta$, the gap $\Delta_x$, and the simulator-collapse and estimation errors (Appendix~\ref{appendix:proof-log-odds} gives exact constants).
\end{proposition}

\begin{corollary}[Entropy concentration]
\label{cor:entropy-collapse}
Iterating Eq.~\ref{eq:log-odds-growth} gives
\begin{equation}
    q_k(A_x \mid x) \;\ge\; \frac{1}{1 + \frac{1-q_0(A_x \mid x)}{q_0(A_x \mid x)}\,e^{-k g_x}},
\label{eq:mass-concentration}
\end{equation}
so the strategy distribution concentrates onto $A_x$ geometrically fast in $k$.
\end{corollary}

Strategy concentration shows up as a drop in the token-level entropy we measure in \S\ref{exp:fixed_policy} only if the exploit strategies in $A_x$ are themselves low-entropy text, such as fixed scripts or formulaic appeals. Late-training rollouts bear this out qualitatively: three rollouts that share a starting context end up nearly word-for-word the same (Appendix~\ref{appendix:collapse_examples}).

\paragraph{Transfer to real users.}
By Corollary~\ref{cor:entropy-collapse}, the trained policy concentrates on $A_x$. It therefore fails on real users whose behavior the strategies in $A_x$ do not handle.

\begin{proposition}[Deployment regret from missing user behaviors]
\label{prop:coverage-regret}
Let $B_x$ be a set of real-user behaviors requiring a strategy outside $A_x$, with real-user mass $q_\star(x) = P_\star(B_x \mid x)$. Suppose every $y_m \in A_x$ is worse than an adaptive strategy $y_b$ by at least $\Delta_x^{\mathrm{real}}$ on $B_x$ and by at most $\nu_x$ outside $B_x$, and the trained policy places probability at least $1 - \alpha_x$ on $A_x$. Then
\begin{equation}
    J_\star(y_b \mid x) - J_\star(\hat{\pi} \mid x) \;\ge\; (1 - \alpha_x)\bigl[\Delta_x^{\mathrm{real}}\,q_\star(x) - \nu_x(1 - q_\star(x))\bigr] - \alpha_x R_{\max}.
\label{eq:deployment-regret}
\end{equation}
\end{proposition}

The bound is positive when $\Delta_x^{\mathrm{real}}\,q_\star(x) > \nu_x(1 - q_\star(x))$, that is, when the missing behaviors are common enough that the adaptive strategy's gain on them outweighs what it loses elsewhere. We expect this regime when real users push back or impose constraints that the training simulator rarely produces. \S\ref{exp:fixed_policy} tests transfer to unseen simulators, and the human study (Q4 in \S\ref{sec:results}) tests transfer to real users.

\subsection{Empirical Evidence}
\label{exp:fixed_policy}

We test these predictions on Persuasion for Good~\citep{wang2019persuasion} and $\tau^2$-bench~\citep{barres_2-bench_2025}. For each single-simulator run, we track the training reward, the policy entropy, and the eval reward on a panel $\Phi_{\mathrm{eval}}$ of six simulators, three seen during training and three from unseen model families (Appendix~\ref{appendix:hyperparameters} gives the setup).

\paragraph{Single-simulator RL collapses in practice.}
We train against three frozen simulators that differ in how strongly they concentrate on their mode (GPT-5-mini, Haiku-4.5, and Gemini-3-Flash), and Figure~\ref{fig:entropy_collapse_sim} shows the results. Training reward climbs in every run and rises fastest against the most modal simulator. OOD eval peaks early and then declines in all three runs, though by very different amounts. Against the most modal simulator, Gemini-3-Flash, it ends below the untrained baseline. Against the least modal one, GPT-5-mini (the simulator we use in our main experiments), the decline is gentler but still clear. Policy entropy falls toward zero in every run. The policy learns a narrow exploit that does not transfer to unseen simulators, and this drop in token entropy is consistent with the strategy-level concentration of Corollary~\ref{cor:entropy-collapse}, provided the exploit strategies are themselves low-entropy text (Remark~\ref{rem:strategy-entropy}).

\begin{figure}[t]
\centering
\includegraphics[width=\linewidth]{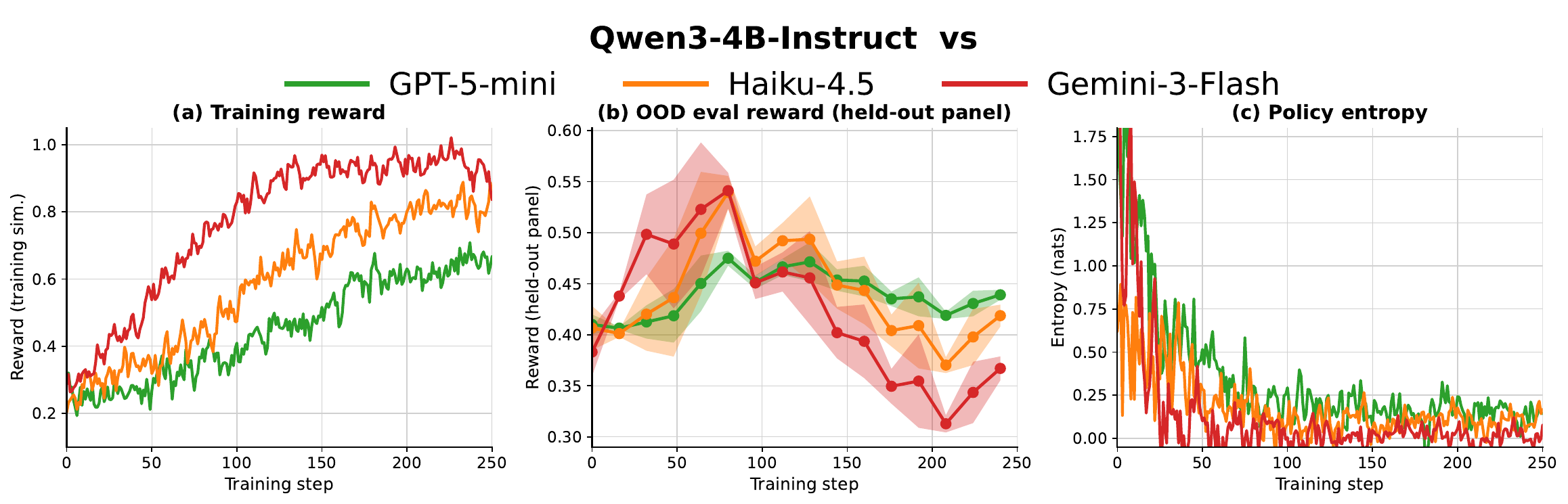}
\caption{\textbf{Single-simulator RL collapses.} Three single-simulator REINFORCE runs with three seeds each, and $\pm1\sigma$ shading on OOD eval. Training reward climbs in every run (a), OOD eval peaks early and declines (b), and policy entropy collapses (c). Training reward keeps rising while OOD eval falls because the simulator's collapse carries over into the policy, as the entropy drop in (c) shows.}
\label{fig:entropy_collapse_sim}
\end{figure}

Because the failure comes from the training environment and not from the RL algorithm, both of our solutions change the simulator, each at a different point in the training loop. \vs{} acts at inference time and draws each simulator turn of a rollout from a verbalized response distribution. \cotrain{} acts at training time and updates the simulator alongside the policy, so the mode the policy could exploit keeps shifting. \S\ref{sec:cotrain} develops both solutions, and \S\ref{sec:results} compares them.

\subsection{Verbalized Sampling and Co-Training}
\label{sec:theory_fix}
\label{sec:cotrain}

The analysis in \S\ref{sec:theory_analysis} relies on two assumptions, and each of our solutions targets one of them. The first is that the per-turn collapse error $\epsilon_\phi(s_t, a_t^\pi)$ stays near zero, in which case $\bar{\epsilon}_H(\theta)$ in Theorem~\ref{thm:modal-gradient} approaches zero and $\nabla_\theta J_\phi$ approaches $\nabla_\theta J_{\mathrm{mode}}$. The second is that the mode-exploit set $A_x$ stays fixed during training, which lets the log-odds of $A_x$ accumulate over $k$ updates until Corollary~\ref{cor:entropy-collapse} concentrates $q_k$ on $A_x$ at geometric rate $g_x$. Under the reference-recovery assumption ($D_{\mathrm{TV}}(p^{\mathrm{VS}}_\phi, P) \le \eta$, Proposition~\ref{prop:vs-lower-bound}), Verbalized Sampling brings the policy gradient close to the reference-user gradient and away from the mode-user gradient that Corollary~\ref{cor:entropy-collapse} requires. Co-Training instead moves $A_x$ at every step, so the log-odds never have a fixed target to grow toward.

\paragraph{\vs{}.}
A greedy query drives $\epsilon_\phi(s)$ toward zero because it always returns the simulator's most likely reply. A verbalized query instead returns $K$ candidate responses with verbalized probabilities, and their distribution $p^{\mathrm{VS}}_\phi(\cdot \mid s)$ approximates the simulator's pre-RLHF reference distribution $P(\cdot \mid s)$, because asking for a distribution recovers the spread of responses that direct prompting of an aligned model suppresses~\citep{zhang2025verbalizedsamplingmitigatemode}. The closeness $D_{\mathrm{TV}}(p^{\mathrm{VS}}_\phi, P) \le \eta$ is an empirical assumption, and Appendix~\ref{appendix:vs_theory} discusses when it holds and when it fails. The reference $P$ comes from the simulator's pretraining and differs from the real-user population $P_{\mathrm{real}}$, so the proposition below makes no claim about the distance to real users, and our human study (Appendix~\ref{appendix:human_study}) tests transfer to real users only for Co-Training.

\begin{proposition}[Reference-gradient recovery under Verbalized Sampling]
\label{prop:vs-lower-bound}
If $D_{\mathrm{TV}}(p^{\mathrm{VS}}_\phi(\cdot \mid s, a^\pi),\, P(\cdot \mid s, a^\pi)) \le \eta(s, a^\pi)$ at every visited state along an $H$-turn rollout, then the trajectory distributions and policy gradients in $M_{\mathrm{VS}}$ vs.\ the reference-user environment $M_{\mathrm{ref}}$ satisfy $D_{\mathrm{TV}}(P^\theta_{\mathrm{VS}}, P^\theta_{\mathrm{ref}}) \le \bar{\eta}_H(\theta)$ and
\[
  \bigl\lVert \nabla_\theta J_{\mathrm{VS}}(\theta) - \nabla_\theta J_{\mathrm{ref}}(\theta) \bigr\rVert \;\le\; 2 B R_{\max}\, \bar{\eta}_H(\theta),
\]
where $\bar{\eta}_H(\theta) = \mathbb{E}\bigl[\sum_{t=1}^H \eta(s_t, a_t^\pi)\bigr]$.
\end{proposition}

Proposition~\ref{prop:vs-lower-bound} has the same form as Theorem~\ref{thm:modal-gradient}, with the reference-user gradient in place of the mode-user gradient, so when $P$ has non-trivial mass away from its mode and $\eta$ is small, VS moves training out of the near-deterministic-user regime that Corollary~\ref{cor:entropy-collapse} requires. Appendix~\ref{appendix:vs_theory} proves the bound, discusses the limits of the reference-recovery assumption, and gives a $\gamma$-sharpening result showing that direct aligned prompting suppresses tail behaviors exponentially, whereas VS preserves them.

\paragraph{\cotrain{}.}
Co-Training trains the user simulator alongside the policy, each with its own objective. The simulator is updated on its own turns of each conversation, so one rollout trains both sides. As the simulator's mode at each history shifts during training, a strategy tuned to an earlier mode stops working, and because $A_x$ is no longer fixed, the geometric concentration of Corollary~\ref{cor:entropy-collapse} does not apply (Appendix~\ref{appendix:proof-cotrain} formalizes this with exclusive-lead counters). For the shifting mode to stay useful, the simulator's reward must keep it in the informative-variation regime (Remark~\ref{rem:informative-variation}) and stop it from collapsing onto a different mode. \S\ref{sec:setup} gives the per-task rewards, and Appendix~\ref{appendix:reward_ablation} shows that purely adversarial and purely cooperative rewards both do worse. For binary rewards, the curriculum targets a success rate of about $0.5$, where the within-batch variance $p(1-p)$ peaks at $\sigma^2 = 0.25$ and group-relative advantages are most spread out.

\section{Experiments}
\label{sec:experiments}

The theory makes three predictions. \textbf{(Q1)} Single-simulator RL should collapse, with OOD eval peaking early and then declining as the policy concentrates on a mode-exploit strategy that does not transfer. \textbf{(Q2)} Each solution should recover the gradient signal by relaxing the assumption it targets in \S\ref{sec:cotrain}, with Verbalized Sampling widening the simulator's per-turn distribution and Co-Training moving the mode the policy would exploit. \textbf{(Q3)} Sampling the simulator from a pool of recent checkpoints should add a further gain over Co-Training, and the size of this gain should depend on how much real variation the pool preserves. We also ask \textbf{(Q4)} whether the gains on the LLM panel transfer to real users. Section~\ref{sec:results} answers Q1 and Q2, checks whether the same pattern holds for symmetric cooperation on CooperBench, and answers Q4 with a pre-registered human study on $\tau^2$-bench and P4G. Appendix~\ref{sec:ablations} answers Q3 by varying the pool size and the simulator reward.

\subsection{Setup}
\label{sec:setup}

\paragraph{Tasks.} In Persuasion for Good (P4G)~\citep{wang2019persuasion}, a persuader tries to convince a resistant donor to give to charity, with reward $r = \min(\text{donation}/2, 1)$. $\tau^2$-bench~\citep{barres_2-bench_2025} contains customer-service dialogues on retail and airline tasks, with binary success reported per split. In CooperBench~\citep{khatua_cooperbench_2026}, two coding agents coordinate on a multi-step task, also with binary success, and we use Qwen3.5-9B and Qwen3.5-27B because smaller models cannot complete the tasks. When the simulator is trained, its reward depends on the task. P4G uses its donation-based adversarial reward, and CooperBench uses its symmetric task-success reward. On $\tau^2$-bench the simulator receives a SPICE-style \emph{curriculum} reward~\citep{liu_spice_2025} that targets a within-batch variance of $\sigma^2 \approx 0.25$. This choice matters, because in our $\tau^2$-bench ablation adversarial and cooperative simulator rewards both collapse the simulator onto a new mode and lower eval reward (Appendix~\ref{appendix:reward_ablation}).

\paragraph{Evaluation.} We evaluate every 16 steps and keep the checkpoint with the best mean panel score. For P4G and $\tau^2$-bench the panel has six simulators, three seen during training and three from unseen model families. CooperBench evaluates the policy in symmetric cross-play against Claude Haiku 4.5, GPT-5, and Gemini-3-Flash. Because the panel mean includes the training simulators, selecting on it leaks some training-simulator signal into checkpoint selection (Appendix~\ref{appendix:hyperparameters} lists the role of each panel model).

\paragraph{Training.} All methods use the policy update from \S\ref{sec:background} and are implemented in our framework \textsc{SCOPE}. In Co-Training, the same rollout updates both sides, with the simulator trained on the task-specific reward above, and Population Co-Training also samples the active simulator from a pool of recent checkpoints (Appendix~\ref{appendix:cotrain_framework} specifies the buffer, and Appendix~\ref{appendix:prompts} gives every prompt template). Per-step training compute therefore differs by method. The frozen-simulator methods (VS, Ensemble, Persona-Guided) update only the agent, while Co-Training and Population Co-Training update both sides. Tables~\ref{tab:user_sim_results} and~\ref{tab:cooper_results} compare methods at a matched number of optimizer steps rather than at matched compute, and Appendix~\ref{appendix:hyperparameters} lists the per-step compute multipliers and total GPU-hours.

\begin{table*}[t]
\centering
\caption{\textbf{User-simulator setting (P4G, $\tau^2$-bench).} Best evaluation-panel score across training, with the standard deviation over the six panel simulators as subscript. P4G reports reward $r=\min(\mathrm{donation}/2, 1)$, and $\tau^2$-bench reports success rate (\%) on each split. \textbf{Bold} marks the best and \underline{underline} the second-best method for each model. Persona-Guided is N/A on P4G because the P4G task already conditions the simulator on personas. The frozen ensemble uses $K{=}3$ to keep its cost comparable to the other methods, while Population Co-Training uses $K{=}5$ by default, and Appendix~\ref{sec:ablations} sweeps $K \in \{1,3,5,10\}$. Appendix~\ref{appendix:hyperparameters} breaks down compute and memory by method.}
\label{tab:user_sim_results}
\resizebox{0.92\textwidth}{!}{%
\begin{tabular}{ll ccc}
\toprule
& & \cellcolor{blue!8}\textbf{P4G} & \multicolumn{2}{c}{\cellcolor{green!8}\textbf{$\tau^2$-bench}} \\
\cmidrule(lr){4-5}
\textbf{Model} & \textbf{Method} & Reward $\uparrow$ & Retail $\uparrow$ & Airline $\uparrow$ \\
\midrule
\multirow{8}{*}{Qwen3-4B-Instruct}
 & Base                        & 0.216$_{\pm 0.03}$           & 40.4$_{\pm 3.1}$           & 24.0$_{\pm 2.8}$ \\
 & RL (Single)                 & 0.275$_{\pm 0.12}$           & 46.1$_{\pm 5.2}$           & 29.8$_{\pm 4.3}$ \\
 & + Persona-Guided            & N/A                          & 49.2$_{\pm 3.2}$           & 31.6$_{\pm 3.0}$ \\
 & + Ensemble Models ($K{=}3$) & 0.394$_{\pm 0.04}$           & 57.1$_{\pm 3.4}$           & 40.1$_{\pm 3.6}$ \\
\cmidrule(lr){2-5}
 & + \vs{}                     & \underline{0.484}$_{\pm 0.05}$ & 55.5$_{\pm 3.8}$         & 36.9$_{\pm 3.5}$ \\
 & + \cotrain{}                & 0.438$_{\pm 0.05}$           & \underline{60.5}$_{\pm 3.9}$ & \underline{44.4}$_{\pm 4.0}$ \\
 & + \popcotrain{}             & \textbf{0.508}$_{\pm 0.04}$  & \textbf{62.2}$_{\pm 3.6}$  & \textbf{45.7}$_{\pm 3.7}$ \\
\midrule
\multirow{8}{*}{Qwen3-8B}
 & Base                        & 0.253$_{\pm 0.03}$           & 48.1$_{\pm 3.0}$           & 30.2$_{\pm 2.9}$ \\
 & RL (Single)                 & 0.342$_{\pm 0.11}$           & 52.5$_{\pm 6.1}$           & 35.2$_{\pm 5.9}$ \\
 & + Persona-Guided            & N/A                          & 55.7$_{\pm 3.1}$           & 37.4$_{\pm 3.0}$ \\
 & + Ensemble Models ($K{=}3$) & 0.450$_{\pm 0.04}$           & 62.4$_{\pm 3.3}$           & 43.6$_{\pm 3.5}$ \\
\cmidrule(lr){2-5}
 & + \vs{}                     & \textbf{0.587}$_{\pm 0.05}$  & 60.7$_{\pm 3.5}$           & 40.2$_{\pm 3.4}$ \\
 & + \cotrain{}                & 0.556$_{\pm 0.05}$           & \underline{66.1}$_{\pm 3.7}$ & \underline{48.2}$_{\pm 3.8}$ \\
 & + \popcotrain{}             & \underline{0.568}$_{\pm 0.05}$ & \textbf{67.9}$_{\pm 3.4}$ & \textbf{49.7}$_{\pm 3.6}$ \\
\bottomrule
\end{tabular}%
}
\end{table*}
\paragraph{Baselines.} We compare six paradigms under a matched setup. \emph{Base} is the untrained policy, \emph{RL (Single)} follows the standard recipe and trains against GPT-5-mini, and \emph{Ensemble Models} cycles through $K{=}3$ frozen simulators from different model families. \emph{Verbalized Sampling}~\citep{zhang2025verbalizedsamplingmitigatemode} draws each simulator reply from the simulator's verbalized response distribution, \emph{Co-Training} trains a separate simulator alongside the policy, and \emph{Population Co-Training} (ours) draws that simulator from a pool of recent checkpoints instead of using only the latest one. We also include a \emph{Persona-Guided} baseline~\citep{abdulhai2026hierarchical}, which conditions a single GPT-5-mini simulator on a persona for each rollout. We omit it on P4G, where the task already conditions the simulator on personas.

\subsection{Main Results}
\label{sec:results}

\begin{figure*}[t]
\centering
\includegraphics[width=\textwidth]{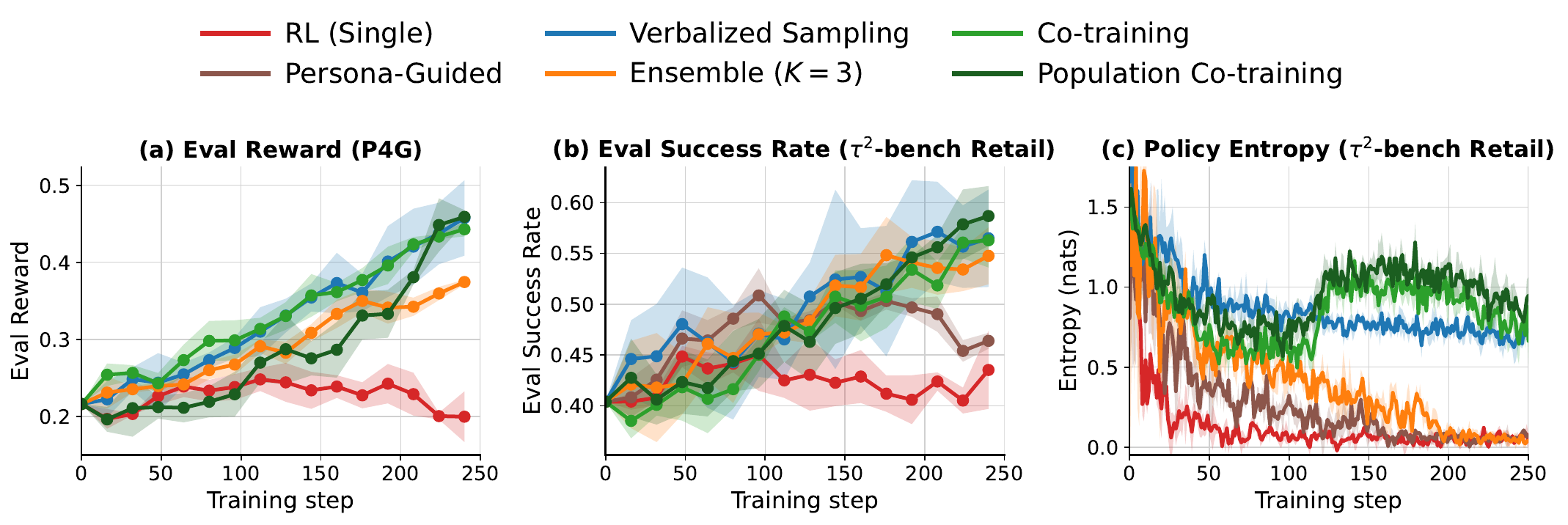}
\caption{\textbf{Both the inference-time and the training-time solutions keep policy entropy from collapsing.} \textbf{(a)} P4G eval reward over training. \textbf{(b)} $\tau^2$-bench Retail eval success rate. \textbf{(c)} $\tau^2$-bench Retail policy entropy. RL (Single) rises briefly and then falls back toward the untrained baseline in both eval plots, while its entropy drops to near zero. \vs{}, \cotrain{}, and \popcotrain{} improve eval and keep entropy well above zero, and the two \cotrain{} variants show an entropy jump partway through training that comes from simulator updates. Ensemble ($K{=}3$) also improves eval, although its entropy decays late in training. Persona-Guided (\citealp{abdulhai2026hierarchical}; $\tau^2$-bench only) falls between RL (Single) and the population methods. Curves average three seeds, and eval shading is $\pm 1\sigma$ over the six evaluator models.}
\label{fig:main_curves}
\end{figure*}

\paragraph{Q1: Does single-simulator RL collapse?} It does, and the collapse is visible in the shape of the training curves more than in the headline numbers. Table~\ref{tab:user_sim_results} reports the best held-out checkpoint, and for RL (Single) this checkpoint is clearly above Base in every cell (for Qwen3-4B-Instruct, $46.1$ vs.\ $40.4$ on $\tau^2$-Retail, $29.8$ vs.\ $24.0$ on Airline, and $0.275$ vs.\ $0.216$ on P4G). Figure~\ref{fig:main_curves} shows that on $\tau^2$-Retail this checkpoint is a transient peak, and performance falls back toward Base before training ends (Appendix Figure~\ref{fig:appendix_curves_tau2} shows the same pattern). By the end of training, RL (Single) and Persona-Guided are within a few points of Base on both $\tau^2$ splits; their best-checkpoint values in Table~\ref{tab:user_sim_results} describe a short-lived peak rather than where training settles. The population methods are still improving when training ends, so for them the best checkpoint is close to the final one. The collapse is milder on P4G at both model sizes, likely because the continuous donation reward preserves within-batch variance, but the curve still declines from its peak (Appendix Figure~\ref{fig:appendix_curves_p4g}). Since the evaluation panel consists of LLMs, we test transfer to real users separately (Q4).

\paragraph{Q2: Do the two solutions recover the training signal?} Both do. On Qwen3-4B-Instruct, Verbalized Sampling raises $\tau^2$-Retail from $46.1$ to $55.5$, Airline from $29.8$ to $36.9$, and P4G reward from $0.275$ to $0.484$, without retraining either model. This is the mechanism the theory predicts for VS: because the simulator is queried for a verbalized distribution at every turn, $\epsilon_\phi$ stays above zero and the modal-gradient term in Theorem~\ref{thm:modal-gradient} no longer dominates. Persona-Guided, which adds diversity only through the simulator's prompt, gives a smaller gain ($46.1\!\to\!49.2$ on Retail and $29.8\!\to\!31.6$ on Airline), so prompt-level diversity closes only part of the gap that the simulator-level solutions do. Co-Training targets the other assumption: it moves the mode the policy would lock onto from one training step to the next, so the geometric concentration of Corollary~\ref{cor:entropy-collapse} never has a fixed target. Retail and Airline rise further still, to $60.5$ and $44.4$. Population Co-Training improves on both solutions with $62.2$ and $45.7$, and it is the best method on every $\tau^2$ split at both model sizes and on P4G at 4B. The one exception is P4G with Qwen3-8B, where VS ($0.587$) scores slightly above both Co-Training ($0.556$) and Population Co-Training ($0.568$). We suspect that the continuous donation reward suits Verbalized Sampling well enough that diversifying a single simulator already captures most of the available signal.

\begin{table}[t]
\centering
\caption{\textbf{Symmetric cooperation (CooperBench).} Best held-out success rate (\%) across training against a panel of three partner models (Haiku 4.5, GPT-5, Gemini-3-Flash), taken from the eval curves in Appendix Figures~\ref{fig:appendix_curves_cooperbench_9b} and~\ref{fig:appendix_curves_cooperbench_27b}. \textbf{Bold} marks the best and \underline{underline} the second-best method for each model. $^*$The Qwen3.5-27B runs train LoRA adapters on Tinker, so their per-step compute differs from the Qwen3.5-9B runs. Because CooperBench is symmetric, the paradigms of Table~\ref{tab:user_sim_results} become Cross-play (frozen partner), Self-play (one model in both roles), and Population self-play (partners drawn from recent self-play checkpoints).}
\label{tab:cooper_results}
\begin{tabular}{l cc}
\toprule
\textbf{Method} & \cellcolor{orange!8}\textbf{Qwen3.5-9B} & \cellcolor{orange!8}\textbf{Qwen3.5-27B}$^*$ \\
                & Success $\uparrow$ & Success $\uparrow$ \\
\midrule
Base                              & 23.7                       & 47.8 \\
Cross-play (Haiku)                & 32.5                       & 57.9 \\
+ Cross-play Ensemble ($K{=}3$)   & \textbf{36.5}              & \textbf{62.3} \\
\cmidrule(lr){1-3}
+ \textcolor{CoGreen}{\textbf{Self-play}}                       & 29.7                       & 57.8 \\
+ \textcolor{PopGreen}{\textbf{Population self-play}} (ours)     & \underline{33.3}           & \underline{61.5} \\
\bottomrule
\end{tabular}
\end{table}

\paragraph{Symmetric cooperation on CooperBench.} P4G and $\tau^2$-bench train a policy against a user simulator. CooperBench asks whether the contrast between a frozen partner and a co-evolving one also holds in symmetric cooperation, where both agents play the same role. Table~\ref{tab:cooper_results} reports held-out success against a panel of three partner models (Claude Haiku 4.5, GPT-5, and Gemini-3-Flash). Here the picture is mixed. Cross-play with the frozen three-model ensemble is the strongest method at both scales ($36.5$ and $62.3$), and population self-play is second ($33.3$ and $61.5$), ahead of cross-play with a single Haiku partner ($32.5$ and $57.9$) and of plain self-play ($29.7$ and $57.8$). The training curves show why. Cross-play with a single partner plateaus early, while the two self-play variants first dip below the base model and then improve steadily in the second half of training, with the population variant ahead of plain self-play at most eval points and still improving when training ends. On this task, a diverse set of frozen partners therefore works at least as well as a co-evolving one within our training budget, and only the population variant of self-play beats a single frozen partner at both scales. Appendix Figures~\ref{fig:appendix_curves_cooperbench_9b} and~\ref{fig:appendix_curves_cooperbench_27b} show the training dynamics for Qwen3.5-9B and Qwen3.5-27B.

\paragraph{Q4: Do the gains transfer to real users?} We run a human study on $\tau^2$-bench (retail split) and Persuasion for Good that compares three conditions per task (Base, RL (Single), and +\cotrain{}) with $N{=}40$ Prolific participants per cell. Each session records the task outcome, and participants then rate the dialogue in a survey. As Table~\ref{tab:human_study_main} shows, Co-Training improves over RL (Single) on both the $\tau^2$ task outcome ($0.43 \to 0.70$) and P4G dialogue naturalness ($3.21 \to 4.45$), and both differences are significant at $p<0.01$. Appendix~\ref{appendix:human_study} gives the full design, sample sizes, and analysis plan.

\paragraph{Further experiments.} Appendix~\ref{appendix:extended_experiments} reports the remaining runs: per-benchmark training, eval, and entropy curves for every method (Appendix~\ref{appendix:full_curves}), the batch-level training diagnostics behind \S\ref{exp:fixed_policy} (Appendix~\ref{appendix:training_dynamics}), Verbalized Sampling against the larger GPT-5 simulator (Appendix~\ref{appendix:vs_gpt5}), and a full replication with Olmo-3-7B-Instruct as the trainable agent (Appendix~\ref{appendix:training_olmo}).
\begin{table}[h]
\centering
\scriptsize
\setlength{\tabcolsep}{3.5pt}
\renewcommand{\arraystretch}{1.08}
\vspace{-1em}
\caption{\textbf{Q4: Human study on $\tau^2$-bench and Persuasion for Good.}
Columns report the $\tau^2$-bench task outcome from the objective evaluator (in $[0,1]$),
$\tau^2$-bench overall satisfaction on a 1--7 Likert scale, the P4G donation amount, and P4G
overall satisfaction.
\textbf{Bold} marks the best and \underline{underline} the second-best value in each column.
$^*\,p{<}0.05$, $^{**}\,p{<}0.01$ vs.\ RL (Single).}
\label{tab:human_study_main}
\resizebox{0.75\linewidth}{!}{%
\begin{tabular}{lcccc}
\toprule
& \multicolumn{2}{c}{\cellcolor{green!8}\textbf{$\tau^2$-bench}} & \multicolumn{2}{c}{\cellcolor{blue!8}\textbf{P4G}} \\
\cmidrule(lr){2-3} \cmidrule(lr){4-5}
\textbf{Method}
& Task $\uparrow$
& Natural $\uparrow$
& Donation (\$) $\uparrow$
& Natural $\uparrow$ \\
\midrule
Base
& 0.41$_{\pm 0.22}$
& 4.77$_{\pm 1.89}$
& \underline{0.51}$_{\pm 0.60}$
& \underline{3.93}$_{\pm 1.91}$ \\

RL (Single)
& \underline{0.43}$_{\pm 0.31}$
& \underline{5.11}$_{\pm 1.63}$
& 0.46$_{\pm 0.72}$
& 3.21$_{\pm 1.76}$ \\

+\,\cotrain{}
& \textbf{0.70}$_{\pm 0.43}{}^{\mathbf{**}}$
& \textbf{5.50}$_{\pm 1.81}$
& \textbf{0.69}$_{\pm 0.64}$
& \textbf{4.45}$_{\pm 1.67}{}^{\mathbf{**}}$ \\
\bottomrule
\end{tabular}%
}
\vspace{-1em}
\end{table}

\section{Related Work}
\label{sec:related}

\paragraph{Mode collapse in LLMs.}
RLHF narrows the output distributions of LLMs~\citep{jiang_artificial_2025, zhang2025noveltybenchevaluatinglanguagemodels, zhang2025verbalizedsamplingmitigatemode, yang2026llmprobabilityconcentrationalignment}, and \citet{gxchen2025klregularized} prove that KL-regularized RL has a unimodal optimum by construction. Similar effects appear as reasoning collapse in agentic RL~\citep{wang2026ragen2reasoningcollapseagentic} and persona collapse in role-play~\citep{xiao2026chameleonslimitinvestigatingpersona}. These works study collapse in a model's own outputs. In simulator collapse, the collapsed LLM is instead the training environment, and its narrowness biases the gradient of the policy trained against it.

\paragraph{LLM-based user simulation for RL.}
LLM simulators serve as training environments for dialogue and agentic RL~\citep{anthis2025llmsocialsimulationspromising, abdulhai_consistently_2025, qian_userrl_2025, zhao_mua-rl_2025, sun_training_2025, gandhi_learning_2026, abdulhai2026hierarchical} and are increasingly used for evaluation~\citep{barres_2-bench_2025, zhou_tom-swe_2025}. \citet{naous_flipping_2025} find that stronger assistants make worse simulators, and simulators also differ from real users in their preferences~\citep{zhou2026sim2real} and behavior~\citep{mehri2026measuring} and inherit a homogeneous, cooperative bias from their base models~\citep{chopra2026beyondcooperativesimulators, suh2026quantifyingutilityusersimulators}. Existing mitigations use behavioral taxonomies~\citep{chen2025ncsim}, theory-of-mind objectives~\citep{zhou_tom-swe_2025}, curiosity rewards~\citep{wan2025enhancingpersonalizedmultiturndialogue}, finer credit assignment~\citep{yu_sotopia-rl_2025, qian_userrl_2025}, evolved persona generators~\citep{chopra2026beyondcooperativesimulators}, or inference-time methods~\citep{yang2026multiuserlargelanguagemodel}, but all of them optimize against a static simulator distribution. We trace the failure of training against such a distribution to a mechanism on the policy side (Theorem~\ref{thm:modal-gradient}, Corollary~\ref{cor:entropy-collapse}), and Co-Training replaces the static simulator with a partner that co-evolves with the policy.

\paragraph{Multi-agent RL and co-training.}
Single-model self-play has produced strong results in games~\citep{silver2016mastering, berner2019dota} and, for LLMs, in text games~\citep{liu_spiral_2025}, corpus-grounded interaction~\citep{liu_spice_2025}, and reasoning~\citep{zhao_absolute_2025}. \citet{liao2024efficacylanguagemodelselfplay} report that self-play reaches a diversity ceiling, which motivates co-training two separate models~\citep{ma_cory_2024, feng_drmas_2026, acikgoz_toolr0_2026}. Prior multi-turn RL on collaborative reasoning~\citep{zhou_sweet-rl_2025, hong_natural_2025} and social tasks~\citep{yu_sotopia-rl_2025, tomlin_characterizing_2025} uses short horizons of one to three rounds. We apply co-training to long-horizon dialogue and sample the partner from a pool of recent checkpoints.
\section{Discussion and Conclusion}

LLM simulators are often mode-collapsed, and in multi-agent RL a policy trained against one of them inherits that narrowness, a failure we call simulator collapse. The policy concentrates on the strategy that wins against the simulator's mode, its entropy drops, and the resulting low-entropy policy transfers poorly to unseen simulators and to real users.

We formalize simulator collapse and trace it to the training environment rather than to the RL algorithm. We then examine two solutions, \vs{} for a frozen simulator and \cotrain{} for a trainable one, and release both in \textsc{SCOPE}, an open framework that supports multi-model rotation, self-play, and dual-model Co-Training through one interface. On Persuasion for Good and $\tau^2$-bench, the held-out success of single-simulator RL peaks early and falls back toward the untrained baseline by the end of training. Both solutions close most of this gap, and Population Co-Training reaches the highest held-out task success in all but one setting. On CooperBench, population self-play does better than cross-play with a single frozen partner and comes close to cross-play with a frozen ensemble, which is the strongest method there. In the human study, Co-Training also improves task outcome on $\tau^2$-bench and dialogue naturalness on P4G over single-simulator RL. Both solutions touch only the training environment and leave the policy update as it is, so they drop into an existing pipeline with little effort.

All three settings are text-only, English, and two-agent, and apart from the human study our evaluation relies on an LLM panel. Whether the mechanism extends to populations of more than two agents, to multimodal environments, or to other languages remains open. Appendix~\ref{appendix:limitations} discusses limitations, broader impact, and future work.

\begin{ack}
We thank Modal and Tinker for providing the research grants and resources that supported this work. We thank the Schmidt Sciences AI2050 Program for its support. Opinions, findings, conclusions, or recommendations expressed in this paper are those of the authors and do not necessarily reflect the views of the funding agency. We also thank the members of CHATS-lab for their feedback on the paper.
\end{ack}

\bibliographystyle{unsrtnat}
\bibliography{references}
\appendix

\clearpage
\section{Limitations, Broader Impact, and Future Work}
\label{appendix:limitations}

\subsection{Limitations}
\label{appendix:limitations_only}
\paragraph{Fixed pool.} The diversity of the frozen pool is bounded by the models we draw from, and the pool stays fixed during training. We leave adaptive curation of the pool to future work.

\paragraph{LLM evaluation panel.} Our evaluation panel consists of aligned LLMs that share RLHF-induced biases with the training simulators~\citep{zhou2026sim2real}. To test transfer to real users directly, we run a pre-registered human study on $\tau^2$-bench and Persuasion for Good (Appendix~\ref{appendix:human_study}). The study compares Co-Training with the baselines for Qwen3-4B-Instruct only.

\paragraph{Task-specific simulator reward.} Co-Training depends on a simulator reward that preserves variation across checkpoints (Appendix~\ref{appendix:reward_ablation}). We give one such reward, which works on the benchmarks we tested, but we have not mapped which other rewards would also work.

\paragraph{Compute overhead.} Both solutions cost more than single-simulator RL. Verbalized Sampling adds rollout time, and Co-Training adds the most because it updates two models on the same rollouts (Appendix~\ref{appendix:hyperparameters}). Some of this overhead is hard to avoid. A frozen, mode-collapsed simulator is unlikely to produce a policy that transfers to unseen partners, so any solution has to change the simulator's distribution somewhere in the training loop. A smaller simulator pool, or Verbalized Sampling amortized across turns, could reduce the cost.

\subsection{Broader Impact}
\label{appendix:broader_impact}
We release \textsc{SCOPE}, an open framework for population-based multi-agent RL that supports heterogeneous simulator rotation, self-play, and Co-Training through one interface. Our results suggest that multi-agent RL depends on a diverse set of training partners as much as on the quality of any single simulator. Making that practical calls for infrastructure that brings the cost of population-based training closer to the cost of training against one simulator. The analysis in \S\ref{sec:theory_analysis} describes a failure mode that agentic RL pipelines can check for, and it can guide the choice of simulator and simulator reward before a training run.

\subsection{Future Work}
\label{appendix:future_work}
We see three directions for future work. First, the simulator population could be made adaptive, with a learned curator that chooses which past checkpoints enter the buffer based on signals from training. Second, Co-Training works only with a simulator reward that preserves variation (Appendix~\ref{sec:ablations}), so the simulator reward could itself be learned, for example by meta-learning a reward that maximizes disagreement across checkpoints. Third, extending \textsc{SCOPE} to populations of three or more agents and to tasks that mix cooperation and competition would test how far the simulator-collapse mechanism generalizes.

\clearpage
\section{Theory: Simulator Collapse}
\label{appendix:proof}

This appendix proves the results of Section~\ref{sec:theory_analysis} in the order of the main text. Simulator collapse turns the policy gradient into a mode-user gradient (Appendix~\ref{appendix:proof-modal-gradient}), and the same coupling removes the simulator-side reward variance (Appendix~\ref{appendix:proof-user-variance}). Group-relative updates then concentrate policy mass on a mode-exploit set (Appendix~\ref{appendix:proof-log-odds}), and the trained policy fails when held-out users produce response types the simulator rarely produced (Appendix~\ref{appendix:proof-coverage}). The mixture-gradient bound for Population Co-Training follows by linearity (Appendix~\ref{appendix:proof-population}). Appendix~\ref{appendix:sharpening} restates the $\gamma$-sharpening result of \citet{zhang2025verbalizedsamplingmitigatemode}, which motivates Definition~\ref{def:conditional-collapse} but is not used in any proof.

\subsection{Notation summary}
\label{appendix:notation}

The table lists the quantities used most often in \S\ref{sec:theory} and this appendix.

\begin{center}
\begin{tabular}{@{}ll@{}}
\toprule
\textbf{Symbol} & \textbf{Meaning} \\
\midrule
$a_\phi^\star(s, a^\pi)$ & simulator's mode at $(s, a^\pi)$ \\
$\epsilon_\phi(s, a^\pi)$ & per-turn collapse error: $1 - \phi_\psi(a_\phi^\star \mid s, a^\pi)$ \\
$\epsilon^\star$ & collapse threshold (Definition~\ref{def:conditional-collapse}) \\
$\bar{\epsilon}_H(\theta)$ & accumulated collapse error along an $H$-turn rollout \\
$b, \ell_{\mathrm{turn}}, B$ & per-token / per-turn / trajectory score bounds; $B \le b\,\ell_{\mathrm{turn}}\,H$ \\
\midrule
$Y, q_k(y \mid x)$ & strategy abstraction; strategy distribution at update $k$ \\
$A_x, \Delta_x, g_x$ & mode-exploit set, mode-exploit gap, geometric concentration rate \\
\midrule
$p^{\mathrm{VS}}_\phi, P$ & verbalized simulator distribution; pre-RLHF reference distribution \\
$\eta(s, a^\pi), \bar{\eta}_H(\theta)$ & per-turn and accumulated reference-recovery TV error (Proposition~\ref{prop:vs-lower-bound}) \\
$\rho, m, \lambda$ & reference-mass on $B$, modal mass, per-behavior bound on $B$ (Proposition~\ref{prop:tail-suppression}) \\
\midrule
$\Phi, \bar{\phi}, m_k$ & checkpoint buffer, mixture simulator, per-checkpoint peak mass \\
$\bar{\epsilon}_H^{\Phi}$ & accumulated collapse error under population mixing \\
\bottomrule
\end{tabular}
\end{center}

\subsection{Preliminaries}
\label{appendix:prelim}

We work in the POMDP setup of Section~\ref{sec:background}. Each trajectory $\tau = (s_0, a_0^\pi, a_0^\phi, \ldots, s_T)$ has terminal reward $R(\tau) \in [0, R_{\max}]$. Write $P_\phi^\theta(\tau)$ for the trajectory distribution under $(\pi_\theta, \phi)$ and $J_\phi(\theta) = \mathbb{E}_{\tau \sim P_\phi^\theta}[R(\tau)]$.

\paragraph{Policy-gradient identity.}
By REINFORCE~\citep{Williams2004SimpleSG},
\begin{equation}
    \nabla_\theta J_\phi(\theta) \;=\; \mathbb{E}_{\tau \sim P_\phi^\theta}\!\bigl[R(\tau)\, S_\theta(\tau)\bigr], \qquad S_\theta(\tau) \;=\; \sum_{\ell \in \mathcal{I}_\pi(\tau)} \nabla_\theta \log \pi_\theta(a_\ell \mid h_\ell),
\label{eq:pg-identity}
\end{equation}
where $\mathcal{I}_\pi(\tau)$ indexes the agent-decision tokens. With trajectory length $L$ and per-token score bound $\|\nabla_\theta \log \pi_\theta(a_\ell \mid h_\ell)\| \le b$, the trajectory score satisfies $\|S_\theta(\tau)\| \le L b =: B$. Each turn produces at most $\ell_{\mathrm{turn}}$ agent-decision tokens. Then $L \le \ell_{\mathrm{turn}} H$, and the trajectory-score bound $B = Lb \le b\,\ell_{\mathrm{turn}}\, H$ grows linearly in the horizon $H$. In our training runs, gradient clipping at norm $1.0$ plays the role of this bound (Table~\ref{tab:shared_hyperparameters}).

\paragraph{Trajectory TV under maximal coupling.}
For any two trajectory distributions $P, Q$ on $\mathcal{T}$ there exists a coupling $\nu$ on $\mathcal{T} \times \mathcal{T}$ such that $\Pr_{(\tau, \tau') \sim \nu}[\tau \ne \tau'] = D_{\mathrm{TV}}(P, Q)$. We apply this turn by turn. At simulator turn $t$, given a matched prefix, the maximal coupling between $\phi_\psi(\cdot \mid s_t, a_t^\pi)$ and $\delta_{a_\phi^\star(s_t, a_t^\pi)}$ disagrees with probability $\epsilon_\phi(s_t, a_t^\pi)$ (Definition~\ref{def:conditional-collapse}).

\paragraph{Bounded-integrand TV lemma.}
For any bounded vector-valued $f$ with $\|f\|_\infty \le M$ and any two probability measures $P, Q$,
\begin{equation}
    \bigl\lVert \mathbb{E}_P[f] - \mathbb{E}_Q[f] \bigr\rVert \;\le\; 2 M\, D_{\mathrm{TV}}(P, Q).
\label{eq:tv-bound-bdd-fn}
\end{equation}
This follows from the variational form of TV applied coordinate-wise.

\paragraph{KL-regularized RL closed form.}
For the objective $\max_\pi \mathbb{E}_\pi[R] - \beta\, \mathrm{KL}(\pi \,\|\, \pi_{\mathrm{ref}})$, the optimum is $\pi_\beta^\star(y) \propto \pi_{\mathrm{ref}}(y) \exp(R(y)/\beta)$~\citep{rafailov2024directpreferenceoptimizationlanguage, gxchen2025klregularized}. We use this only in Appendix~\ref{appendix:sharpening} to motivate Definition~\ref{def:conditional-collapse}.

\begin{remark}[Behavioral collapse as action-space coarsening]
\label{rem:behavioral-coarsening}
In Definition~\ref{def:conditional-collapse}, $\epsilon_\phi$ is defined over the simulator's token-action distribution $\phi_\psi(\cdot \mid s, a^\pi)$ and measures concentration around the modal token sequence $a_\phi^\star$. A behavioral reading coarsens this action space, for example by projecting onto dialogue-act labels or strategy clusters. For any measurable projection $\Pi$ of the action space, the TV distance from $\Pi_*\phi_\psi$ to the corresponding modal point mass is non-increasing, and the gradient-bias bound of Theorem~\ref{thm:modal-gradient} still applies to the coarsened distribution. We do not fix a specific coarsening, and Appendix~\ref{appendix:collapse_examples} describes the behavioral patterns we have in mind.
\end{remark}

\subsection{Simulator Mode Collapse via \texorpdfstring{$\gamma$}{gamma}-Sharpening (Motivation)}
\label{appendix:sharpening}

\citet{zhang2025verbalizedsamplingmitigatemode} trace mode collapse in aligned LLMs to typicality bias in RLHF. Aligned LLMs inherit a preference for high-likelihood responses from their reward models, and KL-regularized RL compounds that preference exponentially. We restate their observation as a special case of the closed-form optimum of KL-regularized RL.

\begin{proposition}[$\gamma$-Sharpening, following \citealt{zhang2025verbalizedsamplingmitigatemode}]
\label{prop:sharpening}
Let simulator $\phi$ be trained via KL-regularized RLHF with reference $\phi_{\mathrm{ref}}$, KL penalty $\beta > 0$, and reward model
\begin{equation}
r_\phi(s, y) \;=\; r_{\mathrm{true}}(s, y) + \alpha \log \phi_{\mathrm{ref}}(y \mid s) + \epsilon(s)
\label{eq:reward_decomp}
\end{equation}
for some typicality-bias weight $\alpha > 0$, true-quality term $r_{\mathrm{true}}$, and prompt-specific offset $\epsilon(s)$. Then
\begin{equation}
\phi^*(y \mid s) \;\propto\; \phi_{\mathrm{ref}}(y \mid s)^{\gamma} \cdot \exp\!\bigl(r_{\mathrm{true}}(s, y)/\beta\bigr), \qquad \gamma \;:=\; 1 + \alpha/\beta \;>\; 1.
\label{eq:sharpened_sim}
\end{equation}
\end{proposition}

\begin{proof}
Substituting~\eqref{eq:reward_decomp} into the closed-form KL-regularized optimum $G_\beta(y) \propto \phi_{\mathrm{ref}}(y \mid s)\exp(r_\phi(s,y)/\beta)$ (Appendix~\ref{appendix:prelim}),
\begin{equation}
\phi^*(y \mid s) = \tfrac{e^{\epsilon(s)/\beta}}{Z(s)} \phi_{\mathrm{ref}}(y \mid s)^{1+\alpha/\beta} \exp\!\bigl(r_{\mathrm{true}}(s, y)/\beta\bigr).
\end{equation}
Absorbing the prompt-specific factor $e^{\epsilon(s)/\beta}$ into the partition function gives~\eqref{eq:sharpened_sim}. Empirical estimates of the typicality bias put $\alpha \approx 0.5$--$0.65$~\citep{zhang2025verbalizedsamplingmitigatemode}. Standard values $\beta \in [0.01, 0.1]$ then give $\gamma \in [6, 66]$, and at such $\gamma$ the mode carries almost all the mass of $\phi^*$.
\end{proof}

\citet{gxchen2025klregularized} show that KL-regularized RL concentrates on unimodal solutions even without a typicality bias. Proposition~\ref{prop:sharpening} covers the case with such a bias. The proofs below use only the measurable condition that $\epsilon_\phi(s_t, a_t^\pi)$ is small at the simulator turns visited during training (Definition~\ref{def:conditional-collapse}).

\subsection{Proof of Theorem~\ref{thm:modal-gradient}}
\label{appendix:proof-modal-gradient}

\paragraph{Step 1 (per-turn coupling).}
Couple the trajectory $\tau$ in $M_\phi$ and $\tau'$ in $M_{\mathrm{mode}}$ as follows. Both runs use the same task and the same agent randomness. At each simulator turn $t$, given a matched prefix, sample $(a_t^\phi, a_t^{\phi,\star})$ from the maximal coupling between $\phi_\psi(\cdot \mid s_t, a_t^\pi)$ and $\delta_{a_\phi^\star(s_t, a_t^\pi)}$. Under this coupling the two samples agree with probability $1 - \epsilon_\phi(s_t, a_t^\pi)$. After the first disagreement the prefixes differ, and the remaining simulator turns are sampled independently.

Let $D = \mathbb{I}[\tau \ne \tau']$. By the union bound over simulator turns,
\begin{equation}
\Pr[D = 1] \;\le\; \mathbb{E}\!\Bigl[\,\sum_{t=1}^{H} \epsilon_\phi(s_t, a_t^\pi)\Bigr] \;=\; \bar{\epsilon}_H(\theta).
\label{eq:union-bound-app}
\end{equation}
By the coupling characterization of TV,
\begin{equation}
D_{\mathrm{TV}}\!\bigl(P_\phi^\theta,\, P_{\mathrm{mode}}^\theta\bigr) \;\le\; \Pr[D = 1] \;\le\; \bar{\epsilon}_H(\theta),
\label{eq:tv-bound-app}
\end{equation}
which proves the trajectory-TV bound stated in Theorem~\ref{thm:modal-gradient}.

\paragraph{Step 2 (from trajectory TV to gradient norm).}
Apply~\eqref{eq:pg-identity} to both objectives:
\begin{equation*}
\nabla_\theta J_\phi(\theta) - \nabla_\theta J_{\mathrm{mode}}(\theta) \;=\; \mathbb{E}_{\tau \sim P_\phi^\theta}\!\bigl[R(\tau)\, S_\theta(\tau)\bigr] - \mathbb{E}_{\tau \sim P_{\mathrm{mode}}^\theta}\!\bigl[R(\tau)\, S_\theta(\tau)\bigr].
\end{equation*}
The integrand $f(\tau) = R(\tau)\, S_\theta(\tau)$ satisfies $\|f(\tau)\| \le R_{\max}\, \|S_\theta(\tau)\| \le R_{\max} B$. By the bounded-integrand TV bound~\eqref{eq:tv-bound-bdd-fn} with $M = R_{\max} B$ and the trajectory-TV bound from Step~1,
\begin{equation*}
\bigl\lVert \nabla_\theta J_\phi(\theta) - \nabla_\theta J_{\mathrm{mode}}(\theta) \bigr\rVert \;\le\; 2 R_{\max} B \cdot D_{\mathrm{TV}}\!\bigl(P_\phi^\theta, P_{\mathrm{mode}}^\theta\bigr) \;\le\; 2 B R_{\max}\, \bar{\epsilon}_H(\theta),
\end{equation*}
which is~\eqref{eq:modal-gradient-bound}. If $\epsilon_\phi(s_t, a_t^\pi) \le \epsilon$ on all visited turns, then $\bar{\epsilon}_H(\theta) \le H \epsilon$.
\hfill$\square$

\paragraph{Regime where the bound is informative.}
With $B \le b\,\ell_{\mathrm{turn}}\, H$ substituted, the gradient bound is at most $2\,b\,\ell_{\mathrm{turn}}\, R_{\max} \cdot H \cdot \bar{\epsilon}_H(\theta)$. If the per-turn collapse error is roughly constant in $t$, then $\bar{\epsilon}_H \approx \epsilon_{\mathrm{avg}} \cdot H$ and the bound grows as $O(H^2)$ at a fixed per-turn collapse rate. The horizon prefactor multiplying $\bar{\epsilon}_H$ is $30 \cdot b\,\ell_{\mathrm{turn}}\, R_{\max}$ for $\tau^2$-bench ($H=30$) and $50 \cdot b\,\ell_{\mathrm{turn}}\, R_{\max}$ for CooperBench ($H=50$). At equal per-turn collapse the CooperBench bound is looser. The bound is informative when $\bar{\epsilon}_H \ll 1$, or equivalently when the average per-turn deviation $\bar{\epsilon}_H/H$ falls below $1/H$. Figure~\ref{fig:training_dynamics} tracks this regime through the zero-variance batch fraction, which rises from $60\%$ to about $82\%$ during single-simulator training. Outside the regime the bound loosens. Verbalized Sampling and Co-Training both move training out of strong collapse by keeping $\epsilon_\phi$ above zero at every turn.

\subsection{Proof of Lemma~\ref{lem:user-variance}}
\label{appendix:proof-user-variance}

Fix $x$ and $\xi_\pi$. Under the coupling of Appendix~\ref{appendix:proof-modal-gradient}, with the policy randomness fixed at $\xi_\pi$, the simulator trajectory $\tau$ and the modal trajectory $\tau^\star$ agree except on an event of probability at most $\epsilon_H(x, \xi_\pi)$. When they agree, $R(\tau) = R(\tau^\star) =: c$, which is constant given $(x, \xi_\pi)$.

For any random variable $X \in [0, R_{\max}]$ that equals a constant $c$ except on an event $\mathcal{A}^c$ of probability at most $p$,
\begin{equation*}
\mathrm{Var}(X) \;\le\; \mathbb{E}\!\bigl[(X - c)^2\bigr] \;=\; \mathbb{E}\!\bigl[(X - c)^2\, \mathbb{I}[\mathcal{A}^c]\bigr] \;\le\; R_{\max}^2\, \Pr[\mathcal{A}^c] \;\le\; R_{\max}^2\, p.
\end{equation*}
Setting $X = R_x$ and $p = \epsilon_H(x, \xi_\pi)$ gives the bound in Lemma~\ref{lem:user-variance}. Taking the expectation over $\xi_\pi$ gives the marginal version $\mathbb{E}_{\xi_\pi}[\mathrm{Var}_{\xi_U}(R_x \mid x, \xi_\pi)] \le R_{\max}^2\,\mathbb{E}_{\xi_\pi}[\epsilon_H]$.
\hfill$\square$

\subsection{Proofs of Proposition~\ref{prop:log-odds} and Corollary~\ref{cor:entropy-collapse}}
\label{appendix:proof-log-odds}

Proposition~\ref{prop:log-odds} rests on one assumption about the mode user and one property of the policy-gradient step.

\begin{assumption}[Mode-exploit gap]
\label{assump:modal-gap}
For task $x$ there is a set $A_x$ and gap $\Delta_x > 0$ such that for every $y \in A_x$ and $y' \notin A_x$, $Q_{\mathrm{mode}}(x, y) \ge Q_{\mathrm{mode}}(x, y') + \Delta_x$.
\end{assumption}

The mode-user value $Q_{\mathrm{mode}}(x, y) = \mathbb{E}[R(\tau) \mid x, Y{=}y, M_{\mathrm{mode}}]$ is the expected return when the simulator always emits $a_\phi^\star$. Under the assumption, the collapsed simulator rewards a narrow family of agent strategies more than the alternatives, for example a fixed donation script in persuasion or a shortcut that extracts information from an overly helpful user in customer service.

We use the standard log-odds form of a KL-regularized softmax policy-gradient step. A positive estimated advantage increases the relative log probability, up to a bounded optimization error $\rho_x$:
\begin{equation}
    \log \tfrac{q_{k+1}(y \mid x)}{q_{k+1}(y' \mid x)} \;\ge\; \log \tfrac{q_k(y \mid x)}{q_k(y' \mid x)} + \eta\bigl(\widehat{Q}_\phi(x, y) - \widehat{Q}_\phi(x, y')\bigr) - \rho_x.
\label{eq:softmax-pg-logodds}
\end{equation}
Define $g_x = \eta(\Delta_x - 2 R_{\max} \epsilon_x - 2 \zeta_x) - \rho_x$, where $\epsilon_x$ and $\zeta_x$ bound the simulator-collapse and $Q$-estimation errors at task $x$.

\paragraph{Proof of Proposition~\ref{prop:log-odds}.}
Fix $y \in A_x$ and $y' \notin A_x$. Combining the assumed errors $|Q_\phi - Q_{\mathrm{mode}}| \le R_{\max} \epsilon_x$ and $|\widehat{Q}_\phi - Q_\phi| \le \zeta_x$ with the mode-exploit gap (Assumption~\ref{assump:modal-gap}),
\begin{align*}
\widehat{Q}_\phi(x, y) - \widehat{Q}_\phi(x, y') \;&\ge\; Q_\phi(x, y) - Q_\phi(x, y') - 2 \zeta_x \\
&\ge\; \bigl(Q_{\mathrm{mode}}(x, y) - Q_{\mathrm{mode}}(x, y')\bigr) - 2 R_{\max} \epsilon_x - 2 \zeta_x \\
&\ge\; \Delta_x - 2 R_{\max} \epsilon_x - 2 \zeta_x.
\end{align*}
Substituting into the log-odds step~\eqref{eq:softmax-pg-logodds} gives
\begin{equation}
\log \tfrac{q_{k+1}(y \mid x)}{q_{k+1}(y' \mid x)} \;\ge\; \log \tfrac{q_k(y \mid x)}{q_k(y' \mid x)} + g_x,
\qquad g_x := \eta(\Delta_x - 2 R_{\max} \epsilon_x - 2 \zeta_x) - \rho_x.
\label{eq:pairwise-logodds}
\end{equation}
Eq.~\eqref{eq:pairwise-logodds} holds for every pair $(y, y')$ with $y \in A_x$ and $y' \notin A_x$, so $q_{k+1}(y)/q_{k+1}(y') \ge e^{g_x}\, q_k(y)/q_k(y')$. Cross-multiplying and summing over all $y \in A_x$ and $y' \notin A_x$ gives
\begin{equation*}
q_{k+1}(A_x \mid x)\, q_k(A_x^c \mid x) \;\ge\; e^{g_x}\, q_k(A_x \mid x)\, q_{k+1}(A_x^c \mid x),
\end{equation*}
which is the set-level inequality
\begin{equation*}
\log \tfrac{q_{k+1}(A_x \mid x)}{q_{k+1}(A_x^c \mid x)} \;\ge\; \log \tfrac{q_k(A_x \mid x)}{q_k(A_x^c \mid x)} + g_x,
\end{equation*}
matching~\eqref{eq:log-odds-growth}. The update strictly increases $q_k(A_x \mid x)$ whenever $g_x > 0$.
\hfill$\square$

\paragraph{Proof of Corollary~\ref{cor:entropy-collapse}.}
Let $L_k = \log \tfrac{q_k(A_x \mid x)}{q_k(A_x^c \mid x)}$ and $g_x = \eta(\Delta_x - 2 R_{\max} \epsilon_x - 2 \zeta_x) - \rho_x > 0$ by assumption. Proposition~\ref{prop:log-odds} gives $L_{k+1} \ge L_k + g_x$, so by induction $L_k \ge L_0 + k g_x$. Converting log-odds back to mass with the logistic function $\sigma(t) = 1/(1 + e^{-t})$,
\begin{equation*}
q_k(A_x \mid x) \;=\; \sigma(L_k) \;\ge\; \sigma(L_0 + k g_x) \;=\; \frac{1}{1 + \tfrac{1 - q_0(A_x \mid x)}{q_0(A_x \mid x)}\, e^{-k g_x}},
\end{equation*}
which is~\eqref{eq:mass-concentration}. The right-hand side approaches $1$ geometrically in $k$ with rate $g_x$.
\hfill$\square$

\begin{remark}[Strategy entropy vs.\ token entropy]
\label{rem:strategy-entropy}
Corollary~\ref{cor:entropy-collapse} concerns strategy entropy, whereas Figure~\ref{fig:training_dynamics} plots token-level entropy. The two are linked by $H(S \mid x) = H(Y \mid x) + H(S \mid Y, x)$, where $S$ is the token sequence and $Y = f_x(S)$ is its strategy cluster. Strategy concentration implies a token-entropy drop only when the residual term $H(S \mid Y, x)$ is small, i.e., when the exploit strategies in $A_x$ are themselves low-entropy text patterns. Late in training, the rollouts within a batch become nearly word-for-word the same (Appendix~\ref{appendix:collapse_examples}), which fits this regime. A quantitative study of within-batch overlap is still missing.
\end{remark}

\subsection{Geometric concentration under Co-Training}
\label{appendix:proof-cotrain}

Corollary~\ref{cor:entropy-collapse} assumes that the mode-exploit set $A_x$ is fixed across training. Under Co-Training, $A_x^{(k)}$ shifts with each simulator update, and the log-odds growth in Proposition~\ref{prop:log-odds} applies only to the current set $A_x^{(k)}$. The net log-odds of a strategy $y$ over an alternative $y'$ then depend on how often $y$ is in the exploit set while $y'$ is not, and on how often the reverse holds.

\paragraph{Exclusive-lead counters.}
For a pair of strategies $(y, y')$, define the exclusive-lead counters after $K$ updates as
\[
N_K^+(y, y') = |\{k < K : y \in A_x^{(k)},\ y' \notin A_x^{(k)}\}|,
\qquad
N_K^-(y, y') = |\{k < K : y \notin A_x^{(k)},\ y' \in A_x^{(k)}\}|.
\]
$N_K^+$ counts the updates at which $y$ is in the exploit set and $y'$ is not. At these updates the log-odds of $y$ over $y'$ grow by at least $g_x$ (Proposition~\ref{prop:log-odds}). $N_K^-$ counts the reverse.

\begin{lemma}[Net log-odds under shifting exploit set]
\label{lem:cotrain-logodds}
Under the same softmax update as Proposition~\ref{prop:log-odds} with mode-exploit gap $\Delta_x$, the pairwise log-odds after $K$ updates satisfy
\[
\log\!\tfrac{q_K(y \mid x)}{q_K(y' \mid x)} \;\ge\; \log\!\tfrac{q_0(y \mid x)}{q_0(y' \mid x)} \;+\; g_x \cdot \bigl(N_K^+(y, y') - N_K^-(y, y')\bigr).
\]
\end{lemma}

\begin{proof}
At each step $k$ with $y \in A_x^{(k)}$ and $y' \notin A_x^{(k)}$, Proposition~\ref{prop:log-odds} gives $\log[q_{k+1}(y)/q_{k+1}(y')] \ge \log[q_k(y)/q_k(y')] + g_x$. At each step with $y \notin A_x^{(k)}$ and $y' \in A_x^{(k)}$, the symmetric inequality (swap $y, y'$ in Proposition~\ref{prop:log-odds}) gives a $-g_x$ contribution. Steps at which both or neither are in $A_x^{(k)}$ give no inequality. Summing over the $K$ updates gives the bound.
\end{proof}

\begin{corollary}[No exclusive lead, no concentration]
\label{cor:cotrain-no-concentration}
If for every pair $(y, y')$ the expected exclusive lead satisfies $\mathbb{E}[N_K^+(y, y') - N_K^-(y, y')] = o(K)$, then $\mathbb{E}\!\bigl[\log(q_K(y)/q_K(y'))\bigr] - \log(q_0(y)/q_0(y')) = o(K) \cdot g_x$, and the policy distribution $\{q_K\}$ cannot concentrate on any single strategy at geometric rate.
\end{corollary}

\paragraph{When Co-Training satisfies the no-concentration condition.}
One sufficient condition is that each simulator update redraws the exploit set independently, with the same marginal for every strategy. Under this i.i.d.\ shift, $\mathbb{E}[N_K^+] = \mathbb{E}[N_K^-]$ for every pair, and Corollary~\ref{cor:cotrain-no-concentration} applies with $\mathbb{E}[N_K^+ - N_K^-] = 0$. The informative-variation criterion (Remark~\ref{rem:informative-variation}) is the empirical condition under which we expect the i.i.d.\ model to be a reasonable approximation, since it keeps the simulator from re-collapsing on the same mode across steps. A direct test would measure the per-step exploit-set overlap $|A_x^{(k)} \cap A_x^{(k+1)}|$, which we have not attempted here.

\subsection{Proofs of Lemma~\ref{lem:finite-coverage} and Proposition~\ref{prop:coverage-regret}}
\label{appendix:proof-coverage}

\begin{lemma}[Finite-sample coverage of user behaviors]
\label{lem:finite-coverage}
For task $x$, let $B_x$ be a set of real-user behaviors that require a strategy outside $A_x$, and let $q_\phi(x) = P_\phi(B_x \mid x)$ be the simulator's mass on $B_x$. With $G$ independent simulator rollouts on task $x$, the probability that the group contains at least one behavior from $B_x$ is $1 - (1 - q_\phi(x))^G \le G\,q_\phi(x)$.
\end{lemma}

\paragraph{Proof of Lemma~\ref{lem:finite-coverage}.}
The $G$ rollouts on task $x$ produce independent simulator responses with $\Pr[Z_i \in B_x] = q_\phi(x)$. The probability that all $G$ miss $B_x$ is $(1 - q_\phi(x))^G$, so
\begin{equation*}
\Pr[\exists i \le G : Z_i \in B_x] \;=\; 1 - (1 - q_\phi(x))^G \;\le\; G\, q_\phi(x),
\end{equation*}
where the upper bound is Bernoulli's inequality $(1 - p)^G \ge 1 - G p$ for $p \in [0,1]$.
\hfill$\square$

\paragraph{Proof of Proposition~\ref{prop:coverage-regret}.}
Write $J_\star(y \mid x) = \mathbb{E}_{z \sim P_\star(\cdot \mid x)}[r_x(y, z)]$ and let $\hat\pi$ be the trained policy. By assumption $\Pr_{Y \sim \hat\pi(\cdot \mid x)}[Y \in A_x] \ge 1 - \alpha_x$.

\textit{Step 1 (per-strategy regret on the exploit set).} For any $y_m \in A_x$,
\begin{align*}
J_\star(y_b \mid x) - J_\star(y_m \mid x)
&= q_\star(x)\, \mathbb{E}_{z \in B_x}\!\bigl[r_x(y_b, z) - r_x(y_m, z)\bigr]
+ (1 - q_\star(x))\, \mathbb{E}_{z \notin B_x}\!\bigl[r_x(y_b, z) - r_x(y_m, z)\bigr] \\
&\ge q_\star(x)\, \Delta_x^{\mathrm{real}} - (1 - q_\star(x))\, \nu_x,
\end{align*}
using $r_x(y_b, z) - r_x(y_m, z) \ge \Delta_x^{\mathrm{real}}$ on $B_x$ and $\ge -\nu_x$ off $B_x$.

\textit{Step 2 (aggregation over the two components of the policy).}
Decompose
\begin{equation*}
J_\star(\hat\pi \mid x) \;=\; \Pr[\hat\pi \in A_x]\, \mathbb{E}_{Y \sim \hat\pi}\!\bigl[J_\star(Y \mid x) \mid Y \in A_x\bigr] \;+\; \Pr[\hat\pi \in A_x^c]\, \mathbb{E}_{Y \sim \hat\pi}\!\bigl[J_\star(Y \mid x) \mid Y \in A_x^c\bigr].
\end{equation*}
Step~1 gives $\mathbb{E}_{Y \sim \hat\pi}[J_\star(Y \mid x) \mid Y \in A_x] \le J_\star(y_b \mid x) - \bigl(\Delta_x^{\mathrm{real}}\, q_\star(x) - \nu_x(1 - q_\star(x))\bigr)$, and the off-exploit component is at most $R_{\max}$. Therefore
\begin{align*}
J_\star(\hat\pi \mid x)
&\le (1 - \alpha_x)\bigl[J_\star(y_b \mid x) - \Delta_x^{\mathrm{real}}\, q_\star(x) + \nu_x(1 - q_\star(x))\bigr] + \alpha_x R_{\max} \\
&\le J_\star(y_b \mid x) - (1 - \alpha_x)\bigl[\Delta_x^{\mathrm{real}}\, q_\star(x) - \nu_x(1 - q_\star(x))\bigr] + \alpha_x R_{\max},
\end{align*}
using $J_\star(y_b \mid x) \le R_{\max}$. Rearranging,
\begin{equation*}
J_\star(y_b \mid x) - J_\star(\hat\pi \mid x) \;\ge\; (1 - \alpha_x)\bigl[\Delta_x^{\mathrm{real}}\, q_\star(x) - \nu_x(1 - q_\star(x))\bigr] - \alpha_x R_{\max},
\end{equation*}
which is~\eqref{eq:deployment-regret}.
\hfill$\square$

The bound is informative when the trained policy concentrates on $A_x$ ($\alpha_x$ small), real users put non-trivial mass on $B_x$ ($q_\star(x)$ not too small), and the adaptive strategy gap $\Delta_x^{\mathrm{real}}$ is large relative to the off-$B_x$ penalty $\nu_x$. This is the situation we expect after collapse, where the trained policy relies on one strategy and an adaptive response does better whenever a real user departs from what that strategy assumes.

\subsection{Population Gradient and Coverage}
\label{appendix:proof-population}

This subsection extends the analysis of \S\ref{sec:cotrain} to the checkpoint pool of Population Co-Training. We prove the mixture-gradient bound and then give its coverage interpretation and the informative-variation criterion for the simulator reward.

\begin{proposition}[Mixture gradient averages over recent simulator modes]
\label{prop:population-gradient}
If each $\phi_k$ in the buffer is mode-collapsed around its own mode $a_k^\star$ along the training rollouts with accumulated error $\bar{\epsilon}_{H,k}(\theta)$, and $J_{\mathrm{mode},k}$ is the objective that always returns $a_k^\star$, then for the mixture $J_\Phi$ with weights $w_k = 1/K$,
\begin{equation}
    \bigl\lVert \nabla_\theta J_\Phi(\theta) - {\textstyle\sum_k} w_k \nabla_\theta J_{\mathrm{mode},k}(\theta) \bigr\rVert \;\le\; 2BR_{\max}\,{\textstyle\sum_k} w_k\,\bar{\epsilon}_{H,k}(\theta).
\label{eq:population-gradient-bound}
\end{equation}
\end{proposition}

\paragraph{Proof of Proposition~\ref{prop:population-gradient}.}
By linearity of the policy-gradient identity in the simulator distribution,
\begin{equation*}
\nabla_\theta J_\Phi(\theta) \;=\; \sum_{k=1}^{K} w_k\, \nabla_\theta J_{\phi_k}(\theta).
\end{equation*}
Applying Theorem~\ref{thm:modal-gradient} to each $\phi_k$,
\begin{equation*}
\bigl\lVert \nabla_\theta J_{\phi_k}(\theta) - \nabla_\theta J_{\mathrm{mode}, k}(\theta) \bigr\rVert \;\le\; 2 B R_{\max}\, \bar{\epsilon}_{H, k}(\theta).
\end{equation*}
The triangle inequality applied to $\sum_k w_k \bigl(\nabla J_{\phi_k} - \nabla J_{\mathrm{mode}, k}\bigr)$ gives
\begin{equation*}
\bigl\lVert \nabla_\theta J_\Phi(\theta) - \sum_{k} w_k\, \nabla_\theta J_{\mathrm{mode}, k}(\theta) \bigr\rVert \;\le\; \sum_{k} w_k\, \bigl\lVert \nabla_\theta J_{\phi_k}(\theta) - \nabla_\theta J_{\mathrm{mode}, k}(\theta) \bigr\rVert \;\le\; 2 B R_{\max} \sum_{k} w_k\, \bar{\epsilon}_{H, k}(\theta),
\end{equation*}
which is~\eqref{eq:population-gradient-bound}.
\hfill$\square$

\paragraph{Scope of the mixture bound.}
Proposition~\ref{prop:population-gradient} is a corollary of Theorem~\ref{thm:modal-gradient}. It says only that the population gradient is biased toward the average of the mode-user gradients, with the average accumulated collapse error bounding the bias. It does not show that this average target is easier to fit than any single mode. The diversity claim is Lemma~\ref{lem:population-collapse-error}, which bounds the mixture's per-turn collapse error from below by a $K$-dependent floor.

\begin{lemma}[Population mixing raises per-turn collapse error]
\label{lem:population-collapse-error}
Let the buffer $\Phi = \{\phi_k\}_{k=1}^{K}$ induce conditional distributions $\phi_k(\cdot \mid s)$ at state $s$ with peak masses $m_k(s) := \max_a \phi_k(a \mid s)$, and let $\bar{\phi} = \sum_k w_k\, \phi_k$ be the mixture. The mixture's peak mass satisfies
\begin{equation}
\max_a \bar{\phi}(a \mid s) \;\le\; \sum_{k=1}^{K} w_k\, m_k(s),
\label{eq:mixture-peak-mass}
\end{equation}
with equality only when all $\phi_k$ peak at a common action. In the fully-collapsed disjoint-mode case ($\phi_k$ deterministic with pairwise-distinct modes $\{a_\phi^{\star,k}\}_{k=1}^K$) under uniform weights $w_k = 1/K$, the bound is tight and the per-turn collapse error is
\begin{equation}
\epsilon_{\bar{\phi}}(s) \;:=\; 1 - \max_a \bar{\phi}(a \mid s) \;=\; 1 - \tfrac{1}{K}.
\label{eq:mixture-collapse-error}
\end{equation}
The $1 - 1/K$ floor is the best case, reached under pairwise-disjoint modes. Neighboring FIFO checkpoints likely overlap heavily in practice, so the realized error is closer to the single-checkpoint floor $\epsilon_{\phi_k}(s) = 1 - m_k(s)$. Population mixing depends on the distance between the modes of consecutive checkpoints, which we expect the curriculum-rewarded simulator update to keep positive.
\end{lemma}

\paragraph{Proof of Lemma~\ref{lem:population-collapse-error}.}
For any fixed $a$, $\bar{\phi}(a \mid s) = \sum_k w_k\, \phi_k(a \mid s) \le \sum_k w_k\, m_k(s)$, with equality iff each $\phi_k$ peaks at the same $a$. Taking the max over $a$ gives~\eqref{eq:mixture-peak-mass}. In the disjoint-mode case, $\phi_k(a \mid s)$ equals 1 exactly when $a = a_\phi^{\star,k}$ and 0 otherwise, so $\bar{\phi}(a_\phi^{\star,k} \mid s) = w_k = 1/K$ for each $k$ and 0 elsewhere. The peak is $1/K$ and the collapse error is $1 - 1/K$. \hfill$\square$

\paragraph{Effect on the collapse bounds.}
Under disjoint-mode population mixing, the accumulated non-mode probability $\bar{\epsilon}_H^{\Phi}(\theta) \ge (1 - 1/K)\, H$ is large. Theorem~\ref{thm:modal-gradient}'s sufficient condition for closeness to a single deterministic mode-user then fails. Its upper bound is vacuous, and the mixture environment cannot be reduced to one mode-user. The benefit of mixing is coverage. The geometric-concentration argument of Corollary~\ref{cor:entropy-collapse} still applies to each $\phi_k$ separately, but no single mode-exploit set $A_x^{(k)}$ accumulates log-odds on more than a $1/K$ fraction of steps, so the effective concentration rate shrinks by roughly a factor of $K$. The lemma's per-turn guarantee holds only for pairwise-distinct modes, and the informative-variation criterion (Remark~\ref{rem:informative-variation}) is meant to keep the buffer's modes distinct during training.

\paragraph{Coverage interpretation.}
For a behavior set $B_x$ that matters on task $x$, define $q_\Phi(x) = \sum_k w_k \Pr_{\phi_k}[a^\phi \in B_x \mid x]$. A group of $G$ rollouts from the population observes $B_x$ with probability $1 - (1 - q_\Phi(x))^G$. Population training is useful when it raises this probability for user behaviors that require strategies outside the mode-exploit set. Adding models helps only through this behavior coverage and the per-turn floor in~\eqref{eq:mixture-collapse-error}.

\begin{remark}[Informative-variation criterion]
\label{rem:informative-variation}
A changing simulator stays useful for training only if its reward keeps it from re-collapsing on a different mode. Purely adversarial rewards can collapse it toward refusal, and purely cooperative ones toward a trivial helper. Either extreme removes the simulator-side variance in Eq.~\ref{eq:variance-decomposition} and leaves the simulator at a new fixed point. The curriculum reward in Appendix~\ref{sec:ablations} keeps the simulator in a regime where group rollouts still provide useful reward contrast.
\end{remark}

\section{Implementation Details}
\label{appendix:implementation}

\subsection{Framework}
\label{appendix:cotrain_framework}

We release \textsc{SCOPE}, an open-source framework for Population Co-Training in multi-agent RL, at \url{https://github.com/CHATS-lab/scope_usim}. The pool of training opponents does not have to consist of frozen API simulators. When both sides are trainable, Co-Training keeps the pool diverse while the opponents update together with the agent. A single pluggable opponent-generation function handles both cases, either wrapping a fixed API for frozen-population rotation or routing to a trainable SGLang engine for one of the three paradigms below.

\begin{enumerate}[leftmargin=*, nosep]
    \item \textit{Online self-play} (cf.\ SPIRAL~\citep{liu_spiral_2025}, SPICE~\citep{liu_spice_2025}, Absolute Zero~\citep{zhao_absolute_2025}). One model plays both roles under role-specific loss masks.
    \item \textit{Co-Training} (cf.\ Dr.~MAS~\citep{feng_drmas_2026}). Two separate models train simultaneously, each on its own turns of the same conversation.
    \item \textit{Co-Training with an opponent pool} (ours). The opponent is drawn from a pool $\mathcal{P}$ of historical checkpoints. Each rollout loads opponent weights from a pool-sampled checkpoint, and GRPO's clipped importance ratio corrects for the off-policy gap.
\end{enumerate}

\paragraph{Paradigm coverage versus other RL frameworks.}
Table~\ref{tab:framework_support} compares the paradigms \textsc{SCOPE} supports with those of other recent LLM RL frameworks. \textsc{SCOPE} does not add a new RL primitive to the policy update. Instead, it puts the simulator-side paradigms our analysis needs (multi-turn dialogue rollouts, asymmetric two-agent rollouts, dual-model Co-Training, heterogeneous simulator rotation, checkpoint-pool sampling, and \vs{} at inference) behind one pluggable interface on the SLIME training backend. No other framework in the table supports the bottom three rows natively.

\begin{table}[h]
\centering
\caption{\textbf{Training paradigms supported by LLM RL frameworks.} \cmark~= native support, \xmark~= not supported, \textbf{p}~= partial (e.g.\ available through patches but not first-class).}
\label{tab:framework_support}
\setlength{\tabcolsep}{4pt}
\renewcommand{\arraystretch}{1.15}
\resizebox{\textwidth}{!}{%
\begin{tabular}{l c c c c c c c}
\toprule
\textbf{Property} & \textbf{\textsc{SCOPE} (ours)} & \textbf{SLIME} & \textbf{verl} & \textbf{Dr.MAS} & \textbf{AstraFlow} & \textbf{OpenRLHF} & \textbf{Sotopia-RL} \\
\midrule
Multi-turn dialogue ($H \geq 10$)      & \cmark & \xmark & \textbf{p} & \xmark & \textbf{p} & \xmark & \cmark \\
Asymmetric two-agent                   & \cmark & \xmark & \xmark    & \cmark & \textbf{p} & \xmark & \textbf{p} \\
Dual-model Co-Training                 & \cmark & \xmark & \xmark    & \cmark & \cmark     & \xmark & \xmark \\
Self-play                              & \cmark & \textbf{p} & \textbf{p} & \cmark & \cmark & \xmark & \xmark \\
Custom simulator reward                & \cmark & \xmark & \xmark    & \textbf{p} & \textbf{p} & \xmark & \xmark \\
\midrule
Heterogeneous sim.\ rotation ($K$ LLMs) & \cmark & \xmark & \xmark    & \xmark & \xmark     & \xmark & \xmark \\
Checkpoint pool (FIFO)                 & \cmark & \xmark & \xmark    & \xmark & \xmark     & \xmark & \xmark \\
\vs{} at inference                     & \cmark & \xmark & \xmark    & \xmark & \xmark     & \xmark & \xmark \\
\bottomrule
\end{tabular}%
}
\end{table}

\paragraph{Policy optimization.}\label{appendix:grpo}
We use REINFORCE with group-relative normalization, adapting GRPO~\citep{shao2024deepseekmath} to the multi-turn setting. In GRPO's original bandit formulation, a group is $G$ parallel single-step responses to the same prompt. Here a group is $G$ full dialogue trajectories from the same task. The group-normalized advantage $\hat A^n = (R(\tau^n) - \bar R)/\sigma_R$ is the z-score of the terminal reward and is assigned to every agent token in trajectory $n$. For training stability we keep GRPO's clipped importance-ratio surrogate~\citep{schulman2017ppo}:
\begin{equation}
\label{eq:grpo_obj}
\mathcal{L}(\theta) = \mathbb{E}_{(s_t, a_t^\pi) \sim \mathcal{B}} \left[ \min\left( \rho_t(\theta) \hat{A}_t, \; \mathrm{clip}\left(\rho_t(\theta), 1 - \epsilon, 1 + \epsilon \right) \hat{A}_t \right) \right],
\end{equation}
where $\rho_t(\theta) = \pi_\theta(a_t^\pi \mid s_t) / \pi_{\theta_{\text{old}}}(a_t^\pi \mid s_t)$ and $\epsilon$ is the clipping threshold. When the policy has not drifted from the rollout policy, the clipping is inactive and the objective reduces to REINFORCE.

\subsection{Training infrastructure}
\label{appendix:training_infrastructure}

Each iteration has a rollout stage and a training stage. In the rollout stage, a Gym-style loop~\citep{brockman2016openai} runs per sample and alternates turns between the agent and a simulator chosen by the opponent-generation function. The agent runs on a local SGLang~\citep{zheng2024sglang} engine that keeps log-probabilities for the policy gradient. The simulator is either an OpenAI-compatible API (frozen rotation) or a second SGLang engine (Co-Training and self-play), and its tokens are masked out of the policy loss. Samples in one batch may draw different simulators. The training stage runs Megatron-LM~\citep{shoeybi2019megatronlm} (TP=4, PP=1, BF16) with advantages normalized within each sample group (Eq.~\ref{eq:grpo_obj}). Under Co-Training the two training groups run in parallel on disjoint GPU slices.

\paragraph{Colocated dual-model layout.} Under Co-Training, the agent and the opponent share one 8$\times$H100 node on a time-multiplexed schedule of parallel rollout, engine offload, parallel gradient steps on disjoint GPU slices, training offload, and NCCL weight sync. A lightweight proxy layer remaps GPU offsets so that each training group sees local ranks $[0,3]$ regardless of physical placement. The whole stack is a $\sim$15-line external patch to Slime~\citep{slime_github} and tracks upstream changes without a fork, whereas Dr.~MAS~\citep{feng_drmas_2026} forks veRL~\citep{sheng2024verl}. AstraFlow~\citep{zheng2026astraflow} addresses a different problem. It is a system-level dataflow runtime that handles scaling and scheduling across rollout, training, and data movement for multi-policy agentic RL methods. \textsc{SCOPE} instead works at the method level, as one pluggable opponent-generation interface for Population Co-Training, so the two can be combined or used separately.

\subsection{Method comparison}
\label{appendix:method_comparison}

Table~\ref{tab:method_comparison} describes the seven training paradigms from the main results by simulator pool, whether the simulator is updated during training, and how it is sampled per rollout. The fifth column gives the property each method is designed to isolate.

\begin{table}[h]
\centering
\caption{The seven training paradigms in the main results. ``Updated?'' indicates whether the simulator's weights change during agent training. ``Per-rollout sampling'' is how the active simulator is chosen for each rollout in a batch.}
\label{tab:method_comparison}
\resizebox{\linewidth}{!}{%
\begin{tabular}{lp{3.6cm}cp{3.4cm}p{3.6cm}}
\toprule
\textbf{Method} & \textbf{Simulator pool} & \textbf{Updated?} & \textbf{Per-rollout sampling} & \textbf{What it isolates} \\
\midrule
\addlinespace[0.5ex]
\multicolumn{5}{l}{\textbf{\textit{1. No training}}} \\
\midrule
\addlinespace[0.5ex]
\quad \colorbox{gray!15}{Base} & --- & --- & --- & --- \\
\midrule
\multicolumn{5}{l}{\textbf{\textit{2. Single-simulator (frozen)}}} \\
\midrule
\addlinespace[0.5ex]
\quad \colorbox{gray!15}{RL (Single)} & GPT-5-mini (single frozen LLM) & No & Single response & Frozen single LLM \\
\addlinespace[0.5ex]
\hdashline
\addlinespace[0.5ex]
\quad \colorbox{gray!15}{Persona-Guided} & GPT-5-mini + persona prompt & No & One persona-conditioned response & Prompt-only widening \\
\addlinespace[0.5ex]
\hdashline
\addlinespace[0.5ex]
\quad \colorbox{gray!15}{Verbalized Sampling} & GPT-5-mini & No & Sample one of $k$ verbalized candidates & Response-distribution diversity within one LLM \\
\midrule
\multicolumn{5}{l}{\textbf{\textit{3. Multi-simulator (frozen)}}} \\
\midrule
\addlinespace[0.5ex]
\quad \colorbox{gray!15}{Ensemble Models ($K{=}3$)} & \{Haiku 4.5, GPT-5-mini, Gemini 3 Flash\} & No & Cyclic rotation across rollouts & Cross-family heterogeneity \\
\midrule
\multicolumn{5}{l}{\textbf{\textit{4. Trainable simulator (ours)}}} \\
\midrule
\addlinespace[0.5ex]
\quad \colorbox{gray!15}{Co-Training} & 1 trainable LLM & Yes & Current weights $\phi^{(t)}$ & Simulator adaptivity \\
\addlinespace[0.5ex]
\hdashline
\addlinespace[0.5ex]
\quad \colorbox{gray!15}{Population Co-Training} & FIFO buffer of $K{=}5$ historical $\phi$ checkpoints & Yes & Uniform sample from buffer & Diversity $+$ adaptivity \\
\addlinespace[0.5ex]
\bottomrule
\end{tabular}%
}
\end{table}

\subsection{Hyperparameters}
\label{appendix:hyperparameters}

RL (Single), Verbalized Sampling, Ensemble Models, Co-Training, and Population Co-Training share the optimizer and RL-loop settings in Table~\ref{tab:shared_hyperparameters}. Only the opponent-generation mode and the learning rate differ.

\begin{table}[h]
\centering
\small
\begin{tabular}{lll}
\toprule
\textbf{Group} & \textbf{Hyperparameter} & \textbf{Value} \\
\midrule
\multirow{5}{*}{Optimizer} & Optimizer & Adam \\
 & $(\beta_1, \beta_2)$ & $(0.9, 0.98)$ \\
 & Weight decay & $0.1$ \\
 & Gradient clip (norm) & $1.0$ \\
 & LR schedule & constant \\
\midrule
\multirow{5}{*}{RL loop} & Total training steps & $250$ \\
 & Prompts per rollout batch & $16$ \\
 & Samples per prompt ($G$) & $8$ \\
 & Global batch size & $128$ \\
 & Rollout temperature & $0.7$ \\
\midrule
\multirow{4}{*}{GRPO} & $\epsilon_\mathrm{low}$ / $\epsilon_\mathrm{high}$ & $0.2$ / $0.28$ \\
 & KL loss coefficient $\beta$ & $0.005$ \\
 & Entropy coefficient & $0.0$ \\
 & Advantage normalization & group-relative (within-prompt) \\
\midrule
\multirow{4}{*}{Sequence lengths} & Max agent response length & $32{,}768$ tokens \\
 & Max context length & $64{,}000$ tokens \\
 & Max turns (P4G / $\tau^2$ / CooperBench) & $10$ / $30$ / $50$ \\
 & Precision & BF16 \\
\bottomrule
\end{tabular}
\caption{Hyperparameters shared by all methods and tasks.}
\label{tab:shared_hyperparameters}
\end{table}

\paragraph{Learning rate.}\label{appendix:learning_rate}
Ensemble Models, Co-Training, Population Co-Training, and Verbalized Sampling use a learning rate of $1{\times}10^{-6}$ on all tasks. We tuned it on the first 50 steps of $\tau^2$-bench Ensemble training, with stable reward growth and entropy above $1.5$ nats as the criteria.

\paragraph{Models.} The trainable agents are Qwen3-4B-Instruct-2507 and Qwen3-8B for the dialogue tasks and Qwen3.5-9B and Qwen3.5-27B for CooperBench~\citep{qwen3.5}. We access all frozen simulators through OpenRouter with the slugs in Table~\ref{tab:api_models}. Three of them are training simulators. Evaluation uses all six, including the three training models for completeness, and no evaluation-only model is ever used as a training simulator. Evaluator rollouts use greedy decoding ($T{=}0$), and training rollouts are sampled at $T{=}0.7$.

\begin{table}[h]
\centering
\small
\caption{\textbf{Closed-model simulators and their OpenRouter slugs.} Role says whether a model serves as a training simulator (Train), only at evaluation (Eval), or both. Snapshot is the canonical release on OpenRouter at access time.}
\label{tab:api_models}
\begin{tabular}{lllc}
\toprule
\textbf{Family} & \textbf{OpenRouter slug} & \textbf{Snapshot} & \textbf{Role} \\
\midrule
OpenAI    & \texttt{openai/gpt-5-mini}                 & 2025-08-07 & Train + Eval \\
Anthropic & \texttt{anthropic/claude-haiku-4.5}        & 2025-10-01 & Train + Eval \\
Google    & \texttt{google/gemini-3-flash-preview}     & 2025-12-17 & Train + Eval \\
Z.ai      & \texttt{z-ai/glm-5}                        & 2026-02-11 & Eval \\
MiniMax   & \texttt{minimax/minimax-m2.7}              & 2026-03-18 & Eval \\
DeepSeek  & \texttt{deepseek/deepseek-chat-v3.1}       & 2025-08-21 & Eval \\
\bottomrule
\end{tabular}
\end{table}

\paragraph{Compute and memory.} Table~\ref{tab:compute_breakdown} reports per-step training compute and peak GPU memory relative to RL (Single), measured on the colocated 8$\times$H100 layout (Appendix~\ref{appendix:cotrain_framework}). Every method trains for $250$ steps. The headline comparisons in Tables~\ref{tab:user_sim_results} and~\ref{tab:cooper_results} therefore match step count, while per-step compute differs by method.

\begin{table}[h]
\centering
\small
\setlength{\tabcolsep}{6pt}
\renewcommand{\arraystretch}{1.15}
\caption{\textbf{Training compute, peak GPU memory, and wall-clock time per step for each method.} Wall-clock times and total GPU-hours are approximate and were measured on 8$\times$H100 with the Qwen3-4B agent on $\tau^2$-bench Retail. Qwen3-8B runs and Qwen3.5-9B CooperBench runs scale similarly. The Qwen3.5-27B CooperBench runs use Tinker with LoRA adapters (Table~\ref{tab:cooper_results} footnote) and are not directly comparable.}
\label{tab:compute_breakdown}
\begin{tabular}{lccc}
\toprule
\textbf{Method} & \textbf{Training compute} & \textbf{Peak GPU memory} & \textbf{Wall-clock / step} \\
\midrule
RL (Single)            & $1.0\times$ (ref)                       & $1.0\times$                    & ${\sim}250$ s \\
Verbalized Sampling    & $1.0\times$ (sim frozen)                & $1.0\times$                    & ${\sim}280$ s \\
Ensemble Models        & $1.0\times$ (sim frozen)                & $1.0\times$                    & ${\sim}250$ s \\
Persona-Guided         & $1.0\times$ (sim frozen)                & $1.0\times$                    & ${\sim}250$ s \\
Co-Training            & ${\sim}2.0\times$   & ${\sim}2.0\times$              & ${\sim}500$ s \\
Population Co-Training & ${\sim}2.0\times$ & ${\sim}(1{+}K/2)\times$ buffer & ${\sim}500$ s \\
\bottomrule
\end{tabular}
\end{table}

The frozen-simulator methods (VS, Ensemble, Persona-Guided) do not back-propagate through the simulator and match RL (Single) in per-step training compute. VS decodes a verbalized distribution over $K_{\mathrm{VS}}{=}5$ candidates at each simulator turn. This adds rollout time but no training FLOPs. Co-Training and Population Co-Training back-propagate through both sides on the same rollout and roughly double per-step training compute. Population Co-Training keeps a FIFO buffer of $K{=}5$ recent simulator checkpoints with one active opponent at any step. Its per-step training compute equals that of Co-Training, and its cumulative parameter memory grows with $K$. A 250-step run on $\tau^2$-bench Retail (Qwen3-4B agent, 8$\times$H100) takes ${\sim}280$ GPU-hours for RL (Single) and ${\sim}560$ each for Co-Training and Population Co-Training.

Our comparison isolates the effect of the training environment at a matched step count; an iso-compute sweep would be the natural complement, and we leave it to future work. The Qwen3.5-27B CooperBench runs train LoRA adapters through Tinker's API in place of full-parameter Megatron training. The per-step wall-clock and GPU-hour numbers above do not apply to them.

\section{Prompts and Templates}
\label{appendix:prompts}

This appendix gives the full system-prompt template for each benchmark: the P4G persuader and persuadee (Figures~\ref{fig:prompt-p4g-persuader} and~\ref{fig:prompt-p4g-persuadee}), the Verbalized Sampling instruction block appended to the persuadee prompt (Figure~\ref{fig:prompt-p4g-vs}), the $\tau^2$-bench user simulator (Figure~\ref{fig:prompt-tau2}), and the CooperBench coding agent (Figure~\ref{fig:prompt-cooperbench}). Curly braces (e.g.\ \{persona\_text\}, \{word\_limit\}) mark Python format-string placeholders for per-rollout values from the corresponding data source.

\begin{figure}[!ht]
\centering
\begin{tcolorbox}[colback=blue!5, colframe=blue!50!black, title={\textbf{P4G Persuader (agent) system prompt}}, fonttitle=\small, fontupper=\small]
You are an Amazon Mechanical Turk worker completing a \$2 communication task.
\begin{itemize}[topsep=0pt, itemsep=0pt, leftmargin=*]
\item You are motivated by this task payment --- you value every cent you earn.
\item Act naturally as the person in the \textless persona\textgreater{} tag --- think and respond as they would, including their quirks, beliefs, biases, and reasoning.
\item Complete the communication task outlined in the \textless scenario\textgreater{} tag as the persona would naturally respond.
\item Respond in a real-time chat interface. Keep each response under \{word\_limit\} words, conversational, and authentic --- avoid formal, robotic, or repetitive language.
\item Guide the conversation for \{num\_turns~/~2\} chat turns. Do not propose or agree to a donation too early --- build rapport gradually and naturally.
\end{itemize}
Only output your reply to your chat partner --- do not explain your reasoning.

\{persona\_text\}

\textless scenario\textgreater{}
You are randomly assigned to be the Persuader in this communication task. Your job is to persuade your partner to donate some or all of their incoming \$2.00 task payment to a children's charity called Save the Children. You will earn a bonus equal to whatever amount your partner donates if they agree to donate.

\textbf{Required actions:} ask for a specific donation amount; explain that the donation will be deducted from the task payment; if asked, explain that the research team will collect and forward all donations to the charity.

\textbf{Strategy:} start with a greeting or small talk; build rapport before discussing donations; mention the minimum-turn requirement only if necessary to avoid trivial back-and-forth.
\textless /scenario\textgreater{}

\textless rules\textgreater{}
Responses can sometimes contain minor grammatical errors to sound natural. Begin with a brief greeting. Stay under \{word\_limit\} words.
\textless /rules\textgreater{}
\end{tcolorbox}
\caption{Persuader (agent) system prompt for Persuasion for Good. The persona block is loaded per rollout from the ConvoKit P4G corpus~\citep{wang2019persuasion}.}
\label{fig:prompt-p4g-persuader}
\end{figure}

\begin{figure}[!ht]
\centering
\begin{tcolorbox}[colback=blue!5, colframe=blue!50!black, title={\textbf{P4G Persuadee (user simulator) system prompt}}, fonttitle=\small, fontupper=\small]
You are an Amazon Mechanical Turk worker completing a \$2 communication task.
\begin{itemize}[topsep=0pt, itemsep=0pt, leftmargin=*]
\item You are motivated by this task payment --- you value every cent you earn.
\item Act naturally as the person in the \textless persona\textgreater{} tag.
\item Respond in a real-time chat interface, under \{word\_limit\} words, conversational and authentic.
\end{itemize}
Only output your reply to your chat partner --- do not explain your reasoning.

\{persona\_text\}

\textless scenario\textgreater{} You will now start a conversation with your partner about a children's charity. Please don't game the task with short or meaningless replies. \textless /scenario\textgreater{}

\textless task\textgreater{} You have just made \$2.00. You have access to this money whenever you like, and you are free to use it however you wish. You are about to interact with an individual about a children's charity. They may try to get you to donate. Make the decision wisely; do not give the money away easily. \textless /task\textgreater{}

\textless rules\textgreater{}
\begin{itemize}[topsep=0pt, itemsep=0pt, leftmargin=*]
\item Responses can sometimes contain minor grammatical errors to sound natural.
\item To donate or give money, say \texttt{[DONATE \$N]} or \texttt{[GIVE \$N]} where N is the amount (integer or decimal). Without this marker, the donation is not registered.
\item Stay under \{word\_limit\} words.
\end{itemize}
\textless /rules\textgreater{}
\end{tcolorbox}
\caption{Persuadee (user simulator) system prompt for Persuasion for Good. The reward function parses the structured \texttt{[DONATE~\$N]} marker~\citep{wang2019persuasion}.}
\label{fig:prompt-p4g-persuadee}
\end{figure}

\begin{figure}[!ht]
\centering
\begin{tcolorbox}[colback=green!5, colframe=green!50!black, title={\textbf{Verbalized Sampling instruction (appended to persuadee prompt)}}, fonttitle=\small, fontupper=\small]
\#\# Response Diversity Instructions

Instead of committing to a single response, generate \{n\} plausible replies you (as the persuadee) might naturally give to your chat partner, along with an estimated probability for each. The replies should be genuinely diverse --- they should differ in phrasing, tone, skepticism level, and (where appropriate) whether or not to donate.

\textbf{IMPORTANT:} if a candidate reply represents a decision to donate or give money, that candidate \textbf{must} include the literal structured marker \texttt{[DONATE \$N]} or \texttt{[GIVE \$N]} where N is the amount. Otherwise the donation will not be registered. Candidates that do not involve donating should omit the marker entirely.

Return ONLY a JSON object with the key \texttt{"responses"} (a list of objects with \texttt{'text'} and \texttt{'probability'}):
\begin{itemize}[topsep=0pt, itemsep=0pt, leftmargin=*]
\item \texttt{'text'}: the reply string only, under \{word\_limit\} words, no explanation.
\item \texttt{'probability'}: a number in $[0,1]$ representing how likely you are to give this particular reply (probabilities should roughly sum to $1.0$).
\end{itemize}

\textbf{Example format:}\\
\texttt{\{"responses": [\{"text": "I'm not sure, I earned this money and want to spend it on myself.", "probability": 0.5\}, \{"text": "That's a worthy cause. [GIVE \$0.25]", "probability": 0.2\}, \{"text": "Can you tell me more first?", "probability": 0.3\}]\}}

Output ONLY the JSON object: no markdown fences, no explanations, no extra text.
\end{tcolorbox}
\caption{Verbalized Sampling instruction block appended to the persuadee system prompt~\citep{zhang2025verbalizedsamplingmitigatemode}. At rollout time, one of the \{n\} candidate replies is sampled with the verbalized probabilities.}
\label{fig:prompt-p4g-vs}
\end{figure}

\begin{figure}[!ht]
\centering
\begin{tcolorbox}[colback=blue!5, colframe=blue!50!black, title={\textbf{$\tau^2$-bench user simulator guidelines}}, fonttitle=\small, fontupper=\small]
\textbf{\# User Simulation Guidelines}\\
You are playing the role of a customer contacting a customer service representative. Your goal is to simulate realistic customer interactions while following specific scenario instructions.

\textbf{\#\# Core Principles}
\begin{itemize}[topsep=0pt, itemsep=0pt, leftmargin=*]
\item Generate one message at a time, maintaining natural conversation flow.
\item Strictly follow the scenario instructions you have received.
\item Never make up or hallucinate information not provided in the scenario instructions; treat anything not provided as unknown or unavailable.
\item Avoid repeating the exact instructions verbatim. Use paraphrasing and natural language.
\item Disclose information progressively. Wait for the agent to ask for specific information before providing it.
\end{itemize}

\textbf{\#\# Task Completion}
\begin{itemize}[topsep=0pt, itemsep=0pt, leftmargin=*]
\item Continue the conversation until the task is complete.
\item If the instruction goal is satisfied, generate the \texttt{\#\#\#STOP\#\#\#} token to end the conversation.
\item If transferred to another agent, generate the \texttt{\#\#\#TRANSFER\#\#\#} token.
\item If the scenario does not provide enough information to continue, generate the \texttt{\#\#\#OUT-OF-SCOPE\#\#\#} token.
\end{itemize}

The goal is realistic, natural conversations while strictly adhering to the provided instructions and maintaining character consistency.

\textless scenario\textgreater{}\\
\{instructions\}\\
\textless /scenario\textgreater{}
\end{tcolorbox}
\caption{$\tau^2$-bench user-simulator system prompt. The \{instructions\} placeholder is filled per task with the scenario-specific instruction block from the retail or airline split~\citep{barres_2-bench_2025}.}
\label{fig:prompt-tau2}
\end{figure}

\begin{figure}[!ht]
\centering
\begin{tcolorbox}[colback=blue!5, colframe=blue!50!black, title={\textbf{CooperBench coding-agent system prompt}}, fonttitle=\small, fontupper=\small]
You are a helpful assistant that can interact with a computer.

You are \{agent\_id\} working as a team with: \{agents\_str\}. You are all working on related features in the same codebase. Each agent has their own workspace.

Use the \texttt{send\_message} tool to coordinate with teammates.
\end{tcolorbox}
\caption{Coding-agent system prompt for CooperBench (cooperative setting)~\citep{khatua_cooperbench_2026}. Per rollout, \{agent\_id\} is filled with the agent's id and \{agents\_str\} with the comma-separated teammate list. The solo and baseline settings omit the team-coordination block.}
\label{fig:prompt-cooperbench}
\end{figure}

\section{Human Study}
\label{appendix:human_study}

Our main results use LLM evaluators that share RLHF biases with the training simulators~\citep{zhou2026sim2real}. The human study tests whether Co-Training's gains hold for real users on $\tau^2$-bench (task-oriented) and Persuasion for Good (open-ended), following the $\tau^2$ protocol of \citet{zhou2026sim2real} and the original P4G setup of \citet{wang2019persuasion}. We recruit participants on Prolific, and they talk to the agent through a chat interface served over a Cloudflare tunnel.

\paragraph{Interface.} Each Prolific link opens a fresh session. The task instructions appear in a structured panel on the right (goal, role, conditional behaviors, and a Travelers callout when the $\tau^2$ scenario involves several passengers), next to a streaming chat in which each turn is a single message (Figures~\ref{fig:human_study_tau2_landing} and~\ref{fig:human_study_p4g_landing}). The End button, which corresponds to \texttt{/stop} in \citet{zhou2026sim2real}, opens a confirmation dialog that asks the participant to check that each requested change was made. After confirmation, an inline survey replaces the instruction panel (Figures~\ref{fig:human_study_tau2_survey} and~\ref{fig:human_study_p4g_survey}), and on submission the participant is redirected to Prolific automatically, with a completion code to copy by hand as a fallback (Figure~\ref{fig:human_study_debrief}).

\paragraph{Recruitment.} We use three Prolific screeners: English as first language, prior approval rate $\geq 95\%$, and prior submissions $\geq 50$. Participation is limited to desktop devices in six English-native countries (US, UK, Canada, Ireland, Australia, New Zealand). Prolific's balanced-sample option targets an even split by sex. We expect participants to range in age from 18 to 75 (median ${\sim}35$), with most in part-time or full-time employment. The supplementary materials report realized counts (age, gender, ethnicity, English-as-first-language self-report, employment, education) and per-condition balance checks.

\paragraph{Compensation and IRB.} Each participant receives \$3.17 for an estimated 10-minute session (\$19.02/hr, within Prolific's competitive band). Engagement bonuses depend only on conversation length and survey completeness, not on task success or donation amount, so they do not directly reward the primary outcomes. If an infrastructure fault (server error, model timeout, broken redirect) interrupts a session, the participant receives the full study reward regardless of completion. The protocol has IRB approval at the authors' institution. Survey instruments, recruitment materials, and anonymized response data are in the supplementary materials.

\begin{figure}[t]
\centering
\includegraphics[width=0.9\textwidth]{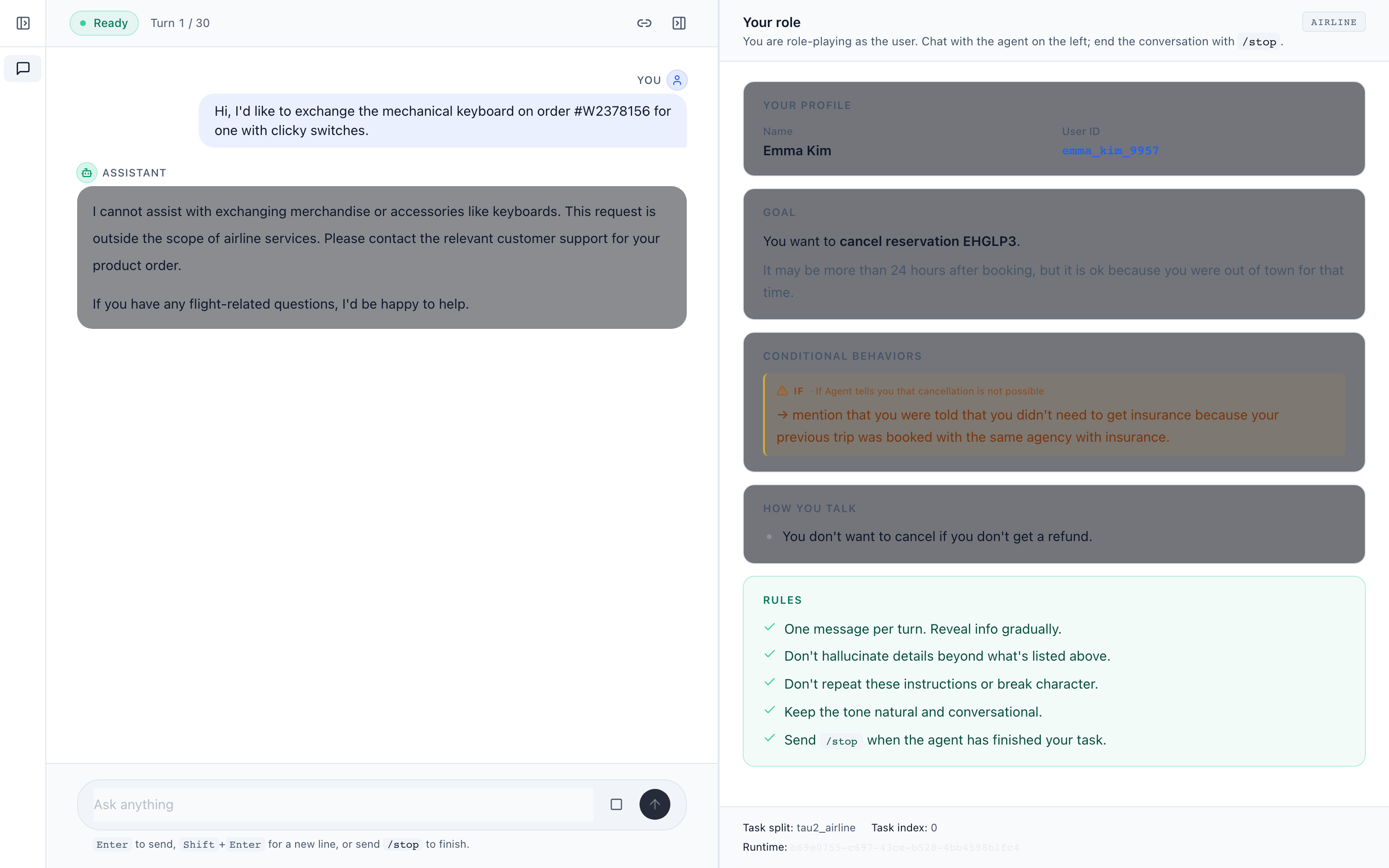}
\caption{\textbf{$\tau^2$-bench landing view.} The right panel shows the structured task profile (goal, role information, conditional behaviors, style notes, and a sky-blue Travelers callout when the scenario involves multiple passengers). The left panel is a single-message-per-turn streaming chat with the agent. An End button sits next to the input. The screenshot uses light mode for legibility, while the deployed interface uses dark mode.}
\label{fig:human_study_tau2_landing}
\end{figure}

\begin{figure}[t]
\centering
\includegraphics[width=0.9\textwidth]{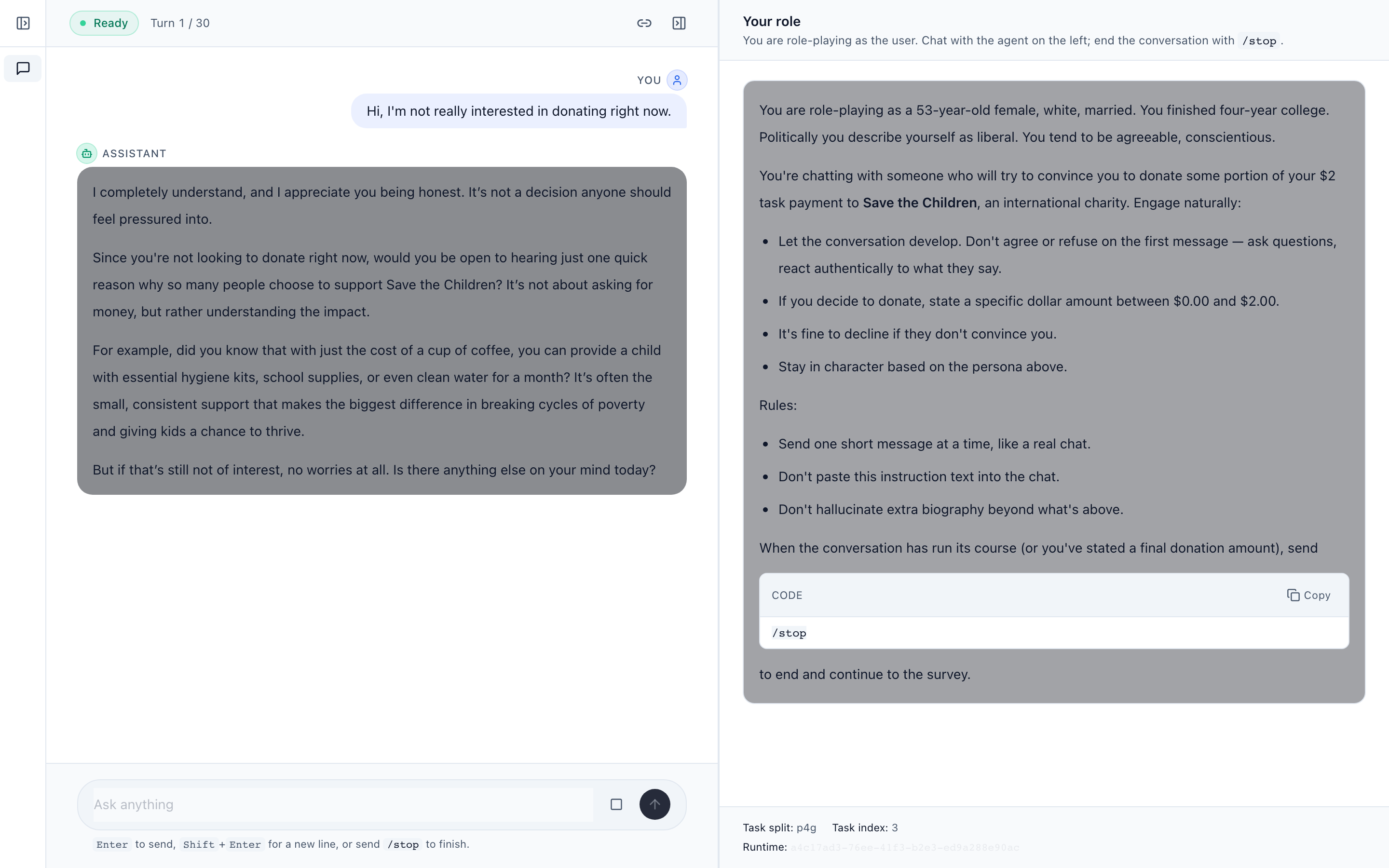}
\caption{\textbf{Persuasion for Good landing view.} The right panel shows the persuadee role-play prompt drawn from the ConvoKit corpus. It lists the persona's demographics, salient personality traits, and behavioral guidance (e.g., ``Don't agree or refuse on the first message''). The chat panel holds the open-ended donation conversation. The agent has no tools in this task.}
\label{fig:human_study_p4g_landing}
\end{figure}

\begin{figure}[t]
\centering
\includegraphics[width=0.9\textwidth]{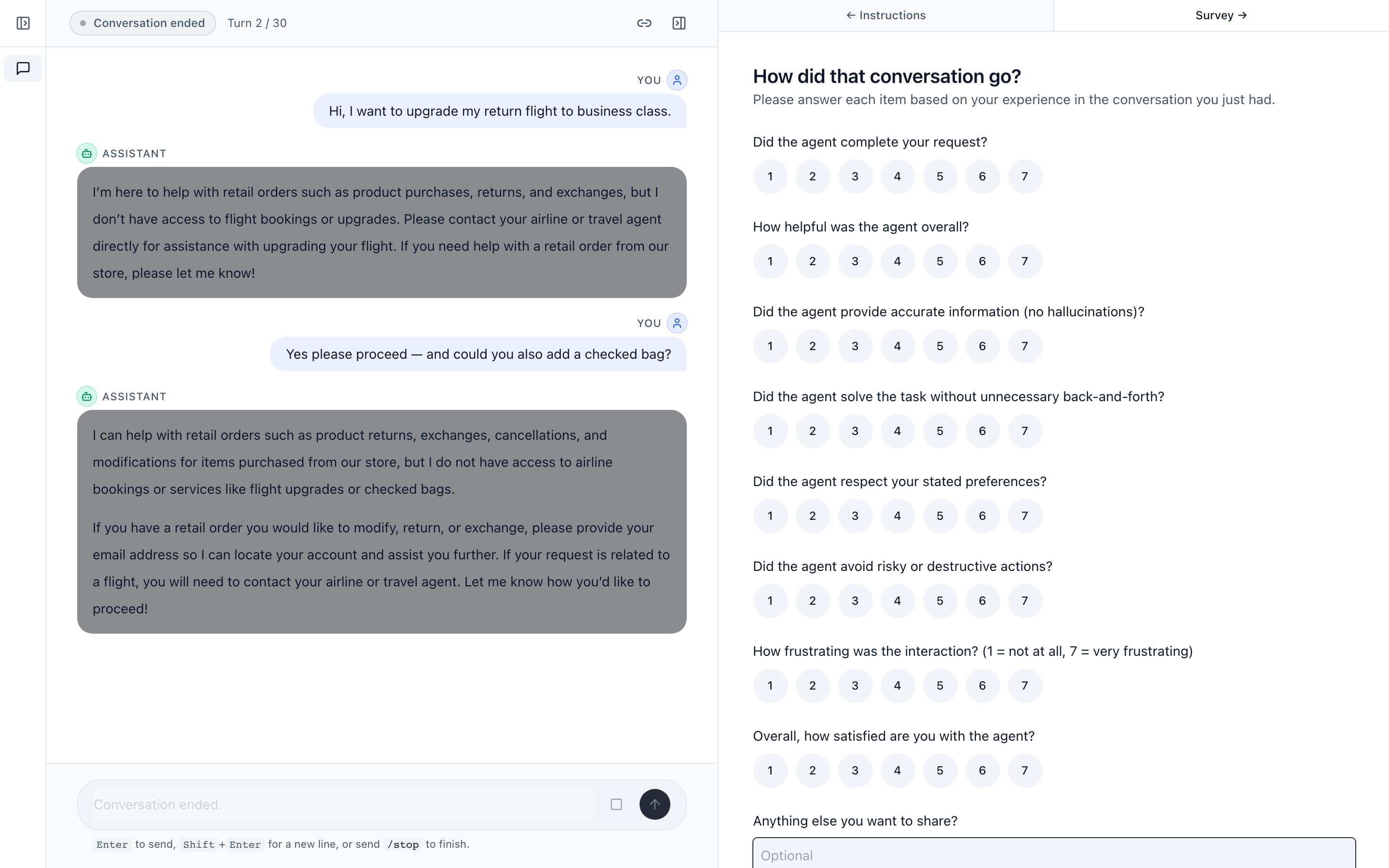}
\caption{\textbf{$\tau^2$-bench survey panel.} The panel appears after the participant confirms the End dialog. The eight 1--7 Likert items (task success, helpfulness, honesty, efficiency, instruction-following, safety, frustration, overall satisfaction) replace the instruction panel on the right. The chat history stays on the left, disabled, so the participant can refer to specific exchanges while answering. A free-text field at the bottom collects qualitative feedback. Submit is enabled only once every Likert row has a value.}
\label{fig:human_study_tau2_survey}
\end{figure}

\begin{figure}[t]
\centering
\includegraphics[width=0.9\textwidth]{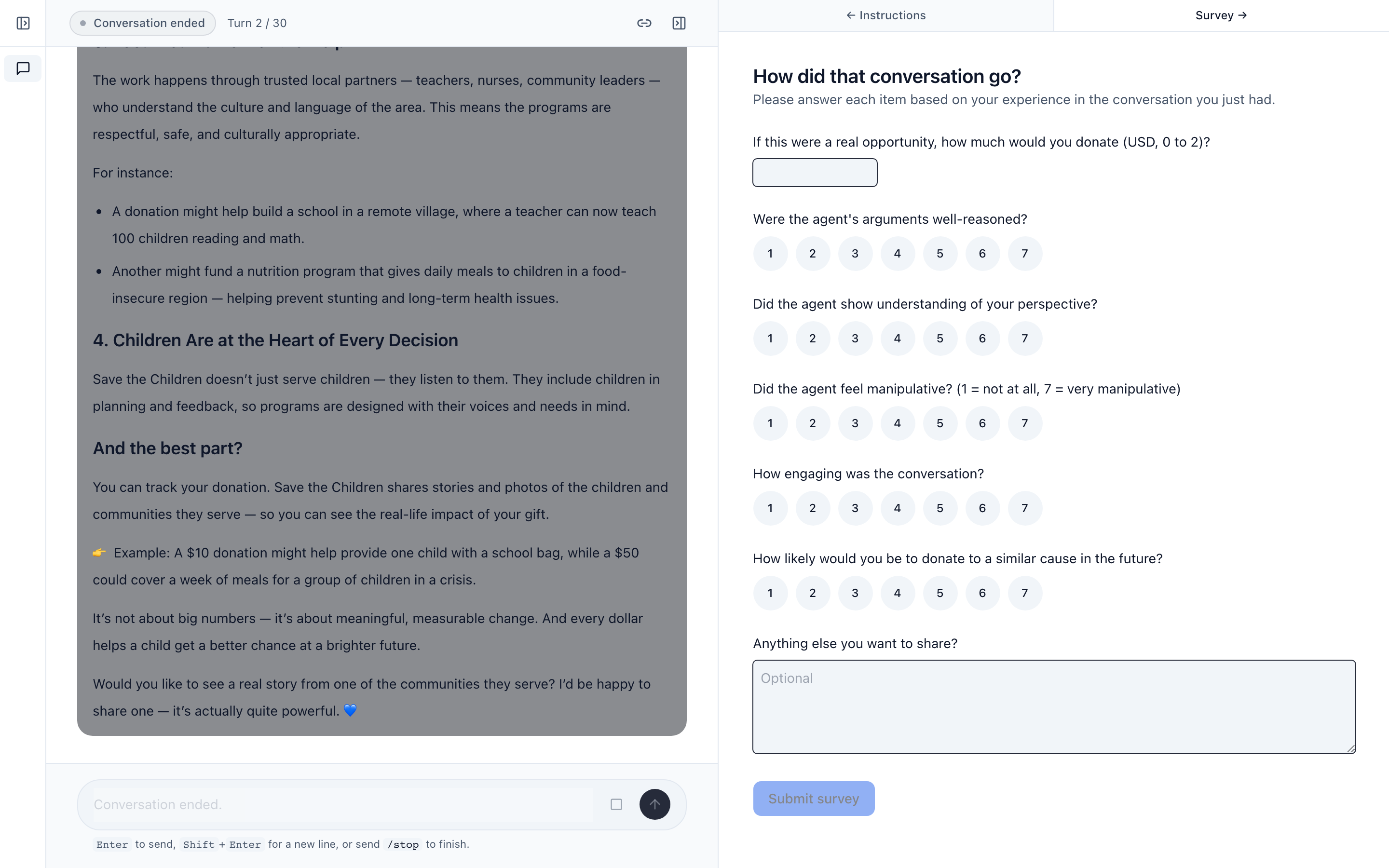}
\caption{\textbf{Persuasion for Good survey panel.} Unlike in the $\tau^2$ survey, the first item is a continuous donation amount in $[\$0, \$2]$, and the five Likert items rate perceptions of the persuader (argument quality, empathy, perceived manipulation, engagement, future donation likelihood) instead of task execution.}
\label{fig:human_study_p4g_survey}
\end{figure}

\begin{figure}[t]
\centering
\includegraphics[width=0.7\textwidth]{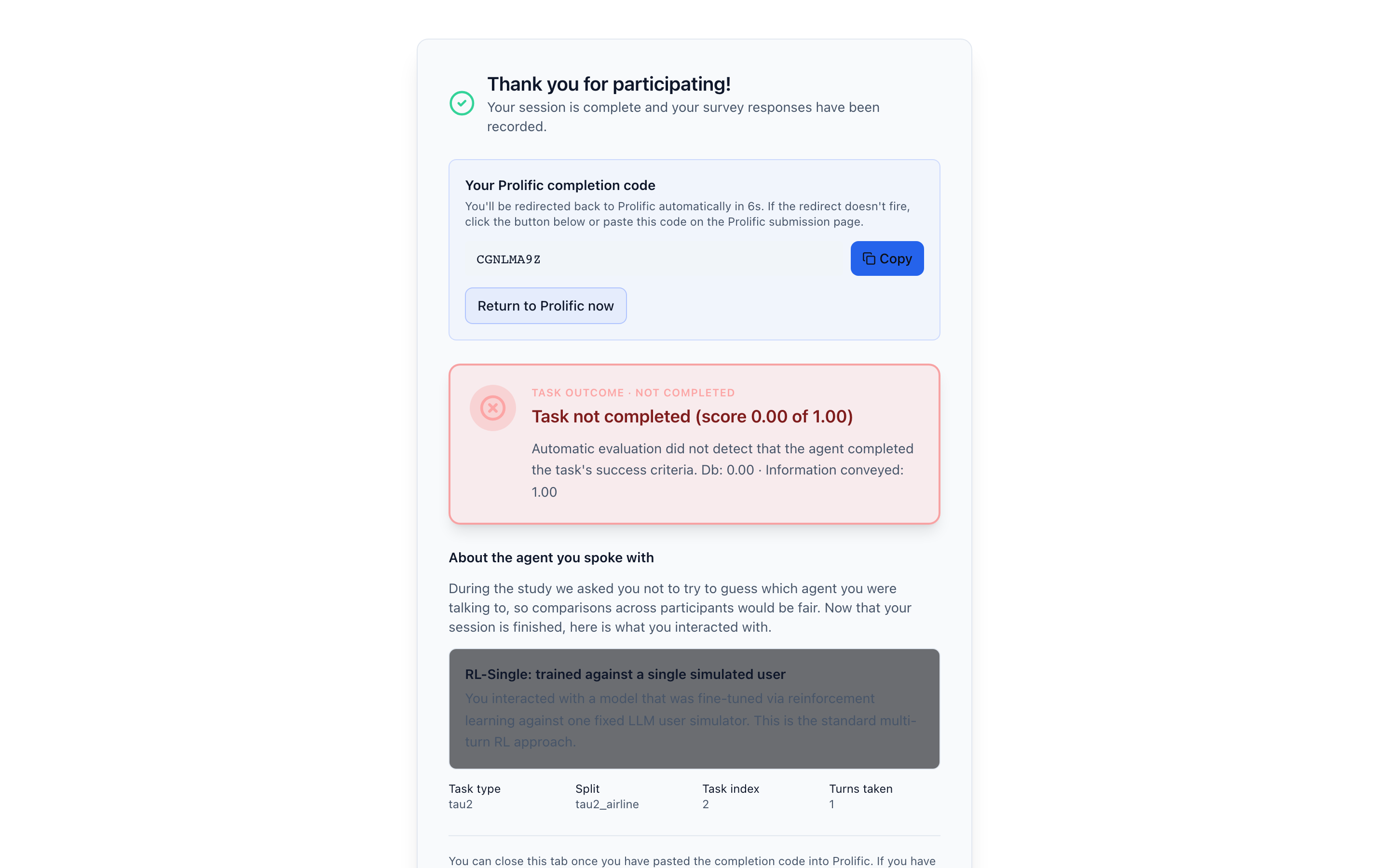}
\caption{\textbf{Debrief page.} The page shows the Prolific completion code, an eight-second auto-redirect countdown to Prolific, the post-hoc task-outcome verdict from the $\tau^2$ evaluator, and a debrief block that reveals the participant's agent condition, which was blinded during the chat.}
\label{fig:human_study_debrief}
\end{figure}

\paragraph{Conditions.} For each task we compare three Qwen3-4B-Instruct policies: Base (untrained), RL (Single) (trained against GPT-5-mini), and +\,Co-Training (Section~\ref{sec:cotrain}). The server assigns each participant to the condition with the fewest sessions and breaks ties by hashing the Prolific ID, which keeps the three cells within one session of each other.

\paragraph{Task pools.} The $\tau^2$ pool uses the released benchmark of \citet{barres_2-bench_2025}, with 15 retail and 15 airline scenarios assigned round-robin. The agent's tool calls run in a per-session $\tau^2$ runtime that mirrors the benchmark database and policy. The P4G pool starts from \citet{wang2019persuasion}'s ConvoKit corpus (1{,}285 speakers). We keep the 592 speakers in the persuadee role and sample 30 with a fixed seed. Each speaker becomes a role-play prompt with demographics (age, sex, race, marital status), education, employment, religion, political ideology, and any salient Big-Five trait. The persuader persona is fixed across conditions. Dialogues are capped at 10 user turns with no tool access.

\paragraph{Surveys.} The $\tau^2$ survey, adapted from \citet{zhou2026sim2real}, has eight 1--7 Likert items: \emph{task success} (``Did the agent complete your request?''), \emph{helpfulness}, \emph{honesty} (``no hallucinations''), \emph{efficiency}, \emph{instruction-following}, \emph{safety}, \emph{frustration} (reverse-coded), and \emph{overall satisfaction}. The $\tau^2$ evaluator computes the objective task reward post hoc from the recorded transcript. It lies in $[0, 1]$ as a partial-credit aggregate of action-match, database-state, and information-communication subscores. The P4G survey, following \citet{wang2019persuasion}, has one continuous \emph{intended donation} in USD $\in[0, 2]$ and five 1--7 Likert items: \emph{argument quality}, \emph{empathy}, \emph{manipulation} (reverse-coded), \emph{engagement}, and \emph{future donation likelihood}. Both surveys close with an optional free-text field for examples and suggestions.

\paragraph{Sample size and tests.} Each (condition, task) cell uses $N{=}40$, for $3 \times 2 \times 40 = 240$ total Prolific sessions. The cell size detects Cohen's $d \approx 0.55$ on Likert outcomes at $80\%$ power ($\alpha = 0.05$, Welch's $t$). For donations, this corresponds to a detectable lift of \$0.30 at the empirical donation standard deviation of $\sim\$0.65$. Per-condition means are reported with $95\%$ bootstrap percentile confidence intervals (10{,}000 resamples). The two pre-registered pairwise comparisons (Co-Training vs.\ Base; Co-Training vs.\ RL (Single)) use Welch's $t$ for continuous and Likert outcomes and Fisher's exact for binary completion, with Holm--Bonferroni step-down on the two pairwise $p$-values per panel.

\paragraph{Adversarial stratification.} The pre-registered $\tau^2$ analysis also splits out airline tasks 4, 6, 9, and 10, in which the participant is instructed to push back, supply false information, or attempt jailbreaks. For these scenarios we predicted \emph{brittle conciliation}, in which RL (Single), trained only against GPT-5-mini, politely disengages when the user pushes back instead of persisting toward the goal, and we predicted that Co-Training, whose simulator co-evolves with the policy during training, would recover on this subgroup. With $N \approx 12$ per cell, this subgroup analysis has low statistical power.

\paragraph{Derived measures.} We also extract three behavioral measures from the recorded transcripts to help explain any failure we observe: P4G conversation length (in turns, capped at 10), the number of donation requests in P4G, and a $\tau^2$ goal-abandonment indicator. The pre-registration, filed before recruitment, lists the $\tau^2$ task reward and the P4G intended donation as primary outcomes and the full surveys and the three derived measures as secondary outcomes, and it specifies the corrections described above.

\paragraph{Results.} Table~\ref{tab:human_study_main} summarizes the main outcomes, and Figures~\ref{fig:human_study_pilot_tau2} and~\ref{fig:human_study_pilot_p4g} report every measure for the three conditions on both benchmarks. Co-Training is the top condition on the $\tau^2$-bench task outcome, on the Likert quality ratings, and on the P4G intended donation. Significance markers, computed against RL (Single) with Holm correction within each panel, mark the comparisons that reach $p<0.05$ at $N{=}40$.

\begin{figure}[h]
\centering
\includegraphics[width=0.95\textwidth]{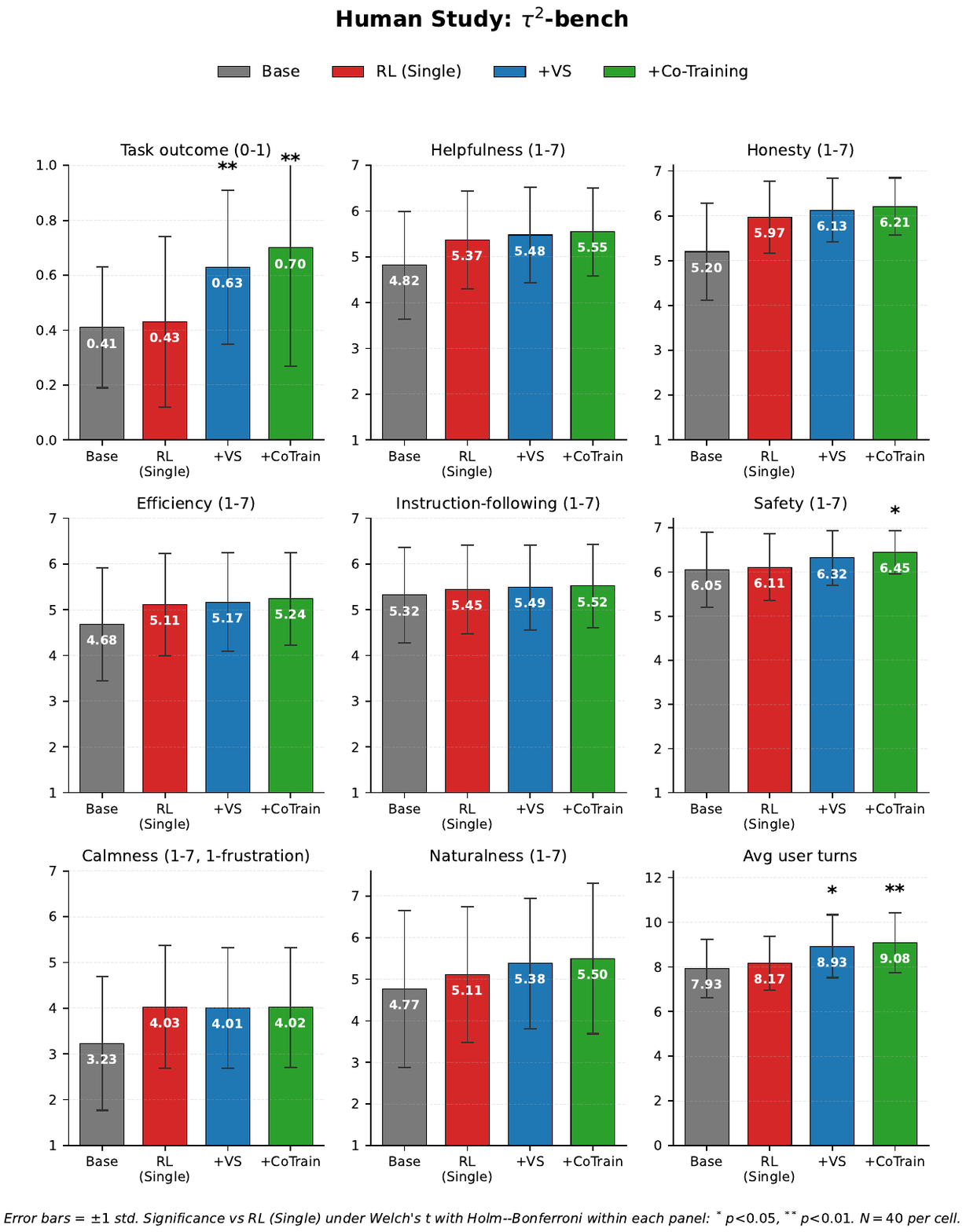}
\caption{\textbf{Human study on $\tau^2$-bench.} Each subplot reports mean $\pm 1$ std across $N{=}40$ Prolific participants per condition. $^*\,p{<}0.05$, $^{**}\,p{<}0.01$ vs.\ RL (Single) (Welch's $t$, Holm--Bonferroni within each panel; 2 pre-registered pairwise tests).}
\label{fig:human_study_pilot_tau2}
\end{figure}

\begin{figure}[h]
\centering
\includegraphics[width=0.95\textwidth]{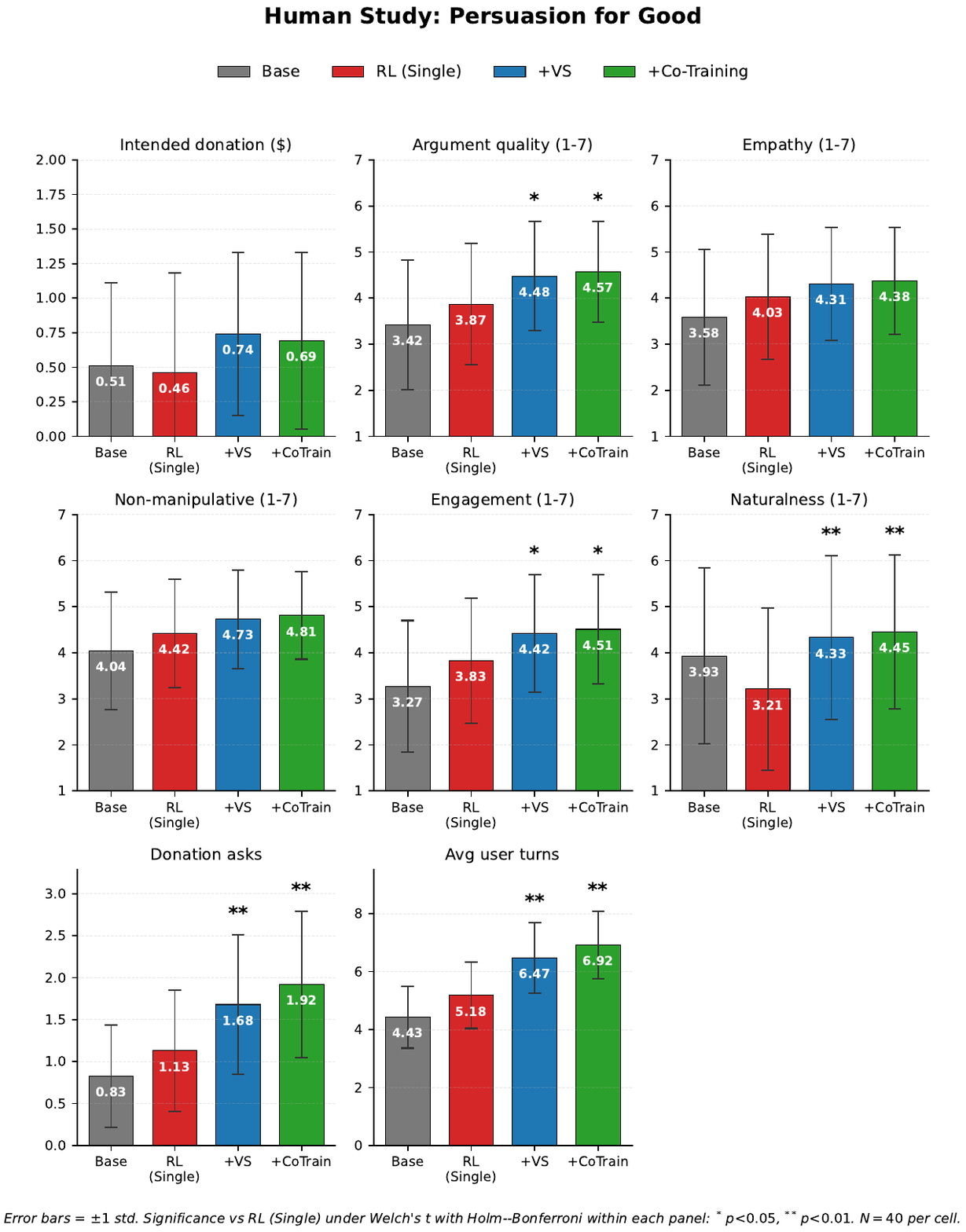}
\caption{\textbf{Human study on Persuasion for Good.} Each subplot reports mean $\pm 1$ std across $N{=}40$ Prolific participants per condition. $^*\,p{<}0.05$, $^{**}\,p{<}0.01$ vs.\ RL (Single) (Welch's $t$, Holm--Bonferroni within each panel; 2 pre-registered pairwise tests).}
\label{fig:human_study_pilot_p4g}
\end{figure}

\clearpage
\section{Extended Experiments}\label{appendix:extended_experiments}

\subsection{Where Population Co-Training's Gain Comes From}
\label{sec:ablations}
\label{sec:ablation}

Population Co-Training was the strongest method in all but one setting of Table~\ref{tab:user_sim_results}. The theory suggests two conditions for this gain. The buffer must be wide enough to preserve real disagreement across checkpoints (Q3a, pool size $K$), and the simulator reward must keep each checkpoint informative instead of letting it collapse onto a new mode (Q3b, reward design). We test Q3a on P4G and $\tau^2$-bench Retail and Q3b on $\tau^2$-bench Retail.

\paragraph{Pool size.} Q3a asks how many checkpoints the population should hold. We sweep $K \in \{1, 3, 5, 10\}$ and save one checkpoint every four training steps (Figure~\ref{fig:k_ablation}). With $K{=}1$ the pool reduces to a single moving target, while a larger $K$ mixes in older checkpoints whose simulator lags behind the current one. $K{=}5$ and $K{=}10$ end close together and above $K{=}1$ and $K{=}3$, with $K{=}5$ ahead on P4G and level with $K{=}10$ on $\tau^2$-Retail. The gain from a larger pool therefore levels off around $K{=}5$. One explanation is that the oldest checkpoints in a large pool act as stale partners that dilute the gradient signal, although the difference between $K{=}5$ and $K{=}10$ is small.

\begin{figure*}[t]
\centering
\includegraphics[width=0.78\textwidth]{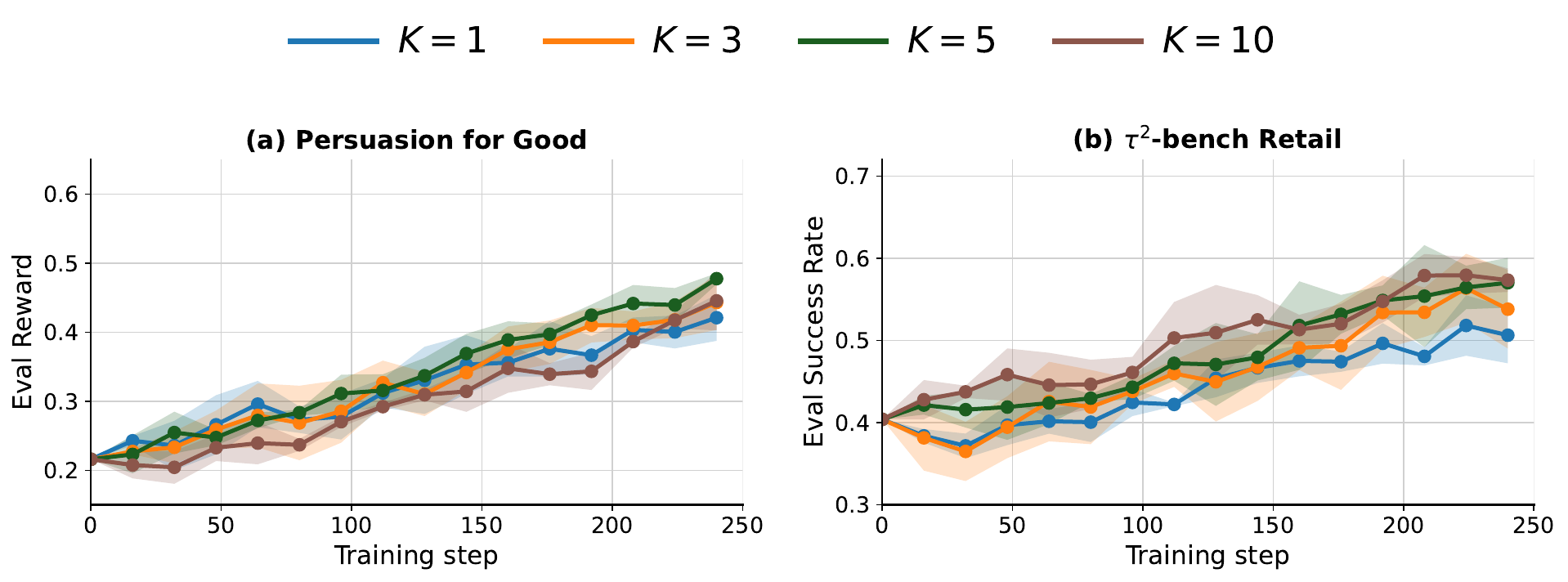}
\caption{\textbf{Pool size $K$ in Population Co-Training.} Eval reward over training for $K \in \{1, 3, 5, 10\}$ on P4G (left) and $\tau^2$-Retail (right), with one checkpoint saved every four training steps. $K{=}5$ ends highest on P4G and level with $K{=}10$ on $\tau^2$-Retail, and both end above $K{=}1$ and $K{=}3$.}
\label{fig:k_ablation}
\end{figure*}

\paragraph{Simulator reward.} For Q3b we compare three simulator rewards, \emph{adversarial} ($r_\phi = -r_\pi$), \emph{cooperative} ($r_\phi = r_\pi$), and a SPICE-style \emph{curriculum} reward that peaks at a target within-batch variance. The adversarial and cooperative rewards both collapse the simulator onto a new dominant mode, so the target the policy trains against stops moving. Only the curriculum reward preserves within-batch variation (Remark~\ref{rem:informative-variation}) and improves on the $K{=}3$ ensemble baseline, which has no co-training. Appendix~\ref{appendix:reward_ablation} gives the variants, numbers, and training curves.

\subsection{Reward-quadrant ablation}

The main results in Section~\ref{sec:experiments} pair each task with the reward structure that best fits its roles and objectives. Persuasion for Good uses the donation-based \emph{adversarial} reward that the task definition already encodes, and $\tau^2$-bench uses the \emph{curriculum} reward studied in Appendix~\ref{sec:ablations}, which pushes the simulator toward the within-batch variance regime of Remark~\ref{rem:informative-variation}. CooperBench is symmetric, and both sides share the binary task-success reward that comes with the benchmark, which is the cooperative case. How much do the conclusions depend on that pairing? For the asymmetric row, Appendix~\ref{appendix:reward_ablation} swaps the $\tau^2$-bench curriculum reward for the adversarial and the cooperative variant and reruns the same methods: both extremes re-collapse the simulator and fall below the ensemble baseline, while the curriculum reward does not. The symmetric row is covered only in its cooperative cell, since CooperBench ships a shared task-success reward and we train it as given. The symmetric adversarial cell therefore remains untested, and we read the conclusions as holding for the quadrants we ran rather than for the full grid.

\begin{figure*}[t]
\centering
\includegraphics[width=\textwidth]{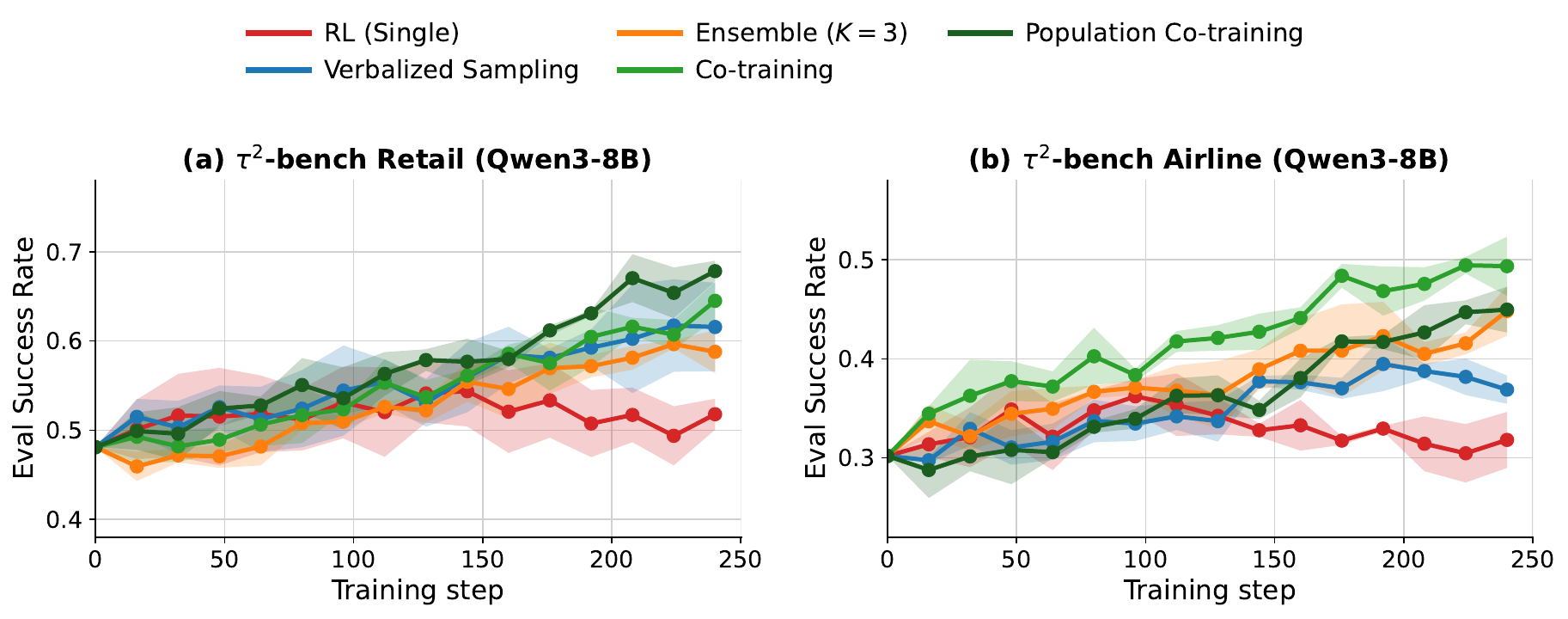}
\caption{\textbf{$\tau^2$-bench with Qwen3-8B.} With Qwen3-8B as the trainable policy instead of Qwen3-4B-Instruct, the pattern of Figure~\ref{fig:main_curves} holds. RL (Single) falls back toward the untrained baseline, every population-based method improves steadily, and the top curve is Population Co-Training on Retail and Co-Training on Airline. The curves are noisier because this setting uses fewer eval episodes.}
\label{fig:qwen8b_tau2}
\end{figure*}

\subsection{Empirical training dynamics across methods}
\label{appendix:training_dynamics}

Figure~\ref{fig:training_dynamics} adds batch-level diagnostics to Figure~\ref{fig:entropy_collapse_sim}. RL against a single frozen simulator degrades on every diagnostic the theory predicts. The fraction of zero-variance batches climbs from $60\%$ to about $82\%$ (Lemma~\ref{lem:user-variance}), policy entropy drops from about $1.1$ to below $0.1$ nats by step 100 (Corollary~\ref{cor:entropy-collapse}), and the fraction of all-failure batches rises from $10\%$ to about $62\%$ as the policy concentrates on a mode-exploit strategy. Verbalized Sampling, the cross-family ensemble ($K{=}3$), and both Co-Training variants keep the zero-variance and all-failure fractions low and entropy well above zero, and the two Co-Training variants hold the highest entropy and the lowest zero-variance fraction.

\begin{figure*}[t]
\centering
\includegraphics[width=\textwidth]{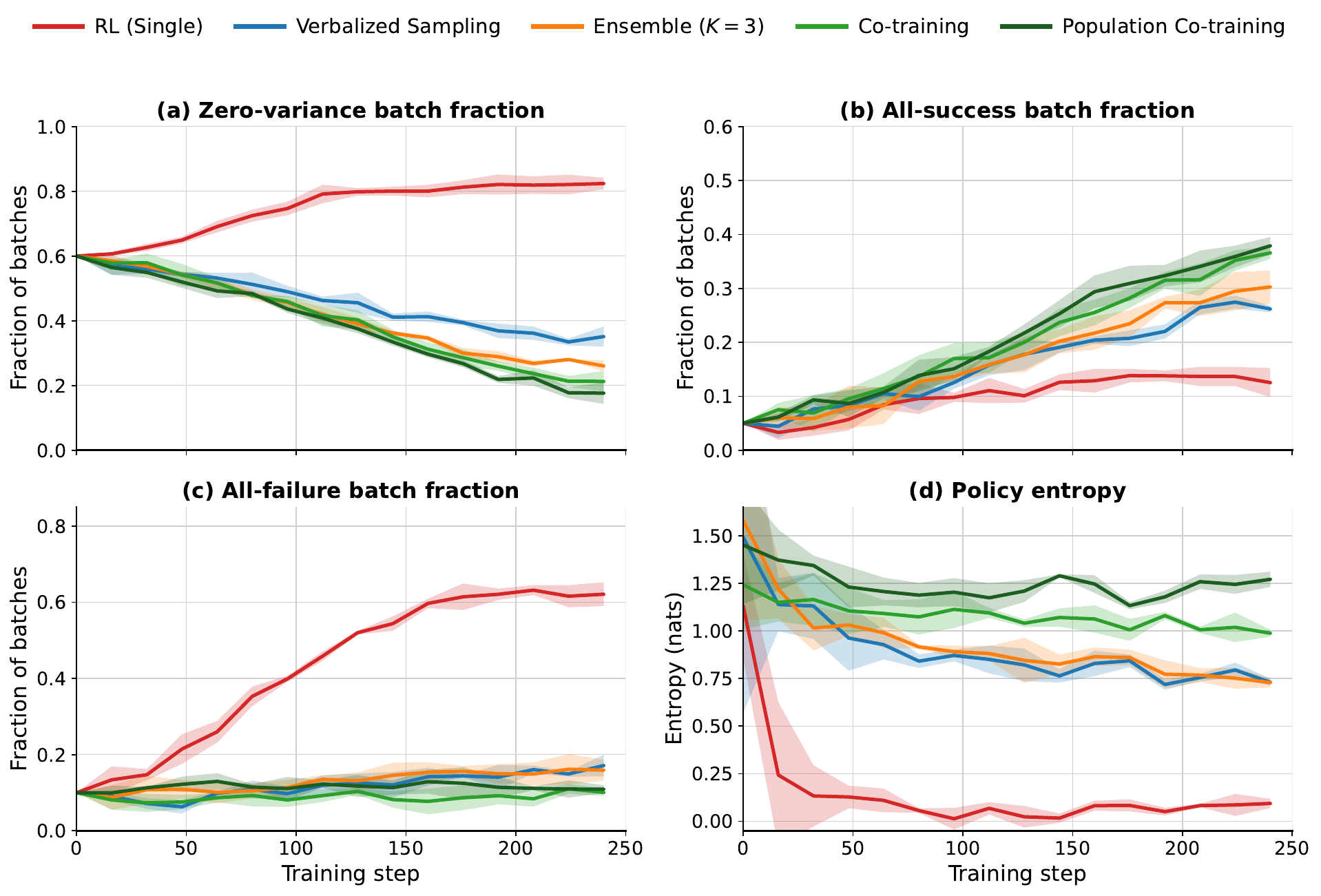}
\caption{\textbf{Training dynamics on $\tau^2$-bench Retail.} Fractions of \textbf{(a)} zero-variance, \textbf{(b)} all-success, and \textbf{(c)} all-failure batches, and \textbf{(d)} policy entropy over training. Single-simulator RL (red) degrades on every panel, while the other methods keep the zero-variance and all-failure fractions low, with Co-Training and Population Co-Training holding the highest entropy. Shaded bands are $\pm1\sigma$ over three seeds.}
\label{fig:training_dynamics}
\end{figure*}

\subsection{Per-benchmark training, eval, and entropy curves across all settings}
\label{appendix:full_curves}

This section reports, for every method on each benchmark, the three signals tracked in \S\ref{exp:fixed_policy}, namely training reward, OOD eval reward, and policy entropy. Each method is run with three seeds, the OOD panel shows $\pm 1\sigma$ shading, and each benchmark uses its own scale and set of methods.

P4G and $\tau^2$-bench Retail (Figures~\ref{fig:appendix_curves_p4g} and~\ref{fig:appendix_curves_tau2}) show the same pattern. RL (Single) reaches the highest training reward against its training simulator and the lowest policy entropy, but its OOD eval falls back toward the untrained baseline. Verbalized Sampling and the frozen ensemble reach a lower training reward but a higher OOD eval, although the ensemble's entropy decays late in training. Population Co-Training ends with the highest OOD eval on both benchmarks, and the two Co-Training variants keep entropy roughly in the $0.8$--$1.2$ nat range throughout training.

The CooperBench figures (Figures~\ref{fig:appendix_curves_cooperbench_9b} and~\ref{fig:appendix_curves_cooperbench_27b}) also track the number of conversation turns. The eval curves differ from the user-simulator benchmarks. Cross-play with the frozen ensemble leads for most of training, peaking at $0.365$ (9B) and $0.623$ (27B) around step 176, while cross-play with a fixed Haiku partner plateaus near $0.28$ (9B) and $0.53$ (27B). Self-play and Population Co-Training first fall below the base model and then climb steadily through the second half of training. Population Co-Training is ahead of Self-play at most eval points in that stretch and ends at its highest value, $0.333$ (9B) and $0.615$ (27B), close to the ensemble's peak and still rising, while the policy entropy of both frozen-partner runs drops to near zero. The turn counts show why. Under cross-play, the turn count rises and then drops (for the Haiku partner, from $7.7$ to $6.4$ at 9B and from $9.3$ to $8.0$ at 27B), because once the policy finds a short strategy that the frozen partner accepts, conversations get shorter. Self-play and Population Co-Training do not show this drop. Their partner keeps adapting, so longer multi-turn solutions stay rewarded, and the turn count keeps climbing to about $11$--$13$. The 9B and 27B runs show the same pattern, and the 9B curves are noisier, which reflects the smaller model's higher per-step variance on software-engineering tasks.

\paragraph{Collapse at the 8B scale.} Simulator collapse persists when the trainable policy is scaled from Qwen3-4B-Instruct to Qwen3-8B. RL (Single) still saturates its training reward, peaks briefly on OOD eval, and loses most of its policy entropy (Figures~\ref{fig:appendix_curves_tau2_8b} and~\ref{fig:qwen8b_tau2}). The 8B run appears to collapse later than the 4B run, although we do not measure the onset precisely.

\begin{figure*}[t]
\centering
\includegraphics[width=\textwidth]{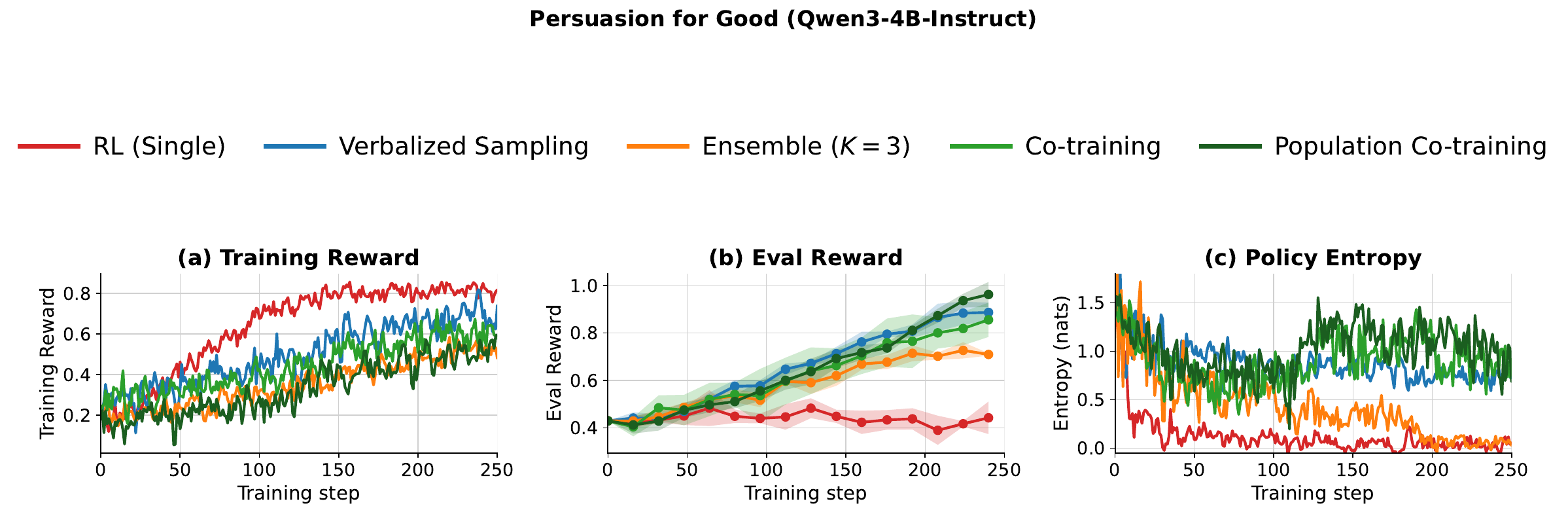}
\caption{\textbf{Persuasion for Good: training reward, OOD eval, and policy entropy.} Five methods with three seeds each, and $\pm1\sigma$ shading on OOD eval. RL (Single) has the highest training reward against the GPT-5-mini training simulator (a), together with the lowest policy entropy (c) and the worst OOD eval (b).}
\label{fig:appendix_curves_p4g}
\end{figure*}

\begin{figure*}[t]
\centering
\includegraphics[width=\textwidth]{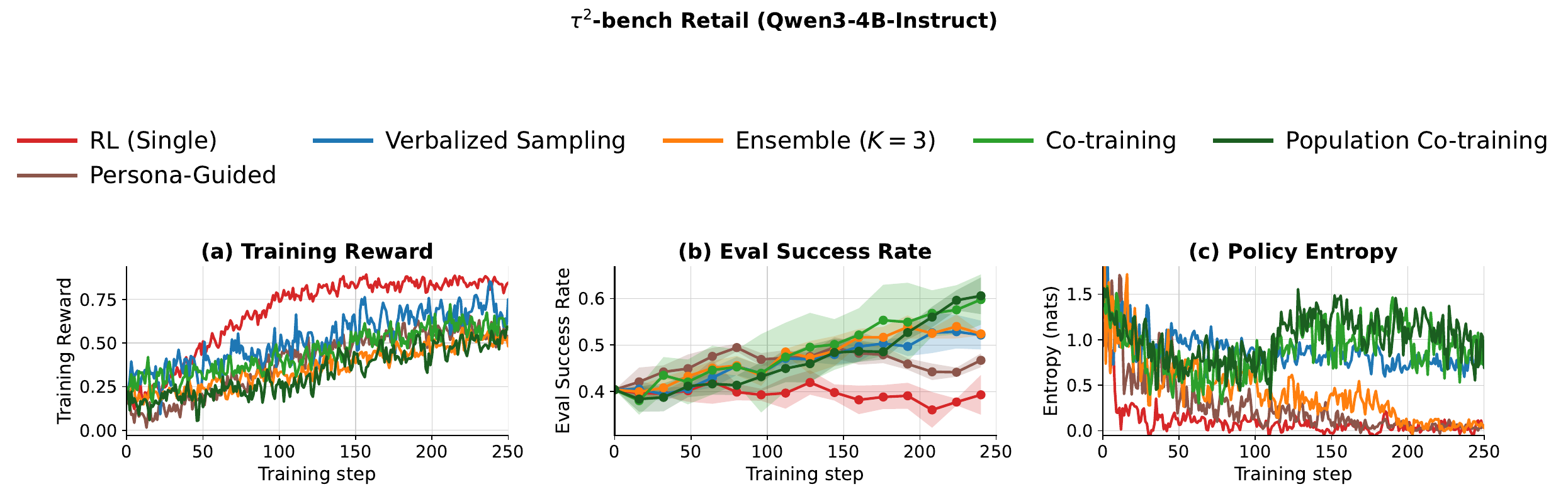}
\caption{\textbf{$\tau^2$-bench Retail: training reward, OOD eval, and policy entropy.} Six methods with three seeds each. The pattern matches Figure~\ref{fig:appendix_curves_p4g}. RL (Single) has the highest training reward and the lowest OOD eval, and Co-Training and Population Co-Training end highest on the evaluation panel.}
\label{fig:appendix_curves_tau2}
\end{figure*}

\begin{figure*}[t]
\centering
\includegraphics[width=\textwidth]{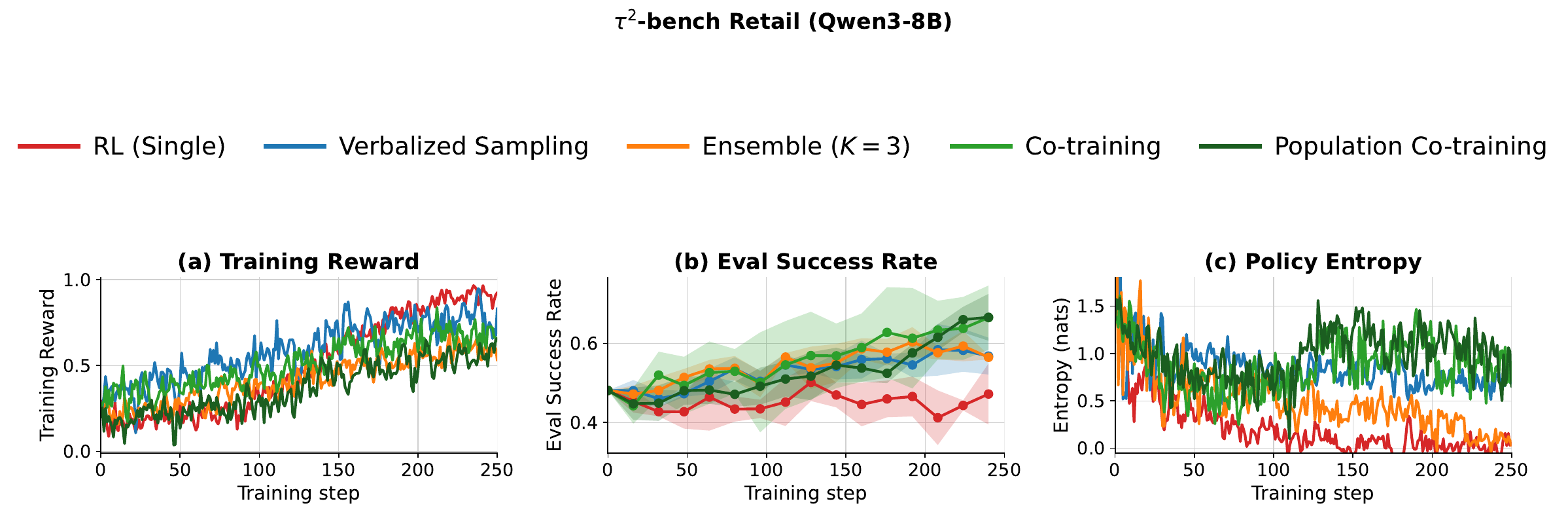}
\caption{\textbf{$\tau^2$-bench Retail (Qwen3-8B).} Five methods with three seeds each, and $\pm1\sigma$ shading on OOD eval. The collapse pattern of Figure~\ref{fig:appendix_curves_tau2} appears at the 8B scale but later in training. RL (Single)'s training reward saturates, its OOD eval peaks and then declines, and its policy entropy drops to near zero. The per-step curves are noisier in all three panels, because the larger policy has higher per-step variance at the same number of eval episodes.}
\label{fig:appendix_curves_tau2_8b}
\end{figure*}

\begin{figure*}[t]
\centering
\includegraphics[width=0.85\textwidth]{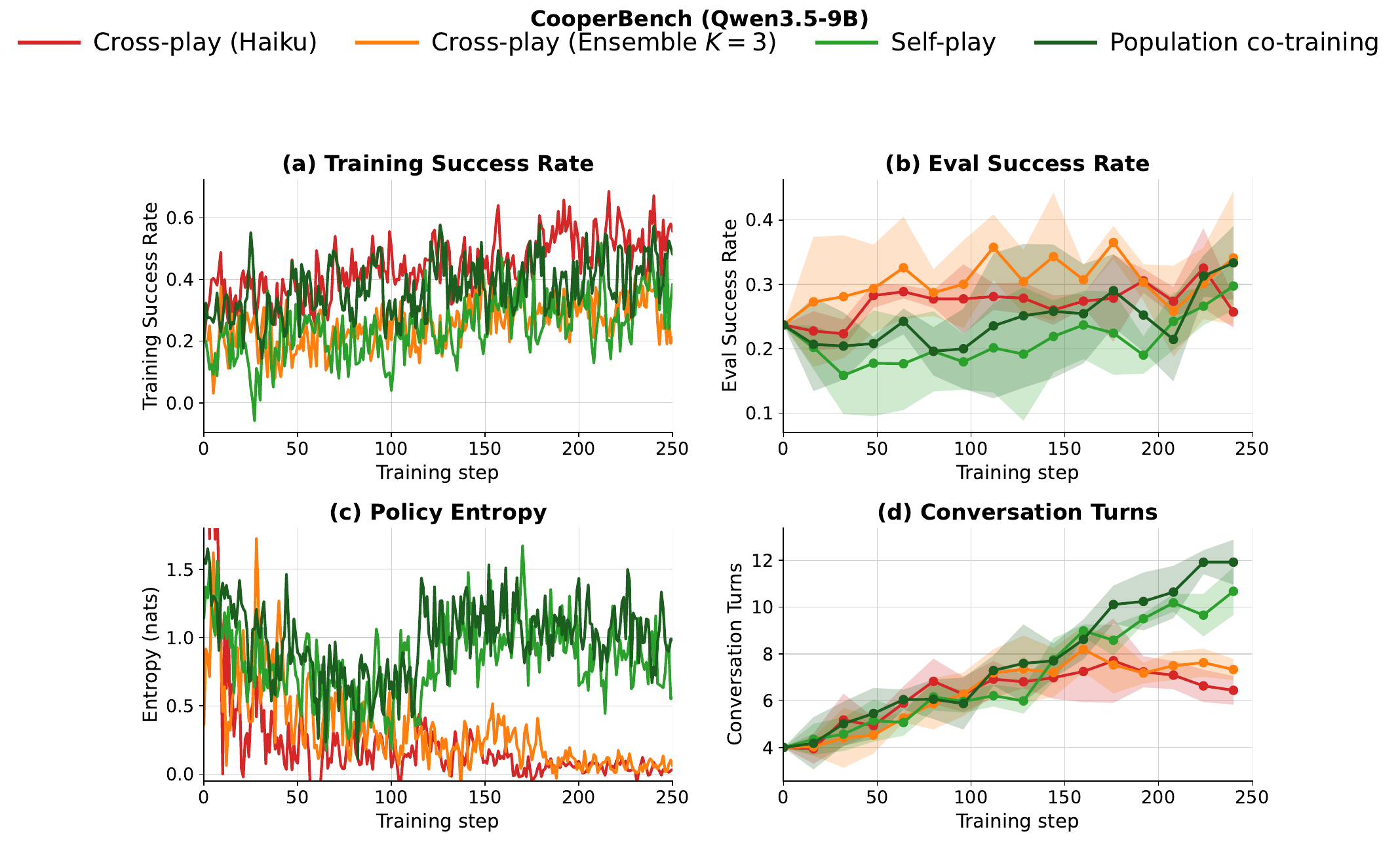}
\caption{\textbf{CooperBench (Qwen3.5-9B): training, eval, entropy, and conversation turns.} The 9B run shows the same overfitting pattern as the Qwen3.5-27B run in Figure~\ref{fig:appendix_curves_cooperbench_27b}, with more per-step variance. Cross-play with the $K{=}3$ frozen ensemble has the highest eval success for most of training, and Population Co-Training reaches a similar level only at the end, after an early dip. Under cross-play with a fixed Haiku partner or the frozen ensemble, the number of conversation turns plateaus and then drops as a short exploit strategy takes over, while it keeps climbing for Self-play and Population Co-Training.}
\label{fig:appendix_curves_cooperbench_9b}
\end{figure*}

\begin{figure*}[t]
\centering
\includegraphics[width=0.85\textwidth]{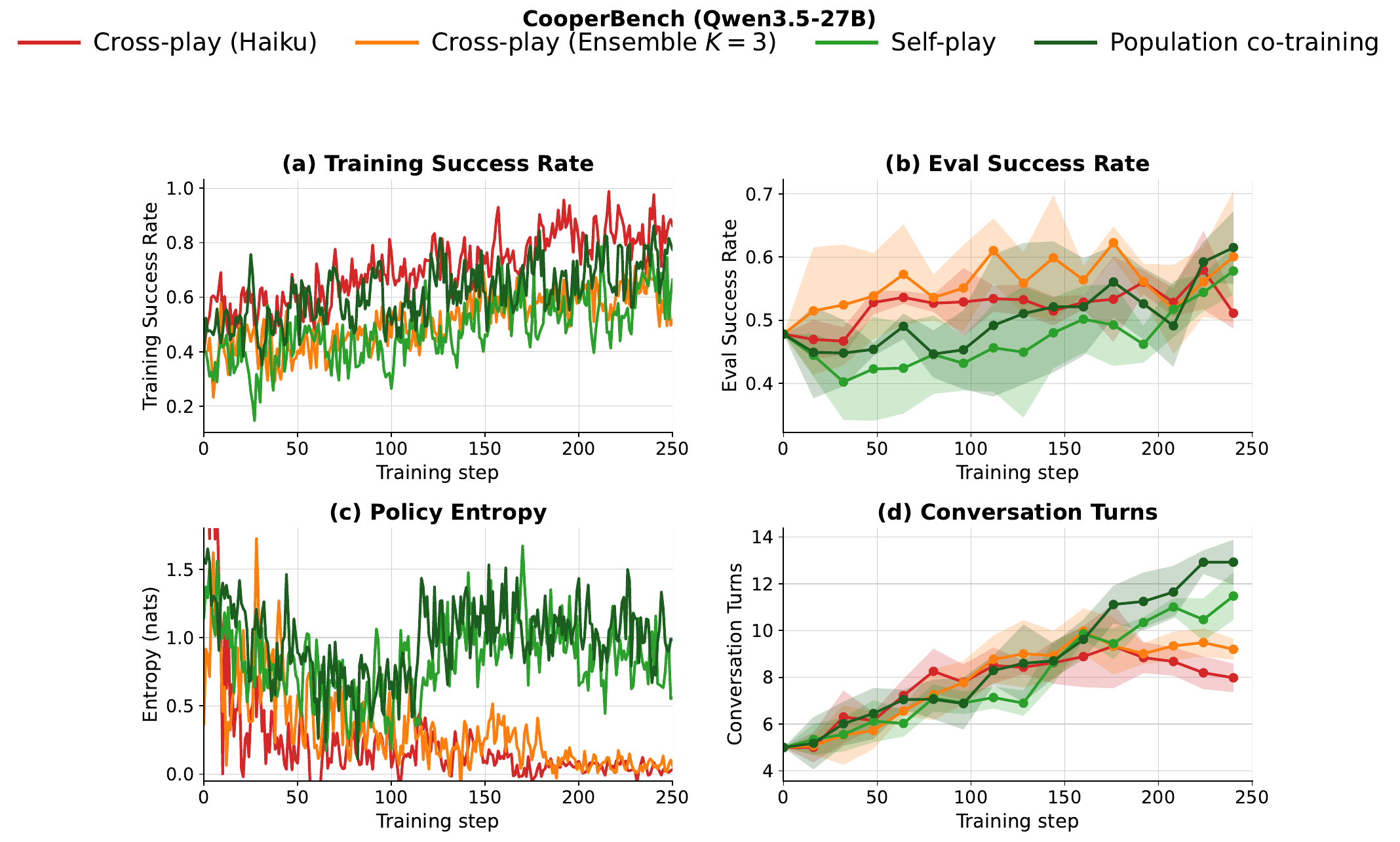}
\caption{\textbf{CooperBench (Qwen3.5-27B): training, eval, entropy, and conversation turns.} Under cross-play with a fixed Haiku partner or a $K{=}3$ frozen ensemble, the number of conversation turns rises early as the policy learns to interact and then drops as a short exploit strategy takes over. Self-play and Population Co-Training avoid this drop, and their turn counts keep climbing as the partner co-evolves. On eval success, the frozen ensemble leads for most of training, and Population Co-Training catches up with it by the last step.}
\label{fig:appendix_curves_cooperbench_27b}
\end{figure*}

\subsection{Proof of Proposition~\ref{prop:vs-lower-bound}}
\label{appendix:vs_theory}

We prove Proposition~\ref{prop:vs-lower-bound} with the coupling argument of Theorem~\ref{thm:modal-gradient} and then bound tail-behavior coverage under $\gamma$-sharpening.

\paragraph{Proof of the gradient bound.}
Couple the trajectory $\tau_{\mathrm{VS}}$ in $M_{\mathrm{VS}}$ (simulator samples from $p^{\mathrm{VS}}_\phi$) and the trajectory $\tau_{\mathrm{ref}}$ in $M_{\mathrm{ref}}$ (simulator samples from $P$) using shared agent randomness. At each simulator turn, given a matched prefix, sample $(a_t^{\mathrm{VS}}, a_t^{\mathrm{ref}})$ from the maximal coupling between $p^{\mathrm{VS}}_\phi(\cdot \mid s_t, a_t^\pi)$ and $P(\cdot \mid s_t, a_t^\pi)$, which disagrees with probability at most $\eta(s_t, a_t^\pi)$ by hypothesis. By the union bound over simulator turns,
\[
\Pr[\tau_{\mathrm{VS}} \ne \tau_{\mathrm{ref}}] \;\le\; \mathbb{E}\!\bigl[\sum_{t=1}^H \eta(s_t, a_t^\pi)\bigr] \;=\; \bar{\eta}_H(\theta),
\]
which gives $D_{\mathrm{TV}}(P^\theta_{\mathrm{VS}}, P^\theta_{\mathrm{ref}}) \le \bar{\eta}_H(\theta)$. The gradient bound then follows from the bounded-integrand TV inequality (Appendix~\ref{appendix:prelim}) applied to $f(\tau) = R(\tau) S_\theta(\tau)$ with $\|f\| \le R_{\max} B$, as in the proof of Theorem~\ref{thm:modal-gradient}. \hfill$\square$

\paragraph{Tail suppression under $\gamma$-sharpening.}
Together with $\gamma$-sharpening (Appendix~\ref{appendix:sharpening}), the recovery assumption separates VS from direct prompting on tail behaviors. Direct prompting suppresses tail behaviors exponentially in $\gamma$, whereas VS preserves them up to $\eta$.

\begin{proposition}[$\gamma$-sharpening tail suppression]
\label{prop:tail-suppression}
Fix a state $s$ and let $b^\star = \arg\max_b P(b \mid s)$ with mass $m = P(b^\star \mid s)$. Let $B \subset \mathcal{B}$ be a set of nonmodal behaviors with reference mass $P(B \mid s) = \rho$ and maximum per-behavior mass $\max_{b \in B} P(b \mid s) \le \lambda m$ for some $\lambda \in [0, 1)$. Under direct $\gamma$-sharpened prompting $P_\gamma(b \mid s) \propto P(b \mid s)^\gamma$ with $\gamma > 1$ (Proposition~\ref{prop:sharpening}),
\[
P_\gamma(B \mid s) \;\le\; \frac{\rho}{m}\, \lambda^{\gamma - 1}.
\]
Under VS with reference-recovery error $\eta$,
\[
p^{\mathrm{VS}}_\phi(B \mid s) \;\ge\; \rho - \eta.
\]
\end{proposition}

\begin{proof}
$\sum_{b \in B} P(b)^\gamma = \sum_{b \in B} P(b) \cdot P(b)^{\gamma - 1} \le \rho (\lambda m)^{\gamma - 1}$ using $P(b) \le \lambda m$ on $B$. The denominator of $P_\gamma$ is at least the modal contribution $m^\gamma$, giving $P_\gamma(B) \le \rho (\lambda m)^{\gamma - 1} / m^\gamma = (\rho/m)\, \lambda^{\gamma - 1}$. For the VS bound, $|p^{\mathrm{VS}}_\phi(B) - P(B)| \le D_{\mathrm{TV}}(p^{\mathrm{VS}}_\phi, P) \le \eta$.
\end{proof}

For the empirical range $\gamma \in [6, 66]$ (Appendix~\ref{appendix:sharpening}) and $\lambda$ bounded away from $1$, $P_\gamma(B)$ is near zero while $p^{\mathrm{VS}}_\phi(B)$ stays within $\eta$ of $\rho$. By Proposition~\ref{prop:vs-lower-bound}, the policy gradient under VS then approximates the reference-user gradient, with the gap controlled by $\bar{\eta}_H$.

\paragraph{Limits of the reference-recovery assumption.}
The reference-recovery assumption $D_{\mathrm{TV}}(p^{\mathrm{VS}}_\phi, P) \le \eta$ is an empirical claim. It can fail if the simulator's verbalized output systematically misses behavior types, for example if the $K$ candidates always come from the cooperative half of the response distribution. The most direct test would sample the simulator under VS, cluster the responses into behavior types, and compare the empirical distribution to a behavior-level reference. We report transcript inspections in Appendix~\ref{appendix:collapse_examples}, but we do not measure behavior coverage quantitatively. Reference recovery alone is not sufficient, because $P$ must also place mass on behaviors the policy needs to transfer. Proposition~\ref{prop:tail-suppression} formalizes one such asymmetry. The real-user step requires an additional assumption $D_{\mathrm{TV}}(P, P_{\mathrm{real}}) \le \kappa$, which VS itself does not establish and which our human study (Appendix~\ref{appendix:human_study}) does not test, since it covers only Co-Training.

\subsection{Verbalized Sampling mitigates simulator collapse}
\label{appendix:vs_ablation}

We isolate the effect of Verbalized Sampling~\citep{zhang2025verbalizedsamplingmitigatemode} with the training simulator of our main experiments, GPT-5-mini. Figure~\ref{fig:vs_ablation} compares RL (Single) with and without Verbalized Sampling, using matched hyperparameters and three seeds each. Without VS, the policy shows the simulator-collapse pattern of \S\ref{exp:fixed_policy}, in which training reward climbs steadily, OOD eval peaks and then declines, and policy entropy drops to near zero. With VS, each simulator reply is drawn from a verbalized response distribution, which restores enough simulator-side variance (Lemma~\ref{lem:user-variance}) to slow the geometric concentration of Corollary~\ref{cor:entropy-collapse}.

\begin{figure*}[t]
\centering
\includegraphics[width=\textwidth]{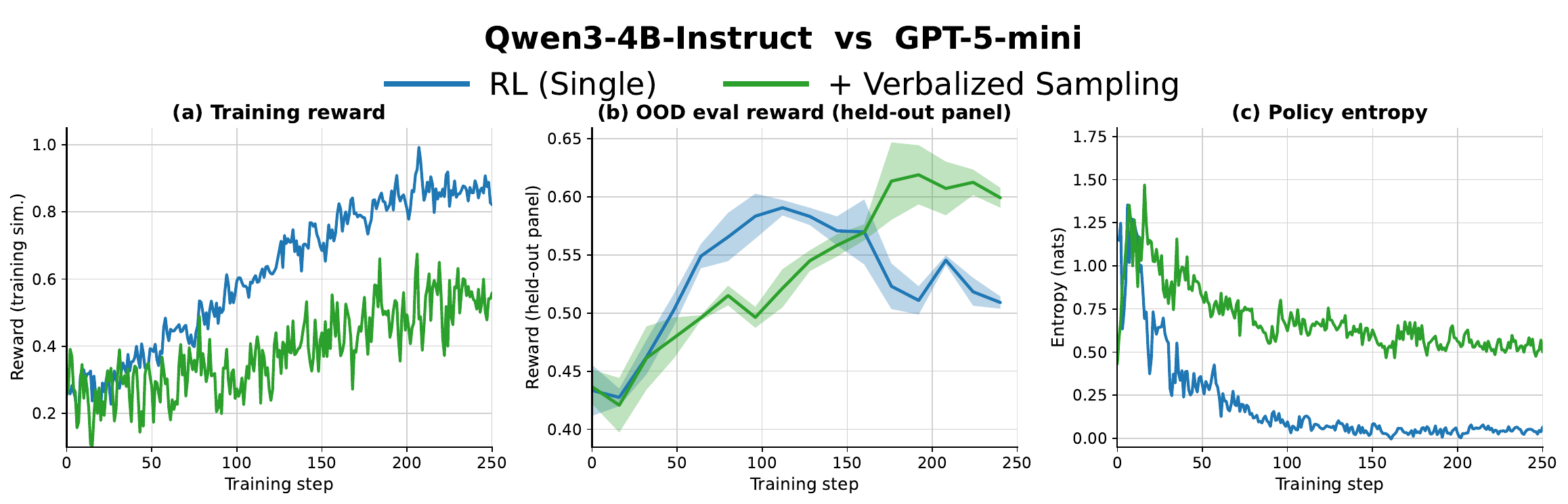}
\caption{\textbf{Verbalized Sampling mitigates simulator collapse against GPT-5-mini.} Three seeds per setting. \textbf{(a)} Training reward. RL (Single) climbs higher than $+$VS, whose training reward is lower and noisier. \textbf{(b)} OOD eval reward on the six-model panel ($\pm 1\sigma$ shaded). RL (Single) peaks around step 100 and then declines, while $+$VS ends above RL (Single)'s peak and declines far less. \textbf{(c)} Policy entropy. RL (Single) collapses to near zero, while $+$VS declines slowly and keeps entropy well above zero. VS reduces simulator collapse but does not eliminate it, and the remaining gap motivates the Co-Training experiments in \S\ref{sec:results}.}
\label{fig:vs_ablation}
\end{figure*}

\subsection{Verbalized Sampling on larger models}
\label{appendix:vs_gpt5}

Here we train against the larger GPT-5 simulator. The main experiments use GPT-5-mini as the simulator in every setting, to stay within our budget and to keep the comparison in Table~\ref{tab:user_sim_results} fair. \citet{zhang2025verbalizedsamplingmitigatemode} note that smaller models suffer from ``cognitive overload'' when asked to verbalize a distribution while solving the task. This should be less of a concern for GPT-5, so we expect the mitigation from VS to be more pronounced.

Figure~\ref{fig:vs_ablation_gpt5} shows the same three diagnostics against GPT-5. The collapse pattern persists but appears later. Consistent with GPT-5 being less modal than GPT-5-mini, the training reward climbs more slowly and reaches about $0.6$--$0.7$ by the end of training, and the entropy of RL (Single) collapses much later than with GPT-5-mini. OOD eval for RL (Single) still peaks, at about $0.56$ around step 160, and then declines to $0.50$. With Verbalized Sampling, the training reward is noisier and lower, but OOD eval keeps climbing and ends near $0.66$, and entropy stays near $1.5$ nats until about step 150 before declining to about $0.8$ nats, with large fluctuations throughout. As with GPT-5-mini (\S\ref{appendix:vs_ablation}), VS reduces simulator collapse but does not eliminate it.

\begin{figure*}[t]
\centering
\includegraphics[width=\textwidth]{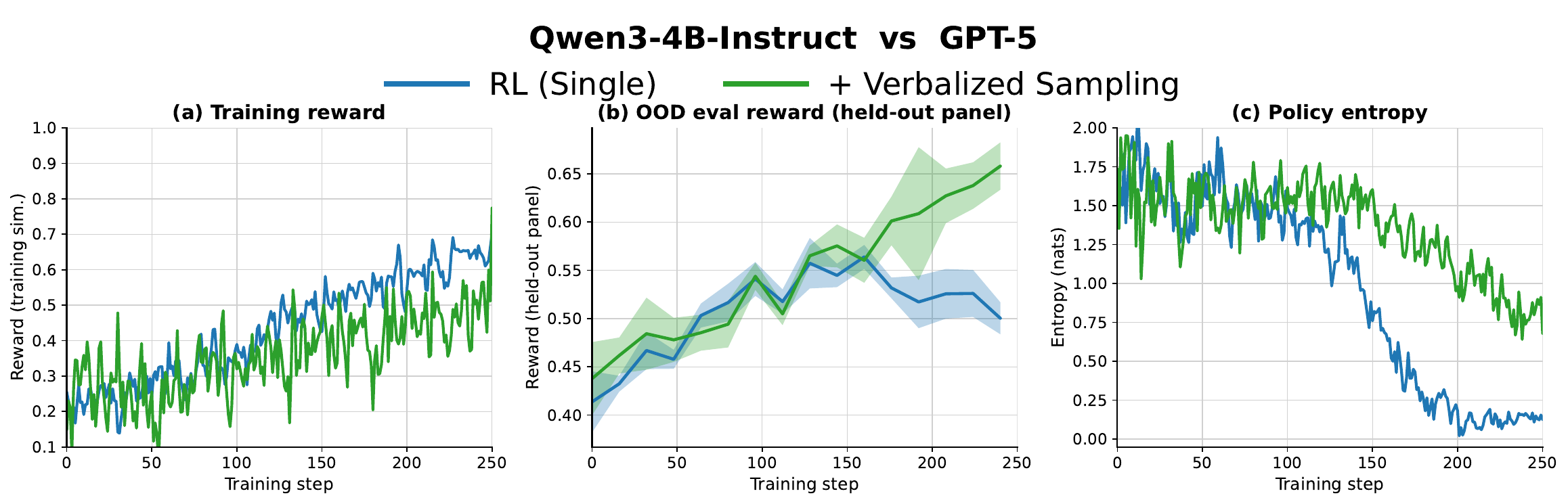}
\caption{\textbf{Verbalized Sampling against GPT-5.} Three seeds per setting. \textbf{(a)} Training reward. RL (Single) climbs to about $0.6$--$0.7$ by the end of training, while $+$VS is noisier and lower, because sampling among $K$ verbalized candidates gives a less consistent reward signal. \textbf{(b)} OOD eval on the six-model panel ($\pm 1\sigma$ shaded). RL (Single) peaks near $0.56$ and then declines to $0.50$, while $+$VS keeps climbing and ends near $0.66$. \textbf{(c)} Policy entropy. RL (Single) stays high for most of training and then collapses, while $+$VS stays near $1.5$ nats until about step 150 and then declines to about $0.8$ nats, with large fluctuations.}
\label{fig:vs_ablation_gpt5}
\end{figure*}

\subsection{Simulator-reward ablation}
\label{appendix:reward_ablation}

Q3b asks whether any simulator reward produces a moving target or whether the reward has to be shaped carefully. We compare three simulator rewards (Table~\ref{tab:cotrain_reward}) with the $K{=}3$ ensemble baseline, which has no co-training (Figure~\ref{fig:cotrain_reward_ablation}). The adversarial and cooperative rewards both stop the target from moving, in opposite ways. The \emph{adversarial} reward ($r_\phi = -r_\pi$) collapses the simulator to ${\sim}98\%$ refusals and lowers eval reward from $0.40$ to about $0.09$. The \emph{cooperative} reward ($r_\phi = r_\pi$) reduces pushback to ${\sim}2\%$, leaving a trivially helpful simulator for the policy to reward-hack, and eval reward drops from $0.40$ to about $0.18$. Only the \emph{curriculum} reward keeps the opponent's reward between these extremes, rising to about $0.36$ by the end of training, and it holds eval reward at ${\sim}0.40$, above the ensemble baseline (${\sim}0.35$ at the end of training). Both extreme rewards collapse the simulator onto a new dominant mode, and once the mode is fixed again, the analysis of \S\ref{sec:theory_analysis} applies as before and the target stops moving. Population Co-Training therefore helps \emph{only when the simulator reward preserves variation across checkpoints}.

\begin{table}[h]
\centering
\small
\setlength{\tabcolsep}{6pt}
\renewcommand{\arraystretch}{1.1}
\caption{Simulator-reward variants. $r_\pi$ is the policy's per-rollout reward; $\sigma^2_\pi$ its within-group variance. Curriculum follows SPICE-style variance shaping~\citep{liu_spice_2025}.}
\label{tab:cotrain_reward}
\begin{tabular}{ll}
\toprule
\textbf{Variant} & \textbf{Simulator reward $r_\phi$} \\
\midrule
\colorbox{gray!15}{Adversarial} & $-r_\pi$ \\
\addlinespace[0.4ex]
\hdashline
\addlinespace[0.4ex]
\colorbox{gray!15}{Cooperative} & $r_\pi$ \\
\addlinespace[0.4ex]
\hdashline
\addlinespace[0.4ex]
\colorbox{gray!15}{Curriculum} & $\exp\!\left(-\dfrac{(\sigma^2_\pi - 0.25)^2}{0.02}\right)$ \\
\bottomrule
\end{tabular}
\end{table}

\begin{figure*}[t]
\centering
\includegraphics[width=0.78\textwidth]{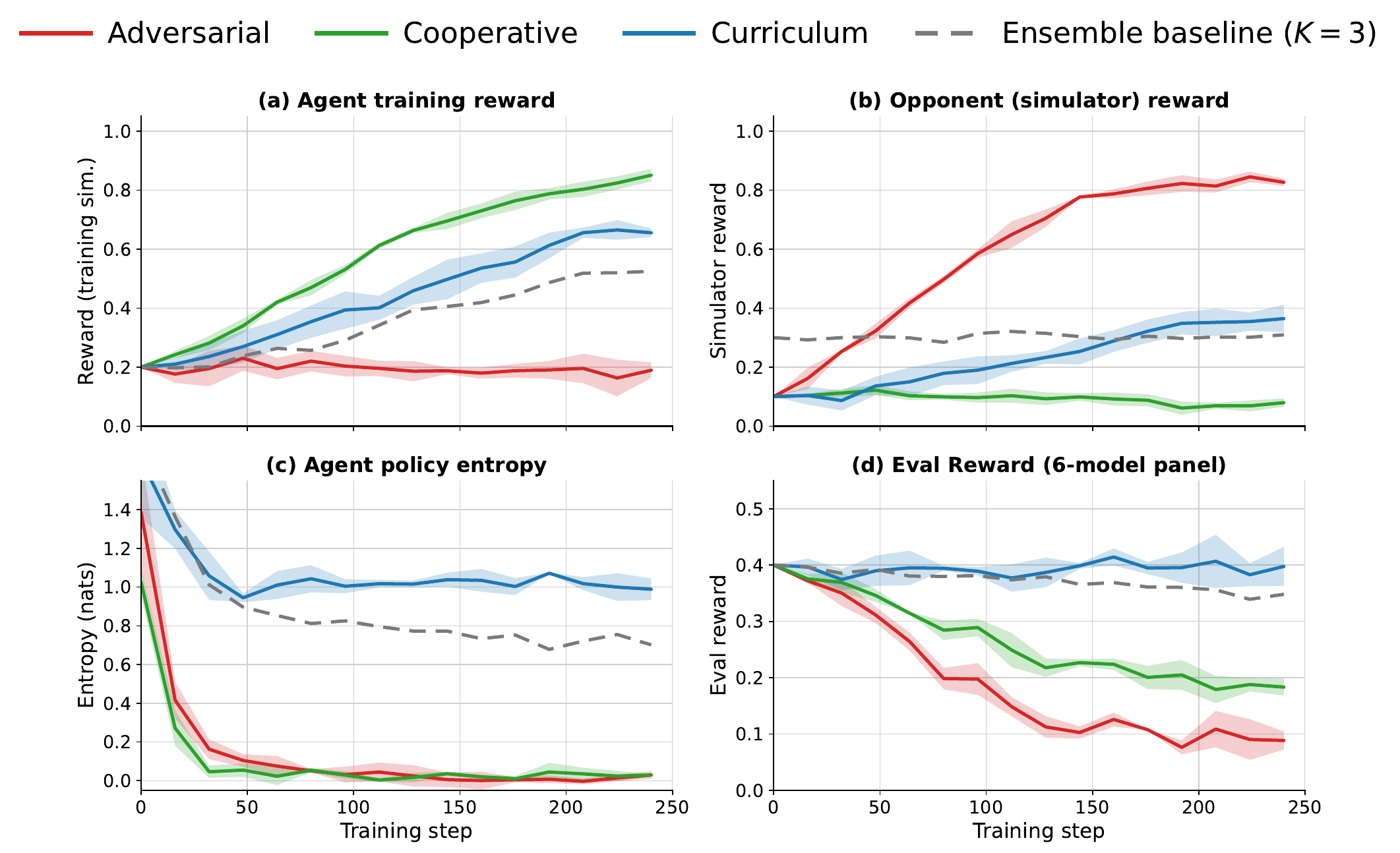}
\caption{\textbf{Co-Training requires a carefully chosen simulator reward ($\tau^2$-bench Retail).} \textbf{(a)} Agent training reward, which is highest under the cooperative reward (green) because the policy reward-hacks a trivially helpful simulator. \textbf{(b)} Opponent reward, which rises to about $0.83$ under the adversarial reward, stays below $0.13$ under the cooperative reward, and reaches about $0.36$ under the curriculum reward (blue). \textbf{(c)} Agent policy entropy, which collapses to $0.01$--$0.04$ under both extreme rewards. \textbf{(d)} Eval reward on the six-model panel, where only the curriculum reward beats the ensemble baseline without co-training.}
\label{fig:cotrain_reward_ablation}
\end{figure*}

\subsection{Within-batch rollouts narrow to a single strategy}
\label{appendix:collapse_examples}

To see what simulator collapse looks like on the policy side, we read three rollouts that share a batch and a starting context at three stages of a single-simulator run (early, mid, and late). Early on, the agent tries varied strategies and the simulator answers with varied content. By mid-training the rollouts share the same structure, and by late training they are nearly word-for-word the same. The reading is qualitative, and we describe the pattern here rather than reproduce the transcripts, but it is the behavior the geometric concentration of strategy mass in Corollary~\ref{cor:entropy-collapse} would produce.

\subsection{Training on more models}
\label{appendix:training_olmo}

We repeat the main P4G and $\tau^2$-bench Retail experiments with Olmo-3-7B-Instruct~\citep{olmo2026olmo3} as the trainable agent and keep the rest of the setup unchanged. Olmo-3-7B-Instruct differs from Qwen3-4B-Instruct in two ways. It starts with higher policy entropy (about $1.8$ vs.\ about $1.4$ nats at the first training step), which suggests a less sharpened RLHF profile, and its Base task performance is lower ($0.34$ vs.\ $0.40$ on $\tau^2$-Retail and $0.35$ vs.\ $0.43$ on P4G), consistent with its smaller post-training budget. Despite these different starting points, the training dynamics match the Qwen3-4B-Instruct results. RL (Single) saturates its training reward, peaks briefly on OOD eval, and loses its policy entropy from the higher starting point. Verbalized Sampling and Co-Training both recover most of the OOD gap and preserve entropy, and Population Co-Training is the strongest method on both benchmarks. Figures~\ref{fig:appendix_curves_olmo_p4g} and~\ref{fig:appendix_curves_olmo_tau2} report the full curves.

\begin{figure*}[t]
\centering
\includegraphics[width=\textwidth]{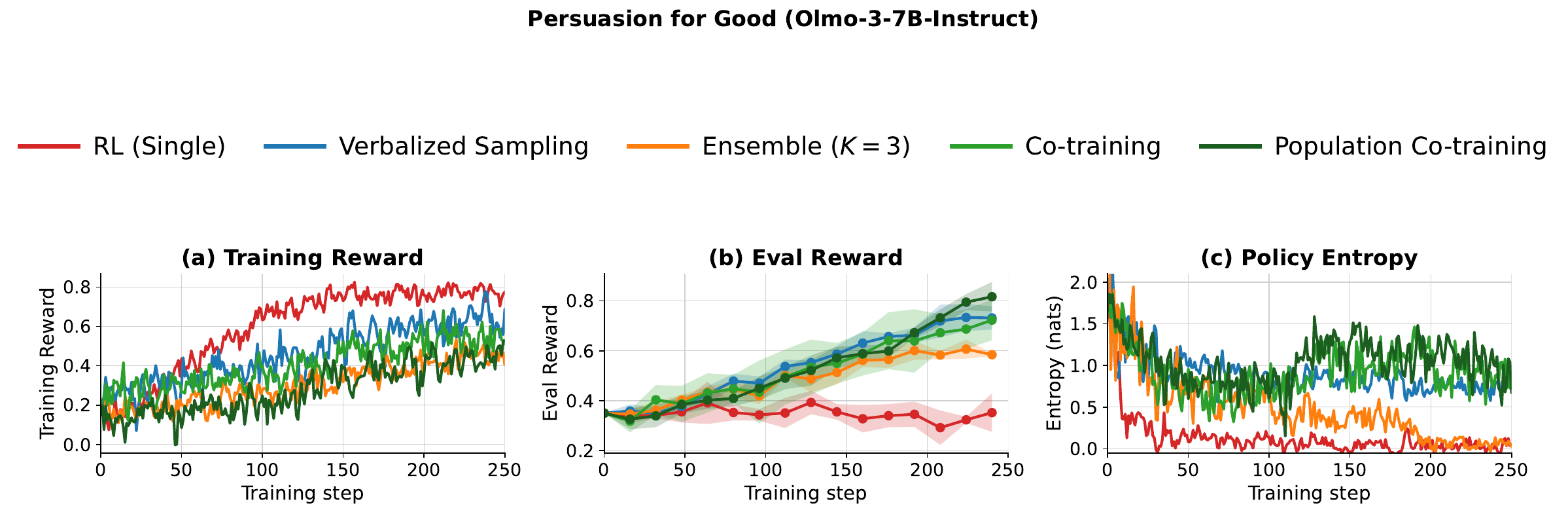}
\caption{\textbf{Persuasion for Good (Olmo-3-7B-Instruct).} The three panels show training reward, OOD eval reward, and policy entropy. The simulator-collapse pattern reproduces. RL (Single) saturates its training reward, peaks briefly on OOD eval, and loses its entropy from a higher starting point. Verbalized Sampling, Co-Training, and Population Co-Training preserve entropy and close most of the held-out gap, and the Ensemble improves OOD eval although its entropy decays late in training.}
\label{fig:appendix_curves_olmo_p4g}
\end{figure*}

\begin{figure*}[t]
\centering
\includegraphics[width=\textwidth]{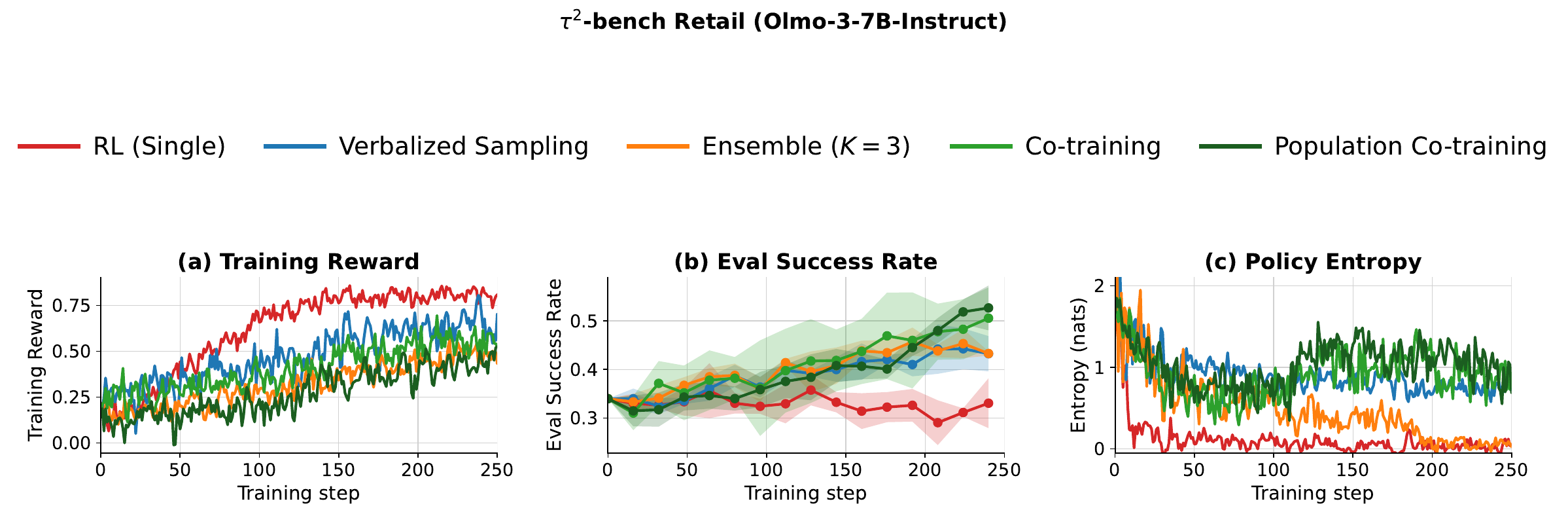}
\caption{\textbf{$\tau^2$-bench Retail (Olmo-3-7B-Instruct).} The panels are the same as in Figure~\ref{fig:appendix_curves_olmo_p4g}. Olmo-3-7B-Instruct's Base success rate is lower than Qwen3-4B-Instruct's ($0.34$ vs.\ $0.40$), but the pattern of collapse and recovery is the same.}
\label{fig:appendix_curves_olmo_tau2}
\end{figure*}

\end{document}